\documentclass[11pt]{article}
\usepackage[margin=.88in]{geometry}
\usepackage[utf8]{inputenc}
\usepackage{amsfonts}
\usepackage{amsgen,amsmath,amstext,amsbsy,amsopn,amssymb}
\usepackage[dvips]{graphicx}
\usepackage{comment}
\usepackage{caption}
\usepackage{subcaption}

\usepackage[shortlabels]{enumitem}
\usepackage{natbib}
\usepackage{booktabs}
\usepackage{amsthm,amsmath}
\usepackage{hyperref}
\usepackage{blkarray}
\usepackage{algorithm}
\usepackage{algorithmic}
\usepackage{url}
\usepackage{longtable}
\usepackage{stmaryrd}
\usepackage{enumitem}
\usepackage{mwe}

\DeclareMathAlphabet\mathbfcal{OMS}{cmsy}{b}{n}

\numberwithin{equation}{section}

\makeatletter
\newcommand*{\rom}[1]{\expandafter\@slowromancap\romannumeral #1@}
\makeatother

\usepackage[dvipsnames]{xcolor}

\newcommand{\argmin}{\mathop{\rm arg\min}}

\newcommand{\abs}[1]{\left\lvert#1\right\rvert}
\newcommand{\norm}[1]{\left\lVert#1\right\rVert}

\newcommand{\E}{\mathbb{E}}
\newcommand{\PP}{\mathop{{}\mathbb{P}}}
\newcommand{\indep}{\perp \!\!\! \perp}

\newcommand{\blambda}{\boldsymbol{\lambda}}
\newcommand{\bmu}{\boldsymbol{\mu}}

\newcommand{\cB}{\mathcal{B}}
\newcommand{\cH}{\mathcal{H}}
\newcommand{\cL}{\mathcal{L}}
\newcommand{\cP}{\mathcal{P}}
\newcommand{\cX}{\mathcal{X}}
\newcommand{\cZ}{\mathcal{Z}}

\newtheorem{theorem}{Theorem}
\newtheorem{assumption}{Assumption}

\newtheorem{lemma}{Lemma}
\newtheorem{corollary}{Corollary}

\newtheorem{Proposition}{Proposition}
\newtheorem{Condition}{Condition}

\newtheorem{remark}{Remark}

\theoremstyle{remark}{}
\theoremstyle{remark}{}

\title{Optimal Regularized Online Allocation by Adaptive Re-Solving}
\author{Wanteng Ma$^1$, ~ Ying Cao$^2$, ~ Danny H.K. Tsang$^2$, ~  Dong Xia$^{1}$\footnote{Ma and Cao are co-first authors. Tsang's research was partially supported by Hong Kong RGC GRF 16211220; Xia's research was partially supported by Hong Kong RGC Grant GRF 16300121 and 16301622. }\\
~ \\
$^1$Department of Mathematics, HKUST\\
$^2$Department of Electronic and Computer Engineering, HKUST}
\date{(\today)}

\begin{document}

\maketitle

\begin{abstract}
    This paper introduces a dual-based algorithm framework for solving the regularized online resource allocation problems, which have potentially non-concave cumulative  rewards, hard resource constraints, and a non-separable regularizer. Under a strategy of adaptively updating the resource constraints, the proposed framework only requests approximate solutions to the empirical dual problems up to a certain accuracy and yet delivers an optimal logarithmic regret under a locally second-order growth condition. Surprisingly, a delicate analysis of the dual objective function enables us to eliminate the notorious log-log factor in regret bound.   The flexible framework renders renowned and computationally fast algorithms immediately applicable, e.g., dual stochastic gradient descent. Additionally, an infrequent re-solving scheme is proposed, which significantly reduces computational demands without compromising the optimal regret performance. A worst-case square-root regret lower bound is established if the resource constraints are not adaptively updated during dual optimization, which underscores the critical role of adaptive dual variable update. Comprehensive numerical experiments demonstrate the merits of the proposed algorithm framework. 
    
\end{abstract}

\begin{sloppypar}

\section{Introduction}\label{sec:intro}
Online resource allocation seeks to maximize the total rewards in an online service system that is subject to resource constraints. As an exemplary model for sequential decision-making, online allocation has drawn considerable attention in recent decades. Meanwhile, it is strongly connected to  other online problems such as revenue management \citep{talluri2004theory}, online linear programming \citep{agrawal2014dynamic}, and ads bidding problems \citep{lee2013real},  to name but a few. Online allocation finds applications in diverse fields, e.g.,  computer science and operation research. Oftentimes, online allocation problems feature resource constraints that are either hard \citep{mehta2007adwords} or soft \citep{mahdavi2012trading}, with different constraint capacities. The goal of a decision maker is to maximize the total rewards (revenue, utility) function by a real-time decision policy that enforces each of the resource constraints. 

So far, existing literature on online allocation mostly focused on additively separable objectives, i.e., the objective function only involves the total rewards that can be simply described as the cumulative rewards by time (e.g., \cite{mehta2007adwords,devanur2009adwords,balseiro2019learning}). 
While a separable objective is favorable for tracking additive total rewards, it falls short of describing globally non-separable quantities such as total resource consumption or average actions. For instance, the average action \citep{agrawal2014fast} in online advertising measures the amount of under-delivery of impressions. Unfortunately, non-separable objectives are considerably under-explored in the literature, and particularly, there is a paucity of work investigating the impact of non-separable regularization on separable cumulative reward functions. 
Here we are interested in regularized online allocation problems, which add a non-separable regularizer to the objective function as a penalty for various purposes such as resource-saving, load balancing, diversity, and fairness \citep{ghosh2009bidding,balseiro2021regularized,celli2022parity}. 
Compared with non-regularized online resource allocation that maximizes an additively separable objective, non-separable regularization poses new challenges to algorithm design and regret analysis.  

In this paper, we study regularized online allocation problem with a, potentially non-concave, reward function and linear resource constraints under the so-called \textit{random input} model \citep{goel2008online} where i.i.d. requests arrive sequentially and follow an unknown distribution. Decisions must be made sequentially, that is, once a request is received with a known reward function, the decision maker shall instantly make an irrevocable decision based on the current request, previous history, and remaining resources. Throughout the paper, we impose hard constraints on the total resource consumption, which shall never be violated so that the decision-maker must wisely control the resource consumption at any time. Clearly, the challenges of online allocation problems mainly stem from the dilemma of fulfilling the current request or reserving the resources for, possibly more rewardable, future ones. The task for a decision maker is to design a strategy that maximizes the regularized total rewards subject to resource constraints. A typical application of the problem under study is online advertising \citep{mehta2007adwords,agrawal2018proportional} where a publisher needs to assign each impression to some advertiser 
 and maximize the click-through rate with budget constraints on each advertiser. Oftentimes, other aspects of resource consumption, including the fairness of advertisers or load balancing, are put into consideration. Towards that end, a regularizer on total click-through rates can be added, in which case the objective function turns out to be the regularized cumulative total click-through rates. 

Our main goal is to design computationally efficient algorithms for the aforementioned regularized online allocation problems, which, simultaneously, achieve theoretically optimal regrets. In the absence of a non-separable regularizer, it has been well recognized that the lower bound of regret of online allocation problems grows at a logarithmic rate \citep{bray2019does,li2021online}. The latter work also proposed adaptive policies that achieve the logarithmic-order regrets up to an additional $\log\log$ factor. Moreover, \cite{arlotto2019uniformly} shows that adaptive policies are, generally, necessary to make a low regret possible. In sharp contrast, to our best knowledge, regrets achieved by prior algorithms \citep{balseiro2021regularized} on regularized online allocation problems are of a square-root order. A first natural question is: can a logarithmic-order regret be achieved in the existence of a non-separable regularizer? 
Actually, we seek an even more ambitious goal: can we achieve a regret of exactly order $O(\log T)$ without the log-log factor so that the lower bound is sharply met?  The next question is more crucial: is there any computationally efficient algorithm that attains the desired regret?
Surprisingly, we give affirmative answers to both questions by designing an adaptive algorithm framework that is flexible, computationally fast, and theoretically guaranteed to achieve the sharply optimal regret. 
Extensive numerical simulations and real data experiments are presented to corroborate the effectiveness of our algorithms.

\subsection{Contributions}\label{sec:contributions}
To summarize, we make the following contributions in this paper. 

\textit{Sharp dual convergence in non-linear and regularized cases.} We provide two parallel approaches to derive the convergence rate of the empirical dual solution to its population counterpart in the case of additive non-linear rewards function and non-separable regularizer. By either localized Rademacher complexity or a partition argument, we show that the dual convergence rate is at $O(T^{-1})$, which improves the known rate $O(T^{-1}\log\log T)$ that was established only for non-regularized linear reward functions \citep{li2021online}. The improvement is made possible by a delicate analysis of the local behavior of the empirical dual program near the optimal solution. 
Different from \cite{balseiro2021regularized}, the proposed new dual form can depict the impact of constraint update, which is the key to our algorithm design. Our analysis also establishes a connection between the approximation errors measured by function values and the deviations of approximate solutions, which are determined by both intrinsic randomness and the approximation accuracy of solutions. It suggests that any approximate solution, up to a certain accuracy, to the dual optimization suffices to guarantee the overall convergence of a primal-dual algorithm, which lays the theoretical foundation for our history-dependent algorithm design. Notably, as a stochastic optimization problem, the derived dual convergence sheds new light on the widely studied Sample Average Approximation (SAA) and Empirical Risk Minimization (ERM) problems and may be of independent interest.

\textit{Adaptive algorithm framework with re-solving.} We propose a flexible dual-based and history-dependent, i.e., reliant on past data and actions, algorithm framework for solving the regularized online allocation problem. As a primal-dual algorithm featuring re-solving, each iteration mainly consists of two routines: primal decision-making and periodical dual optimization. 
At a high level, our adaptive algorithm framework generalizes the history-dependent policy in online linear programming \citep{li2021online}, which evolves from the budget-ratio policy \citep{arlotto2019uniformly,balseiro2019learning} and the {\it re-solving} heuristic in network revenue management \citep{jasin2012re,wu2015algorithms}. There are two key ingredients in the dual optimization of our algorithm framework. First, for each optimization problem, we adaptively update the average remaining resources in the dual problem. Besides fulfilling the resource constraints, this adaptive resource control plays an essential role in achieving a $O(\log T)$ regret rather than the $O(T^{1/2})$ one attained by \cite{balseiro2021regularized}. Second, our algorithm framework only requires an approximate solution to dual optimization, up to a certain accuracy. This allows a flexible choice of computationally efficient algorithms for dual optimization, be they deterministic or stochastic. Paired with first-order methods, our algorithm enjoys an acceptable polynomial-time cost comparable to prior algorithms. Moreover, we also develop the {\it infrequent} resolving technique that only requires solving dual optimizations for $O(\log T)$ times while achieving optimal regret and a fast adaptive dual gradient method with {\it linear} computation costs but a sub-optimal $O(\log^2 T)$ regret. Note that our algorithm framework is also applicable to linear reward functions or non-regularized online allocation problems. 


\textit{Regret analysis.} With its offline optimum as the benchmark, we investigate the regret attained by the adaptive algorithm framework for regularized online allocation problems. Since regret is characterized by dual convergence, the aforementioned new result allows us to derive a sharp regret bound. More exactly, we show that our adaptive algorithm achieves an $O(\log T)$ regret, which matches the best results in {\it constraint-free and non-regularized} online convex optimization \citep{hazan2007logarithmic} and multi-secretary problem \citep{bray2019does}. A matching lower bound is established under our assumptions demonstrating the optimality of our adaptive algorithm framework. To our best knowledge, this is the first theoretical guarantee of an exact $O(\log T)$ regret bound for online non-linear allocation with hard constraints and a non-separable regularizer. The best-known regret even for online linear programming \citep{li2021online} contains an additional $\log\log T$ factor. Distinct from previous work, our non-linear and non-separable primal problem introduces greater challenges when analyzing dual behaviors. By comparing with existing algorithms, we clarify the critical role played by the adaptive resource update in controlling the stopping time and achieving a logarithmic-order regret. In particular, we establish a worst-case $O(T^{1/2})$ lower bound for dual-based algorithms if the resource constraints are not adaptively updated. Basically, without updating the resource constraints, dual-based algorithms suffer from early stopping. 

We then elaborate on the applications of our method and theory to online linear programming, online max-min fairness allocation and online load balancing, etc. Simulation results will also be presented.



\subsection{Related Work}
\subsubsection{Online Linear Allocation}
Many online problems with resource constraints can be formulated into online allocation problems. A large proportion of early work mainly focused on linear models. \cite{vazirani2005adwords,mehta2007adwords,buchbinder2007online} studied the AdWords problem, where a search engine tries to assign some keywords to a set of competing bidders, each with a spending limit (i.e., constraint), and the goal is to maximize the revenue generated by these keyword sales. The rewards in AdWords problem are proportional to consumed resources and, thus, is a special case of online linear allocation. By viewing AdWords as a generalization of online bipartite matching problem, \cite{mehta2007adwords} achieved an optimal $(1-e^{-1})$-competitive ratio, which is defined as the ratio of the revenue of an online algorithm to the revenue of the best offline algorithm. 

Apart from AdWords, two major topics related to online linear allocation are online revenue management problems and online multi-secretary problems. In online revenue management, a decision maker aims to find a dynamic pricing policy that maximizes a company's linear total rewards. Among all the strategies for solving  online revenue management problems, re-solving stands out for its excellent performance.
By combining the re-solving strategy and a trigger-and-threshold mechanism, \cite{reiman2008asymptotically} reduced the regret from previously studied $O(T^{1/2})$ \citep{cooper2002asymptotic} to $O(T^{1/4})$. Equipped with sufficiently frequent re-solving's, \cite{jasin2015performance} proposed to re-estimate the parametric distribution of arrivals and proved that an $O(\log^2 T)$ regret is attained.  \cite{jasin2012re,wu2015algorithms} and \cite{bumpensanti2020re} investigated the special case when the i.i.d.  arrivals obey a  discrete distribution with finite support and established $O(1)$ regrets for re-solving style algorithms when the resource constraints are constants.  Online multi-secretary problem \citep{kleinberg2005multiple,babaioff2007knapsack} is one of the simplest online allocation problems as it has only one integer constraint. Assuming the arrivals obey a {\it known} finite-support discrete distribution, \cite{arlotto2019uniformly} proposed an online budget-ratio (BR) policy where decisions to fulfill or ignore requests are made by comparing the remaining average budget with some fixed thresholds. Their BR policy is adaptive and achieved an $O(1)$ regret but is inapplicable in the case of multiple resource constraints. They also established a regret lower bound $\Omega(T^{1/2})$ for all non-adaptive policies. Conversely, if the arrival distribution is continuous, e.g. a simple uniform distribution over $[0,1]$, \cite{bray2019does} developed a regret lower bound $\Omega(\log T)$ even when the distribution is known to a decision maker. Recent advances in contextual linear optimization \citep{hu2022fast} have demonstrated that this lower bound can be surpassed when instances exhibit more favorable properties, such as finiteness of the hypothesis class and higher-order smoothness, which are not applicable to our problem.

Other independent works of online linear programming also contribute greatly to the understanding of online allocation problems. \cite{agrawal2014dynamic} proposed a dual-based algorithm that dynamically updates dual variables and periodically solves linear programs which achieved an $O(T^{1/2})$ regret under the random permutation model. When the arrivals satisfy the random input model, \cite{devanur2019near} proved that a dual-based algorithm that attained an $O(T^{1/2})$ regret. But their algorithm relies on the knowledge of the optimal allocation, which is unrealistic for most applications. Otherwise, their algorithm requires frequent resolving offline linear programming. More recently, \cite{li2021online} introduced a history-dependent algorithm that adaptively updates the resource constraints, which achieved a regret $O(\log T\log\log T)$ that is almost optimal except the $\log\log$ factor. But their strategy also requires exactly solving an offline linear program of growing sizes with high frequency, which may be computationally intractable for large $T$. An $\Omega(\log T)$ regret lower bound was established, which is consistent with \cite{bray2019does}. 

Another noteworthy recent development in linear allocation, when the item distribution is known, is the study of greedy policies. Research by \cite{kerimov2021optimality} has shown that greedy policies can achieve near-optimality in two-way dynamic matching, while \cite{gupta2022greedy} demonstrated that these policies can produce bounded regret for multi-way matching. However, both of these studies rely on a general position condition for the deterministic approximate linear program, which is not required in our analysis. Furthermore, these greedy policies necessitate knowledge of the optimal solution for the deterministic equivalent, which is infeasible in our scenarios.  Importantly, the aforementioned works focus only on limited distributions with discrete support, while our research extends this scope to encompass continuous distributions. In addition, \cite{balseiro2023survey} introduced a unified framework known as dynamic resource-constrained reward collection and summarized a group of local notions of smoothness and strong convexity, which can be viewed as general cases for our required assumptions. 

\subsubsection{Online Convex Allocation}

Linear objective functions only find limited applications in practice. Online convex allocation moves one step further by allowing convex objective functions. In \cite{agrawal2014fast}, the authors investigated online convex programming that is equipped with a fixed and convex reward function. The imposed stochastic constraints are soft, meaning that a certain degree of constraint violations is allowed. 
Recently, partly due to its computational efficiency, dual mirror descent has been extensively studied for online convex allocation problems. \cite{balseiro2022best,balseiro2020dual} focused on a class of online allocation problems with separable reward functions and resource constraints proportional to time horizon $T$. They proposed a dual-based mirror descent algorithm acting on dual space that achieves $O({T}^{{1}/{2}})$ regret, which was said to be unimprovable under their assumptions. 
Dual mirror descent presents a self-correcting mechanism, which naturally prevents resources from depleting too fast. 
The problem we study in this paper is closer to \cite{balseiro2021regularized}, which is the first to study online convex allocation problems with a non-separable regularizer and hard resource constraints. Their approach is similar to the non-regularized cases \citep{balseiro2022best,balseiro2020dual}, except they define a new separable dual problem and update dual variables using regularized subgradient descent. They showed that, for regularized online convex allocation, dual mirror descent can still perform well and attain an $O(T^{1/2})$ regret. While this regret is optimal for general convex reward functions, it is sub-optimal when the reward functions possess more favorable conditions like strong convexity. 

It is worth briefly mentioning the literature on general online convex optimization, which laid the early foundations of online convex allocation problems. For strongly convex objectives, classical literature on online convex optimization has revealed an optimal logarithmic regret. See \cite{zinkevich2003online,hazan2007logarithmic} and references therein. It is reasonable to expect a logarithmic-order regret for other online problems in the existence of strong convexity. Nevertheless, achieving a logarithmic-order regret is challenging if an additional non-separable regularizer is posed. In literature, regularized online convex programming is commonly solved by the {\it follow-the-regularized-leader} style algorithms \citep{mcmahan2011follow,mcmahan2017survey}.  
Our dual-based adaptive algorithm differs from the follow-the-regularized-leader algorithms as it exploits more historical information rather than just the  gradients and past actions, and it does not follow the leader. More introduction for general online convex optimization can be found in \cite{hazan2016introduction}.

\subsection{Notations}
Some notations will be used throughout the paper. Define $a \wedge b: =\min\{a,b\}$ and $a \vee b: =\max\{a,b\}$. Write $[n]$ as the shorthand of $\{1, \ldots, n\}$. Define the non-negative region $\mathbb{R}_{+}:=\left\{x\middle|x\ge 0 \right\}$. We will always use $i$ to denote dimensions and use $d_i$ for the $i$-th dimension of vector $d$, and for vector sequence $\left\{d_t\right\}_{t=1}^{T}$, i.e., $d_{it}$ stands for the $i$-th entry of vector $d_t$. Denote $(x)^{+}:= \max\left\{x,0 \right\} $, $\|\cdot\|_2$ and $\|\cdot\|_{\infty}$ for the vector $\ell_2$-norm and $\ell_{\infty}$-norm, respectively. We write $\Tilde{O}(\cdot)$ as the big-O notation omitting the logarithmic factor.



\section{Regularized Online Allocation Problem}

We describe the regularized online allocation problem with finite time period $T$ as follows:
\begin{equation}
\label{eq:main_prob}
\begin{aligned}
&\max_{\{x_t, t\in[T]\}} \quad &&\sum_{t=1}^T f_t(x_t) +T\cdot r(\frac{\sum_{t=1}^T b_t x_t}{T})\\
&\text{s.t.} \quad &&\sum_{t=1}^T b_t x_t \preceq dT, \ d\in \mathbb{R}_+^{m}\\
& \quad &&  x_t \in \mathcal{X} , \forall t\in [T].
\end{aligned}
\end{equation}
where $f_t:\mathbb{R}^n \to \mathbb{R}$ is the reward function which may be potentially non-concave, $r :\mathbb{R}^m \to \mathbb{R}$ is a concave regularizer to penalize the average resource consumption,  $\ b_t\in \mathbb{R}^{m\times n}_{}$ is the cost matrix and its entry could be both positive or negative (i.e., we can replenish the resource). We assume our inputs are \textit{stochastic}, meaning that the i.i.d. requests $\{(f_t,b_t)\}_{t=1}^T$ are sampled from an unknown distribution $\mathcal{P}$: $(f_t,b_t)\sim \mathcal{P}$. The decision region $\mathcal{X} \subseteq \mathbb{R}^n_{+}$ is closed and potentially non-convex with void action $0\in \mathcal{X}$ . 
 
 
 Following the online sequential learning setting, we assume that at each time $1\le t \le T$, we first receive a request with known reward function and cost $(f_t,b_t)$ and then make the decision $x_t$ based on the observation of $t$-th request and history $\mathcal{H}_{t-1}:=\left\{f_j,b_j,x_j \right\}_{j=1}^{t-1}$:
 \begin{equation*}
 \label{eq:make_dicision}
 x_t:=A(f_t,b_t,\mathcal{H}_{t-1}),
 \end{equation*}
by taking the total resource constraints $\sum_{j=1}^t b_j x_j \preceq dT $ into consideration. Here $A$ denotes a history-dependent algorithm. 
Our goal is to design such an online algorithm $A$ that can maximize the regularized total reward $\sum_{t=1}^T f_t(x_t) +T\cdot r(T^{-1}\cdot\sum_{t=1}^T b_t x_t)$.   Define the algorithm expected reward over a given distribution $\cP$ as

\begin{equation}
    \label{eq:alg_risk}
   R\left(A\middle| \mathcal{P}\right):=\E_{A,\mathcal{P}} \left[ \sum_{t=1}^{T} f_t(x_t) +T\cdot r(\frac{\sum_{t=1}^T b_t x_t}{T}) \right].
\end{equation}
Here we take expectation with respect to both the inputs and the algorithm $A$ if $A$ is a stochastic algorithm. To measure the performance of an online algorithm, we compare the algorithm reward with the expected offline optimum (or hindsight optimum) defined by
\begin{equation}
    \label{eq:max_risk}
    R^*(\mathcal{P}):=\E_{\mathcal{P}} \left[\max_{x_t\in \cX } \sum_{t=1}^{T} f_t(x_t) +T\cdot r(\frac{\sum_{t=1}^T b_t x_t}{T}), \ s.t. \ \sum_{t=1}^T b_t x_t \preceq dT\right], 
\end{equation}
which serves as the benchmark performance. For a given $\cP$, define the \textit{regret} as $\operatorname{Regret}\left(A\middle|\cP\right):= R^*(\mathcal{P})-R\left(A\middle| \mathcal{P}\right)$.
We then define the \textit{worst-case regret} of an algorithm $A$ as the worst difference between the expected online reward and offline optimum over all the possible distributions in a certain probability family $\Xi$:
\begin{equation}
\label{eq:regret}    
     \operatorname{Regret}(A):=\sup_{\cP\in \Xi} \left\{R^*(\mathcal{P})-R\left(A\middle| \mathcal{P}\right)\right\}, 
\end{equation}
where the distribution family $\Xi$ will be identified later.

Compared with unconstrained online optimization, the key obstacle to designing algorithms for the online allocation problem is to enforce the total resource constraints, which shall not be violated at any time. We can transform the primal problem into a dual one with fewer constraints by the duality theory. This motivates us to investigate the problem \eqref{eq:main_prob} from the dual perspective. 
\subsection{The dual problem}
We consider the dual problem of online allocation \eqref{eq:main_prob}. The Lagrangian of this problem is
\begin{equation}
\label{eq:lagrangian}
    L(x,a,\lambda,\mu):=\sum_{t=1}^T f_t(x_t) +T\cdot r(a)+\mu^{\top}(aT-\sum_{t=1}^Tb_{t}x_{t})+\lambda^{\top}(dT-\sum_{t=1}^Tb_{t}x_{t}).
\end{equation}
Here we introduce the equality constraint $a=(\sum_{t=1}^{T} b_t x_t)/T$ in order to separate $r(T^{-1}\cdot\sum_{i=1}^{T}b_t x_t )$ into additive terms.
Denote the domain of $r(a)$ as $\cZ$ with  $ b\circ\mathcal{X} := \textsf{span}\left\{ b \cdot x \middle| \text{ for all possible } b \text{ and } x \in  \mathcal{X}\right\} \subseteq \cZ$. Define the conjugate function 
\begin{equation}
\label{eq:conjugate}
  \begin{aligned}
    f^*_t(\lambda):= \max\limits_{x\in \cX}\left\{ f_t(x)-x^\top \lambda\right\}, \ \
    r^*(\mu):=  \max\limits_{a \in \cZ}\left\{ r(a)-a^\top \mu\right\}.
\end{aligned} 
\end{equation}

Then, the dual problem of \ref{eq:main_prob} can be written as an additive form:
\begin{equation}
\label{eq:dual_problem}
    \begin{aligned}
    &\min_{\mu,\lambda\succeq 0} \quad &&\bar D_T(\lambda,\mu,d):= \frac{1}{T}\sum_{t=1}^T f_t^{*}(b_t^{\top}(\mu+\lambda)) +r^{*}(-\mu)+d^{\top}\lambda
    \end{aligned}
\end{equation}

Note that, we use two dual variables to pose our problem with $\mu$ for the impact of variable separation and $\lambda$ for the constraint. In contrast to \cite{balseiro2021regularized}, two dual variables enable us to capture the influence of both variable separation and variation of constraints $d$, wherein the latter is crucial for our scheme to be optimal.

Under our stochastic input assumption, \eqref{eq:dual_problem} can be viewed as a sample average approximation (SAA) \citep{shapiro2009lectures} of the following stochastic program (or \textit{fluid} problem):
\begin{equation}
\label{eq:dual_expected}
    \begin{aligned}
    &\min_{\mu,\lambda\succeq 0} \quad &&D(\lambda,\mu,d):=\mathbb{E} f_t^{*}(b_t^{\top}(\mu+\lambda)) +r^{*}(-\mu)+d^{\top}\lambda
    \end{aligned}
\end{equation}

In the following discussion, we will sometimes write $\nu=\lambda+\mu$ and write the dual variable uniformly as $\boldsymbol{\lambda}:=[\nu^{\top}, \mu^{\top}]^{\top} $ in shorthand with substitution. 
If we have known the exact offline solution to \eqref{eq:dual_problem}, denoted by $\blambda_T^*:=[\nu_T ^{*\top}, \mu_T ^{*\top}]^{\top} $, then by choosing the corresponding primal variables we can optimize the primal problem \eqref{eq:main_prob}. However, in the online setting, it is impossible to find such an exact dual solution before time $T$. Thus at time $t$ we turn to solve the $t$-sample average approximation of $D(\lambda,\mu,d)$, i.e., 
\begin{equation}
\label{eq:t_sample_approximate}
        \begin{aligned}
    &\min_{\mu,\lambda\succeq 0} \quad && \bar{D}_t(\lambda,\mu,d):=\frac{1}{t} \sum^{t}_{j=1} f_j^{*}(b_j^{\top}(\mu+\lambda)) +r^{*}(-\mu)+d^{\top}\lambda
    \end{aligned}
\end{equation}
and then use the dual approximate solution $\blambda_t^{\ast}$ to decide the following several primal solutions. Such a re-solving idea has shown its merit in controlling regret both in theory and in practice \citep{jasin2015performance,ferreira2018online,li2021online}. Hence we expect that this idea also works in regularized online allocation problems. Nevertheless, to discuss how practical this re-solving idea is in our setting, we still have three crucial questions to answer: 
\begin{enumerate}
    \item What is the behavior of $\blambda_T^*$ for large $T$ ? While $\blambda_T^*$ is random, from the stochastic programming perspective, as $T$ goes large, the optimal solution to the SAA \eqref{eq:dual_problem}, $\blambda_T^*$, will converge in probability to the solution to its stochastic program \eqref{eq:dual_expected}, denoted by $\blambda^*$. If we want to establish the theory of dual-based algorithms that rely on the approximation of $\blambda_T^*$, we need to first explore the convergence behavior of $\blambda_T^*$ toward $\blambda^*$.
    \item How will the dual approximate solutions affect reward and, consequently, regret? This question is the key to the algorithm design. For online allocation problems, a good approximation of $\blambda^*$ or $\blambda_T^*$ does not necessarily mean a good reward because of the restriction imposed by resource depletion and stopping time. As we will show later, simply solving the convex programming \eqref{eq:t_sample_approximate} is not enough to achieve the optimal regret. We explain the influence of dual approximation on regret in two phases: before and after stopping time, and show that the adaptive strategy of updating constraints is necessary for optimal regret.
    \item How to control the regret as well as make the algorithm computationally efficient? Most of the re-solving techniques require periodically solving potentially large-scale convex programming, which is computationally demanding. Interestingly, we will show that a proper approximation of dual optimal solutions up to certain precisions can significantly reduce computational costs while maintaining the optimal order of regret. The influence of our approximation scheme on the regret is, in general, negligible when compared with the exact optimal solutions.     
\end{enumerate}

We propose the following assumptions that suffice our algorithm to achieve logarithmic regret.

\subsection{Assumptions}
\begin{assumption}[Boundedness assumptions on arrivals]\label{asm:basic}
 The arrival sequences $\{(f_t,b_t)\}$ satisfy:
\begin{enumerate}[label= 1.\arabic*,ref= 1.\arabic*]
    \item \label{asm:basic:1} $\{(f_t,b_t)\}_{t=1}^T$ are generated i.i.d. from an unknown distribution $\mathcal{P}$.
    \item $f_t$ is defined in the closed decision region $\cX\subseteq \mathbb{R}^n_{+}$ with $\norm{x}_{\infty}\le D$ for any $x\in \cX$.
    \item There exists $\bar{f}\in R_+$ such that $\forall x \in \cX$, $|f_t(x)|\le \bar{f}$.
    \item There exists $\bar{b}\in R_+$ such that $\norm{b_{t}}_2 \le \bar{b}$ for any $t$.
    \item \label{asm:basic:5} We assume there exists $\underline{d}>0$, and a large $\bar{d}>0$ such that for any $i\in[m]$, $d_i \in (\underline{d},\bar{d})$. Denote $\Omega_d = \bigotimes_{i=1}^n (\underline{d},\bar{d})$.
\end{enumerate}
\end{assumption}

The assumptions on the upper bound $\bar f$ and $\bar b$ are  common and practical in online allocation problems. It helps us control the size of the problem and ease our analysis. Here we assume that the average resource constraints $d$ is of a reasonable size, i.e., $d_i$ is neither too large nor too small. If $d_i$ is too large, then the constraint itself will be of no interest because the restriction it imposed on the primal variables is negligible. This assumption is the basis for the subsequent discussion of regret, especially for bounding the stopping time.

Under Assumption \ref{asm:basic}, one general feasible region of our regularizer $r(a)$ is $\cZ := \left\{a \middle| \norm{a}_{2}\le \sqrt{n}D\bar{b} \right\}$,  which satisfies $b\circ\mathcal{X} \subseteq \cZ$.  We then describe the necessary assumptions on the regularizer $r$.


\begin{assumption}[Assumptions on the regularizer]\label{asm:regularizer}
The concave regularizer $r$  satisfies: $r$ is concave, closed, and bounded in $\cZ$: $|r|\le \bar r $ with bounded (sub)gradient $\norm{\nabla r (a)}_{\infty}\le G$ for any $a\in\cZ$.
\end{assumption}

The feasible region $\cZ $ here can also be chosen in other shape as long as  $b\circ\mathcal{X} \subseteq \cZ$. Together with Assumption \ref{asm:basic}, we can show that both the population-version and sample-version optimal solutions, $\blambda^*$ and $\blambda_T^{\ast}$, respectively, are uniformly bounded.

\begin{lemma}\label{lemma:bounded_solution}
Under Assumption \ref{asm:basic}, \ref{asm:regularizer}, the optimal solutions to problem  \eqref{eq:dual_problem} and  \eqref{eq:dual_expected} are bounded by:
\begin{equation}
    \begin{gathered}
    \norm{\lambda^*_T}_{\infty} \le \frac{2(\bar{f}+\Bar{r})}{\underline{d}}, \norm{\lambda^*}_{\infty} \le \frac{2(\bar{f}+\Bar{r})}{\underline{d}} \\
    \norm{\mu^*_T}_{\infty} \le G, \norm{\mu^*}_{\infty} \le G
    \end{gathered}
\end{equation}
\end{lemma}

By Lemma \ref{lemma:bounded_solution}, we define the regions that contain all the possible optimal dual variables as $\Omega_\lambda:= \left\{\lambda \middle| \norm{\lambda}_{\infty} \le \frac{2(\bar{f}+\Bar{r})}{\underline{d}}\right\}$, and $\Omega_\mu:= \left\{\mu\middle| \ \norm{\mu}_{\infty} \le G\right\}$. These regions will be the feasible sets of our dual variables since we do not want them to move far from the optimal solution $\blambda^*$. 
Assumption \ref{asm:regularizer} can be easily achieved by many popular regularizers enumerated below. 
\begin{enumerate}
    \item \textbf{$\ell_{1}$-loss}: $r(a):=-\kappa\norm{a}_1$. This regularizer serves as a tool to achieve a sparse resource allocation. 
    \item \textbf{Max-min loss}: $r(a):=\kappa \min_i ({a_i}/{d_i}) $. The  max-min fairness regularizer allows us to maximize the minimum resource consumption. Resources under max-min fairness regularization tend to be distributed fairly so that all resources are utilized adequately. See, e.g., \cite{nash1950,bertsimas2011price,bertsekas2021data}. 
    \item \textbf{Negative max loss}: $r(a):=-\kappa\max_i ({a_i}/{d_i})$. This regularizer represents the load-balancing task: we minimize the maximum resource consumption so that all the resources are evenly distributed and no resource is over-exploited (or balanced load for every computer server in the load-balancing task). The load-balancing regularizer is widely used in, e.g., network design and cloud computing \citep{bejerano2004fairness,radunovic2007unified}. 
    \item \textbf{Entropy loss}: $r(a):=-\kappa\left[\sum_{i=1}^{m} (a_{i}+\delta )\log \left(a_{i}+\delta\right)+\left(1-m\delta-\sum_{i=1}^{m} a_{i}\right) \log \left(1-m\delta-\sum_{i=1}^{m} a_{i}\right)\right]$ with the corresponding feasible region: $\cZ:=\left\{a \in \mathbb{R}_{+}^{m} \mid \sum_{i=1}^{m} a_{i} \leq 1-m\delta\right\}$. We use this entropy loss when our problem is related to stochastic strategies and probabilistic assignment, e.g., randomly assigning impressions to advertisers type $i$ with selected probabilities $a_i+\delta$ in online advertising. Here $1-m\delta-\sum_{i=1}^{m} a_{i}$ means the probability of no-assigning, and $\delta$ is the threshold of minimum assigning. This entropy loss regularizer seeks to find online allocation strategies with high entropy, which may share appealing properties like diversity, fairness, or robustness \citep{agrawal2018proportional}.
    \item \textbf{No regularizer}: $r(a):=0$. In this case, our problem reduces to the non-regularized online convex allocation problem. Therefore, the theory developed in this paper is immediately applicable to non-regularized cases.
\end{enumerate}

In addition, to achieve optimality, we need the following regularity and non-degeneracy assumption of $f^*_t$:

\begin{assumption}[Regularity conditions on the dual problem] \label{asm:nondegeneracy}  We assume that our problem is locally second-order growth and well-conditioned: suppose $(\lambda^*,\mu^*)$ is the optimal solution to the problem \eqref{eq:dual_expected} when $d\in \Omega_d$. 
Define $\nabla{f^*_t}$ as any (sub)gradient of ${f^*_t}$.  Then for any $d \in \Omega_d$, 
\begin{enumerate}[label=3.\arabic*,ref= 3.\arabic*]
    
    \item (locally second order growth)\label{asm:nondegeneracy:1} Let $\nu := \lambda+\mu$ and $\nu^*:=\lambda^*+\mu^*$. The conjugate function $f^*_t$ is continuous and satisfies
\begin{equation*}
\begin{gathered}
\E\left[\langle   \nabla{f^*_t}(b_t^\top \nu)-\nabla{f^*_t}(b_t^\top \nu^*),  b_t^\top \nu-b_t^\top \nu^*\rangle \middle| b_t\right] \ge \underline{\cL}_f \norm{b_t^\top \nu-b_t^\top \nu^*}^2_2 
\end{gathered}   
\end{equation*}
for any $\lambda\in \Omega_{\lambda}$, $\mu\in \Omega_{\mu}$ and constants  $\underline{\cL}_f>0$, conditioned on $b_t$.
    \item (well-conditioned)\label{asm:nondegeneracy:2} The matrix $M := \E\left[ b_t b_t^\top \right]$ is positive definite with minimum eigenvalue $\sigma_{\min}>0$.
\end{enumerate}
\end{assumption}

Assumption \ref{asm:nondegeneracy:1} requires the expected conjugate of reward function to exhibit a local quadratic growth, conditioning on any given $b_t$. 
Assumption \ref{asm:nondegeneracy:1} controls the growth rate of dual function so that it will not degenerate to a line, which is necessary for characterizing dual solutions. 
Assumption \ref{asm:nondegeneracy:2} is easily satisfied since, oftentimes, the constraints are linearly independent. Note that $-\nabla f_t^{\ast}(b_t^{\top}\nu)=\mathop{\arg\max}_{x\in \cX}\{f_t(x)-(\nu)^\top b_t x\}$ represents the primal solution given dual variable $\nu$. Its randomness stems from the stochastic reward $f_t$. From this perspective, Assumption ~\ref{asm:nondegeneracy:1} only concerns the effect of dual variables on their corresponding {\it expected} primal solutions. However, when it comes to the smoothness, merely the perturbation behavior of expected primal solutions is not sufficient for our analysis, and we also need the perturbation behavior of the intrinsically {\it random} primal solutions, which can be controlled by moments. The following assumption serves this purpose. Equivalently, it depicts the variation behavior of the random award function $f_t$.  This moment assumption establishes the connection between dual variables and primal performances. Note that Assumption \ref{asm:nondegeneracy} only concerns the deterministic problem \eqref{eq:dual_expected}, but the empirical problem does not necessarily share these local properties or is even not differentiable. 

\begin{assumption}[Smoothness of moment]
\label{asm:k-thmoment}
Let $\nu := \lambda+\mu$ and $\nu^*:=\lambda^*+\mu^*$ when we choose $d\in\Omega_d$ in \eqref{eq:dual_expected}. For any random variable $V\in \mathbb{R}^m$ that satisfies $\norm{b_t^\top(V-\nu^*)}_2\le \norm{b_t^\top(\nu-\nu^*)}_2 $ a.s., the $1$-th order moment of the (sub)gradient $\nabla f_t^*$ satisfies the following smoothness 
$$\E_{}\left[ \norm{\nabla f_t^*(b_t^\top V)-\nabla f_t^*(b_t^\top\nu^*)}_2\middle| b_t\right]\le L_1 \norm{b_t^\top(\nu-\nu^*)}_2.
$$
for any $d\in \Omega_d$, $\lambda\in \Omega_{\lambda}$, $\mu\in \Omega_{\mu}$ and given $b_t$, where $L_1>0$ is a constant. 
\end{assumption}
  Assumption \ref{asm:k-thmoment} requires the variation of reward function given $b_t$: $f_t\sim\cP|b_t$ to be mild so that the primal solution changes smoothly.  This doesn't mean that $ f^*_t$ must be globally smooth. The expectation here is with respect to $f_t$, $\boldsymbol{\nu}$. A similar description of smoothness can be found in \cite{gorbunov2020unified}. 
Basically, Assumption \ref{asm:k-thmoment} claims that no matter how the reward $f_t$ varies, the difference of primal variables  can be bounded by the difference of dual variables in expectation. 
We note that Assumptions \ref{asm:regularizer}-\ref{asm:k-thmoment} assume the corresponding conditions hold for all the $d\in\Omega_d$.  

\begin{remark} 
Here we list several sufficient conditions, any of which can lead to both Assumptions \ref{asm:nondegeneracy} and \ref{asm:k-thmoment}:
\begin{enumerate}
    \item  $f_t$ is linear with reward $v_t\in \mathbb{R}^m$, i.e., $f_t=v_t^\top x_t$. and for each $i\in [m]$,  $v_{it}$ is with distribution $\abs{\PP(v_{it}>b_{it}^\top \nu|b_t )-\PP(v_{it}>b_{it}^\top \nu^*|b_t )} =\Theta (\abs{b_{it}^\top(\nu-\nu^*)} ) $. 
    \item $f_t$ is  drawn from a finite distribution wherein every possible $f_t$ is locally smooth and strongly concave.
    \item every $f_t$ is concave and continuous with first order growth gradient, and $\E f_t(x)$ admits a lower upward quadratic (LUQ) envelope in a local region near the optimal solution \citep{balseiro2021survey}.
\end{enumerate}
We will revisit the linear case in Section \ref{sec:application} for a detailed discussion. For more possible sufficient conditions, we refer the reader to \cite{kakade2009duality,bubeck2015convex,balseiro2021survey}, etc.
\end{remark}

Define the primal variable given $(\lambda^*,\mu^*)$ as $\tilde{x}_t(\nu^*):=-\nabla f^*_t(b_t^{\top}(\nu^*))$. In this sequel, all the dimensions that satisfy $ d_i-\E\left(b_t\tilde{x}_t(\nu^*)\right)_i=0$ with respect to the original $d$ in (\ref{eq:main_prob}) are referred to as {\it binding dimensions}. Denote $I_{\textsf{B}}= \left\{ i\middle| d_i-\E\left(b_t\tilde{x}_t(\nu^*)\right)_i=0\right\}$ the collection of binding dimensions. Similarly, {\it non-binding dimensions} are written as $I_{\textsf{NB}}= \left\{ i\middle| d_i-\E\left(b_t\tilde{x}_t(\nu^*)\right)_i>0\right\}$. Here for ease of notation, we omit the dependence of $I_{\textsf{B}}$ and $I_{\textsf{NB}}$ on the resource constraint $d$. 
\begin{assumption}[Non-degeneracy]\label{asm:nondegeneracy:3}
Let $(\lambda^{\ast}, \mu^{\ast})$ be the optimal dual solution given the original $d$ in (\ref{eq:main_prob}). Denote $d^*=\E\left(b_t\tilde{x}_t(\nu^*)\right)$. Then:

\begin{enumerate}[label=5.\arabic*,ref= 5.\arabic*]
    
    \item \label{asm:nondegeneracy:3-1} The optimal solution $(\lambda^*,\mu^*)$ satisfies $\lambda^*_i=0$ if and only if $i\in I_{\textsf{NB}}$, i.e., $d_i-d^*_i>0$. 
    \item \label{asm:nondegeneracy:3-2} Further, if $i\in I_{\textsf{NB}} $, then there exists an small constant $\delta_0$ such that the partial gradient $\abs{\nabla_i r(a)-\nabla_i r(d^*) }\le \Bar{\cL}_r\norm{a-d^*}_2$ for any $\norm{a-d^*}_2\le \delta_0$.
\end{enumerate}

\end{assumption}

 Assumption \ref{asm:nondegeneracy:3}  states the non-degeneracy condition for dual problems with nonlinear objectives and the regularizer, which is generalized from the non-degeneracy condition of linear programs  \citep{jasin2012re,jasin2015performance,wu2015algorithms,li2021online}. Assumption \ref{asm:nondegeneracy:3-1} imposes strong complementary slackness on the resource constraints $d\in \Omega_d$ uniformly. This suggests that when $d$ changes within a certain region of $\Omega_d$, the binding or non-binding dimensions of resource constraints for the optimal solution will not change. 
This brings convenience for analyzing adaptive algorithms with frequently updated constraints. 
 Assumption \ref{asm:nondegeneracy:3-1} ensures that binding and non-binding dimensions can be uniquely determined by the dual solution  $\blambda^*$. In our study, Assumption \ref{asm:nondegeneracy:3-1} is indispensable for the regret analysis because it allows the gap between the fluid benchmark and offline maximum to be well controlled.  Assumption \ref{asm:nondegeneracy:3-2} requires the dual optimal solution $\mu^*_i=-\nabla_i r(d^*)$ to be non-degenerate in the non-binding dimension $i\in I_{\textsf{NB}}$: it is unique and smooth with the change of resources in a tiny region near $d^*$. This can be achieved by the aforementioned regularizers as long as $d^*$ is in a good position, e.g., for the max-min loss, we only require  $d^*$ to have a unique minimum dimension.

\section{Dual Convergence}\label{sec:dual_conv}
  For all dual-based online algorithms, the finite-sample convergence rate of dual variables is of great value since it reveals the best performance dual-based algorithms can achieve compared to the deterministic optimum. Recall the optimal solution $\blambda_T^{\ast}$ to the sample average approximation (SSA) in eq. (\ref{eq:dual_problem}). The Law of Large Numbers dictates that $\blambda_{T}^{\ast}$ converges to $\blambda^{\ast}$ in probability as $T\to\infty$. While  the asymptotic behaviors of optimal solutions to SAA have been intensively studied in the literature \citep{kleywegt2002sample,shapiro2009lectures,kim2015guide}, they are not enough for us to develop the non-asymptotical dual convergence in the case of regularized online convex programming.  In this section, we establish the dual convergence bounds under Assumptions \ref{asm:basic}-\ref{asm:k-thmoment}, for the regularized online problem (\ref{eq:main_prob}). We mainly study the convergence of $\norm{\nu^*_T-\nu^*}_2$ since the primal solution is only related to $\nu$, not individually by $\lambda$ or $\mu$. All the theories can be easily extended to the joint convergence $\norm{\blambda^*_T-\blambda^*}_2$ when $r^*$ is also strongly convex. We emphasize that our assumptions hold uniformly for all $d'\in \Omega_d$. Consequently, the dual convergence performance we will derive in this section also holds for all $d'\in \Omega_d$. Here we provide two parallel approaches that can derive the optimal dual convergence rate. One is by localized Rademacher complexity and the other is by partition. Both feature a similar localization idea, which is critical to achieving fast rate $O(T^{-1})$.  
Define $D_t(\blambda,d):=f_t^{*}(b_t^{\top}\nu) +r^{*}(-\mu)+d^{\top}(\nu-\mu)$, and the corresponding (sub)gradient 
$$
\phi_t(\blambda,d):=\nabla D_t(\blambda,d)=\left[\begin{array}{c} b_t\nabla f_t^{*}(b_t^{\top}\nu)+d\\
-\nabla r^{*}(-\mu)-d\end{array}\right].
$$

Then we have $\nabla D(\blambda,d)=\E \phi_t(\blambda, d)$.  
Denote $\Bar{\phi}_T(\blambda,d):=T^{-1}\sum_{t=1}^{T}\phi_t(\blambda,d)$ and the partial derivative w.r.t $\nu$ as $\Bar{\phi}_{T,\nu}(\nu,d)=T^{-1}\sum_{t=1}^{T}b_t\nabla f_t^{*}(b_t^{\top}\nu)+d$ . Both of two approaches focus on the partial second order term of $\Bar{D}_T(\blambda,d)$:
\begin{equation}\label{eq:decomp}
\begin{aligned}
       \Bar{s}_T(\nu,d) & := \Bar{D}_T(\nu, \mu^*,d)-\Bar{D}_T(\nu^*, \mu^*,d) - \underbrace{\left\langle\Bar{\phi}_{T,\nu}(\nu^*,d), \nu -\nu^*\right\rangle}_{\text{first order term}} \\
\end{aligned}
\end{equation}
We may write $\bar{s}_T(\nu)$ for simplicity and $s_t(\nu)$ is defined in a similar spirit. 

\subsection{Localized Rademacher complexity}
We adopt the idea of localization in local Rademacher complexity \citep{bartlett2005local,koltchinskii2006local} to derive a tight probability bound of $\norm{\nu^*_T-\nu^*}$. Assume $\nu^*_T$ is part of the variable that minimizes \eqref{eq:dual_problem}, and suppose $\norm{\nu^*_T-\nu^*}\ge \varepsilon$. Since the expected second order term $s(\nu):= \E \Bar{s}_T(\nu) $ shares a second-order growth property by Assumption \ref{asm:nondegeneracy}, then by \eqref{eq:decomp} and convexity of $\Bar{D}_T(\blambda,d)$, there exists a $\blambda=[\nu^\top,\mu^\top]^\top$ where $\nu\in \mathbb{B}(\nu^*, \varepsilon)$ such that
\begin{equation}\label{eq:sT-ub}
    \begin{aligned}
 \Bar{s}_T(\nu)- s(\nu) & \le 
\Bar{D}_T(\blambda,d)-\Bar{D}_T(\blambda^*,d) - \left\langle\Bar{\phi}_T(\blambda^*,d), \blambda-\blambda^*\right\rangle - s(\nu)+ \left\langle \nabla D(\blambda^*,d), \blambda-\blambda^*\right\rangle  \\
& \le -\frac{ \sigma_{\min} \underline{\cL}_f }{2}\varepsilon^2+  \norm{\nabla_{\nu} D(\blambda^*,d)-\Bar{\phi}_{T,\nu}(\nu^*,d)}\varepsilon,
\end{aligned}
\end{equation}
where we use convexity and the optimality of $\blambda^*$ for the first inequality and the optimality of $\blambda^*_T$ for the second one.
By concentration, it is clear that, for any $\varepsilon>0$,  the gradient $\Bar{\phi}_{T,\nu}(\nu^*,d)$ concentrates to $\nabla_{\nu} D(\blambda^*,d)$ with error upper bounded by $\frac{\sigma_{\min}\underline{\cL}_f }{4}\varepsilon$  with high probability. We can then ensure the empirical process: 
$$\sup_{\nu\in \mathbb{B}(\nu^*, \varepsilon)}\abs{\Bar{s}_T(\nu)- s(\nu)}\ge \frac{\sigma_{\min}\underline{\cL}_f }{4}\varepsilon^2 $$
with high probability. Define localized Rademacher complexity of $\Bar{s}_T$ within a small neighbourhood of $\nu^*$ as $\mathcal{R}_\varepsilon = \E_{\cP} \E_\sigma \left[ \sup_{\nu\in \mathbb{B}(\nu^*,\varepsilon) } \frac{1}{T}\sum_{t=1}^{T} \sigma_t s_t(\nu)  \right]$, where $\sigma_t$ are independent Rademacher random variables. By the convergence theory of empirical process \citep{boucheron2005theory,koltchinskii2011oracle}, we have the follow proposition.
\begin{Proposition}\label{prop:rademacher-prob}
Under Assumption \ref{asm:basic}-\ref{asm:k-thmoment}, the following inequality holds
\begin{equation*}
     \mathcal{R}_\varepsilon \le \sqrt{2m \log(3K) }\frac{2\sqrt{n}\Bar{b}D\varepsilon}{\sqrt{T}} +\frac{{L_1} \varepsilon^2 }{K},
\end{equation*}
for any constant $K>0$. Consequently, if $\varepsilon \ge \frac{64\sqrt{2}\sqrt{n}\Bar{b}D }{\sqrt{T} \sigma_{\min}\underline{\cL}_{f} } \sqrt{\log \frac{100{M} }{\sigma_{\min}\underline{\cL}_{f}} }$, we have the following probabilistic bound:
\begin{equation*}
    \PP\left( \norm{ \nu^*_T-\nu^*}\ge \varepsilon \right)\le m\exp\left({-\frac{T\sigma_{\min}^2\underline{\cL}_{f}^2\varepsilon^2  }{8mn\Bar{b}^2D^2}}\right)+ \exp\left({-\frac{T\sigma_{\min}^2\underline{\cL}_{f}^2\varepsilon^2  }{5000 n\Bar{b}^2D^2}}\right)
\end{equation*} 

\end{Proposition}

\subsection{Partition}
Apart from localized Rademacher complexity, we can also control the dual convergence by providing a uniform lower bound of $\Bar{s}_T$ within a small neighborhood of $\nu^*$ by partition. This argument will show that, with high probability, the empirical second order term $\Bar{s}_T(\nu,d)$ is lower bounded by a quadratic function \citep{li2021online}. 
The rationale is obvious. 
Since $\bar s_T(\nu,d)$ is always convex and converges to a deterministic convex function, the shape of $\bar s_T(\nu,d)$ in a small neighborhood of $\nu^*$ will be very close to $ s(\nu,d)$ as long as $T$ is large enough. Moreover, for a convex function, the local behavior near the deterministic optimal solution is enough to guarantee the global
properties of empirical optimal solutions. Consequently, its global optimal solution $\nu^*_T$ will lie in a small neighbourhood of  $\nu^*$. This delicate analysis also adopts a localization technique, which enables us to reach an optimal result sharper than \cite{li2021online}.


To investigate the second-order term, we focus on a small neighborhood of $\nu^*$. For a constant $H>0$ (to be clarified soon), define $\Omega_{\nu}(\varepsilon):=\left\{\nu \middle| \norm{\nu-\nu^*}_{\infty}\le 4H\varepsilon\right\}$. Actually, it suffices to control the second order term for all dual variables in $\Omega_{\nu}(\varepsilon)$ since we shall show that $\nu_T^{\ast}$ belongs to $\Omega_{\nu}(\varepsilon)$ with a high probability. 
In order to control the shape of second order term in $\Omega_{\nu}(\varepsilon)$, we systematically split the region $\Omega_{\nu}(\varepsilon)$ and derive a uniform concentration bound. 
This enables us to successfully eliminate the O($\log \log T$) factor and achieve a sharper dual convergence bound. 

\begin{Proposition}\label{prop:dual_conv_prob}
Under Assumptions \ref{asm:basic}-\ref{asm:k-thmoment}, we define $H=10\sqrt{n m\log m}\Bar{b}D/(\sigma_{\min} \underline{\cL}_f)$. Then, given any $\varepsilon>0$, the second order term $\Bar{s}_T$ satisfies that for $\forall \nu\in \Omega_{\nu}(\varepsilon)$ and $\norm{\nu-\nu^*}_2> 2H\varepsilon$ , there exists a corresponding $\nu'\in\Omega_{\nu}(\varepsilon)$ such that $\norm{{\nu}'-\nu^*}_2\ge \norm{\nu-\nu^*}_2$ and 
\begin{equation*}
    \begin{aligned}
         \Bar{s}_T(\nu,d)\ge \frac{\sigma_{\min} \underline{\cL}_f}{4} \norm{ \nu'-\nu^*}_2^2 -\frac{2}{5}\sigma_{\min} \underline{\cL}_f H\varepsilon \norm{ \nu'-\nu^*}_2
    \end{aligned}
\end{equation*}
with probability at least $1-2\exp({-\frac{m(T\varepsilon^2-1)\log m }{4}})$. Consequently, under the event that this inequality holds, if there exists a $\blambda^*_T$ that minimizes $\Bar{D}_T$, then it follows that $\norm{\nu^*_T-\nu^* }\le 2H\varepsilon$ with probability at least $1- 2m\exp (-\frac{T\sqrt{n}\Bar{b}D\varepsilon^2 \log m}{2 })$.

 \end{Proposition}

\subsection{Convergence result}
 Equipped with Proposition \ref{prop:rademacher-prob} or \ref{prop:dual_conv_prob}, we can derive the following $O(T^{-1})$ bound for dual convergence. 
\endproof
\begin{theorem}[Dual convergence]\label{thm:dual_conv}
Under Assumptions \ref{asm:basic}-\ref{asm:k-thmoment}, there exists an absolute constant $C_1>0$ such that the empirical dual optimal solution satisfies 
\begin{equation}
\begin{aligned}
    \E \norm{\nu^*_{T}-\nu^*}^2_2 \le & C_1 \frac{\Bar{b}^2 D^2}{\sigma_{\min}^2 \underline{\cL}_f^2} \frac{nm\log m }{T}.
\end{aligned}
\end{equation}
\end{theorem}

\begin{remark}
Our dual convergence bound is sharper than that in \cite{li2021online}. 
Under our assumption, the $O(T^{-1})$ rate is unimprovable, as there always exists a distribution $\mathcal{P} \in \Xi$ that incurs an $\Omega(T^{-1})$ dual convergence rate. For further details, please refer to Appendix \ref{proof:dual_conv}. The bound is of order $\widetilde{O}(mn)$ with respect to dimension $n$ and constraint number $m$. This can be obtained by applying either Proposition \ref{prop:rademacher-prob} or Proposition \ref{prop:dual_conv_prob}. However, these two approaches offer different advantages for individuals with various purposes. Proposition \ref{prop:rademacher-prob} can be easily extended to other related classical statistical problems, such as ERM or the convergence of M-estimator, which are of concern in general statistics and machine learning communities. On the other hand, the result in Proposition \ref{prop:dual_conv_prob} is stronger for optimization and can be adapted to study the behavior of other optimization methods, such as the convergence $\epsilon$-optimal solution, an important result we will discuss later.
\end{remark}

Note that our Proposition \ref{prop:dual_conv_prob} holds uniformly for all $d'\in\Omega_d$. Denote the optimal solutions to problem (\ref{eq:dual_problem}) and (\ref{eq:dual_expected}), given a certain $d'$, by $\nu_{T}^*(d')$ and $\nu^*(d')$, respectively. Then, we actually have 
\begin{equation}\label{eq:uniform_dual_conv}
    \begin{aligned}
         \E \sup_{d'\in\Omega_d} \norm{\nu^*_{T}(d')-\nu^*(d')}^2_2\le C_1 \frac{\Bar{b}^2 D^2}{\sigma_{\min}^2 \underline{\cL}_f^2} \frac{nm\log m }{T}
    \end{aligned}
\end{equation}
Bound (\ref{eq:uniform_dual_conv}) plays a critical role in our regret analysis since the re-solving strategy of our adaptive algorithm framework needs to update the resource constraints.

We then discuss $\epsilon$-optimal solutions of dual problem \eqref{eq:dual_problem}. Our following finite-sample convergence result of $\epsilon$-optimal solution can be viewed as a non-parametric version of SAA convergence developed by large deviation theory \citep{ruszczynski2003stochastic}. Notably, we only make assumptions on the deterministic problem $D(\blambda,d)$ and the smoothness of moment, and our result does not rely on restricted tail conditions such as the moment generating function in \cite{ruszczynski2003stochastic,shapiro2009lectures}. Therefore our result allows more flexible distributions.  

\begin{theorem}[Convergence of dual approximate solution]\label{thm:dual_conv_approx}
Under Assumptions \ref{asm:basic}-\ref{asm:k-thmoment}, suppose $\blambda_T^{\epsilon}$ is an  $\epsilon$-optimal solution that satisfies $\bar{D}_T(\blambda_T^{\epsilon},d)-\bar{D}_T(\blambda_T^{*},d)\le \epsilon$. Then we have the following convergence of $\epsilon$-optimal solution: 
\begin{equation*}
    \begin{aligned}
         \E\norm{\nu_T^{\epsilon}-\nu^{*}}_2^2\le C_1 \frac{\Bar{b}^2 D^2}{\sigma_{\min}^2 \underline{\cL}_f^2} \frac{nm\log m }{T}+ \frac{8\epsilon}{\sigma_{\min}\underline{\cL}_f }
    \end{aligned}
\end{equation*}
\end{theorem}

Theorem \ref{thm:dual_conv_approx} explains how the approximation of dual solutions affects dual convergence. The accuracy remains valid as we directly optimize the deterministic dual function $D(\blambda,d)$. Moreover, this theorem reveals that even if the empirical dual function $\Bar{D}_T(\blambda,d)$ is not strongly convex or smooth, the dual convergence of approximate solution also holds as long as we choose an appropriate accuracy. We can further show that this property is preserved with a slightly different accuracy if we run stochastic optimization algorithms on $\Bar{D}_T$. We describe the convergence of stochastic approximate solution in the following corollary:

\begin{corollary}[Convergence of stochastic dual approximate solution]\label{cor:stocha_dual_conv}
Under Assumptions \ref{asm:basic}-\ref{asm:nondegeneracy}, suppose $\blambda_T^{\epsilon}$ is a stochastic  $\epsilon$-optimal solution generated by stochastic optimization algorithm $\cB$ that satisfies $$\E_{\cB}\left[\bar{D}_T(\blambda_T^{\epsilon},d)-\bar{D}_T(\blambda_T^{*},d)\middle| \Bar{D}_T\right]\le \epsilon$$ for any given $\Bar{D}_T$. Then we have the following convergence of the stochastic $\epsilon$-optimal solution: 
\begin{equation*}
    \begin{aligned}
         \E\norm{\nu_T^{\epsilon}-\nu^{*}}_2^2\le C_2 \frac{\Bar{b}^2 D^2}{\sigma_{\min}^2 \underline{\cL}_f^2} \frac{nm\log m }{T}+ C_3 \epsilon^{\frac{2}{3}}\left(m\left(2\frac{\Bar{f}+\Bar{r}}{\underline{d}} +G\right)\right)^{\frac{1}{3}}/ (\sigma_{\min}\underline{\cL}_f )^{\frac{2}{3}} ,
    \end{aligned}
\end{equation*}
where the expectation is taken with respect to $\cB$ and $\cP$, and $C_2$, $C_3$ are absolute constants.
\end{corollary}

Corollary \ref{cor:stocha_dual_conv} points out that the impact of stochastic optimization on the dual convergence is limited, and the order of dual convergence can still be controlled by $\epsilon$. Compared to Theorem \ref{thm:dual_conv_approx},  the smaller order $\epsilon^{\frac{2}{3}}$ could be viewed as the accuracy loss because of randomness. Even if we do not assume $\bar D_T$ to be strongly convex, the difference between stochastic solutions and the deterministic one $\E\norm{\nu_T^{\epsilon}-\nu_T^{*}}_2^2$ is still under control just as we optimize a strongly convex function. This inspires us to apply the stochastic approximate solutions to the re-solving heuristic because, in many contexts, the benefits of stochastic algorithms greatly outweigh the lower order of convergence $\epsilon^{\frac{2}{3}}$. Theorem \ref{thm:dual_conv_approx} and Corollary \ref{cor:stocha_dual_conv} are all based on Proposition \ref{prop:dual_conv_prob}. With the theory of dual convergence, we are ready to describe our dual-based algorithm framework for online allocation.

\section{Algorithm Framework}\label{sec:algorithm}
Our algorithm extends the linear adaptive re-solving strategy in \cite{li2021online} to non-linear and non-separable objective functions. The key idea is similar to the frequent re-solving strategy in network revenue management (e.g., \cite{jasin2012re,bumpensanti2020re}) in spirit. We keep re-solving dual problems with updated average remaining capacity inspired by the budget-ratio policy \citep{arlotto2019uniformly}. Compared to the re-solving strategy in network revenue management, we also need to keep updating the constraints and re-solving the associated optimization programs. But the difference is that our strategy is dual-based, and the size of our optimization problems grows with time. Fortunately, the optimization in our algorithm can be easier as we only need approximate solutions. The resource control in our algorithm is handled more carefully when compared with the simple dual mirror descent. We show that, non-adaptive policies are too greedy and can't wisely keep the remaining budget balanced in the long run. 

Our dual-based online allocation algorithm is in line with other dual-based online algorithms in spirit: we keep maintaining a dual variable $\blambda_{t}$ by re-solving and whenever a request comes, we instantly give a response based on the dual variable and the request just received. We choose the primal action $x_t$, and $a_t$ given the dual variables by: 
\begin{align*}
    \tilde{x}_t(\nu) &:= \mathop{\arg\max}_{x\in \cX}\{f_t(x)-(\lambda+\mu)^\top b_t x\} = -\nabla f^*_t(b_t^\top(\lambda+\mu)), \\
    \tilde{a}(\mu)&:=\underset{a\in\mathcal{Z}}{\arg\max}\ \big\{r(a)+\mu^{\top}a \big\}=-\nabla r^*(-\mu_{})
\end{align*}
Note that $\nu$ consists of both $\lambda$ and $\mu$ with different feasible sets, and can not be directly yielded by optimization with respect to a single variable. The primal solution $a$ may not explicitly affect our action $x_t$, but it is helpful for our theoretical analysis of dual-based policies and for  algorithm implementation.

We outline our dual-based and history-dependent algorithm framework in Algorithm~\ref{alg:framework}. The algorithm updates  dual variables by solving SAA problems as shown in equation \eqref{eq:adaptive_SAA}. Each $\blambda_t$ is a $\epsilon_t $-optimal solution of the $t$-sample SAA with adaptive resources constraints $d_t$. We emphasize that two ingredients in our algorithm framework are crucial to guarantee an $O(\log T)$ regret: 
(1) the adaptive update of resource constraints $d_t$;  (2) the careful choice of accuracy $\epsilon_j$ for approximate dual solutions. Without the adaptive update of $d_t$, the worst-case regret will never be optimal for some extreme cases (see Section \ref{sec:regret} for more discussion). The dual solution accuracy can be set as either increasing $\epsilon_t=\Theta( t^{-1})$  or decreasing $\epsilon_t=\Theta((T-t)^{-1})$  (or $\epsilon_t=\Theta(t^{-3/2})$, $\epsilon_t=\Theta((T-t)^{-3/2})$  for stochastic optimization algorithms). Approximate solutions help significantly alleviate the total computational cost. 
Our algorithm is history-dependent, meaning that we exploit all the information we have collected up to time $t$. This is the essence of our adaptive strategy. This history-dependent policy makes our algorithm learn more efficiently than other dual-based algorithms that do not learn from history \citep{devanur2019near,balseiro2022best}, at the cost of acceptable extra computation. As is common in the literature on dual-based online algorithms,  we assume that the conjugate $f_t^{\ast}$ and corresponding primal variable $\tilde x_t$ are easily attainable. 

\begin{algorithm}
\caption{History-based resolving algorithm framework}
\label{alg:framework}
\begin{algorithmic}
\REQUIRE regularizer $r$, iteration number $T$, start point $(\lambda_0,\mu_0)$, and initial resource $B_0:=dT$.
\FORALL{$t = 1,\dots,T$}
    \STATE{Receive $(f_t,b_t)\sim \cP$.  }
    \STATE{Calculate $\tilde{x}_{t} :=\tilde{x}_t(\nu_{t-1}):= \mathop{\arg\max}_{x\in \cX}\{f_t(x)-(\lambda_{t-1}+\mu_{t-1})^\top b_t x\} = -\nabla f^*_t(b_t^\top(\lambda_{t-1}+\mu_{t-1}))$.
}
    	\STATE{Select $ x_t:=\begin{cases} \tilde{x}_t&\text{ if }B_{t-1}\ge b_t x_t  \\ 0 &\text{ otherwise }  \end{cases}$ }
	\STATE{Update remaining resources: $B_t:= B_{t-1}-b_t x_t$ and average remaining resources: $d_t:= \frac{B_t}{T-t}$}
    	\STATE{Update dual variable $(\lambda_t,\mu_t)$ via solving the following dual problem by any approximation algorithm $\cB_t$ with accuracy $\epsilon_t$: 
	\begin{equation}\label{eq:adaptive_SAA}
	    \min_{(\lambda,\mu)\in \Omega_{\lambda}\times\Omega_{\mu} } \left\{ \Bar{D}_{t} (\lambda,\mu,d_t):=\frac{1}{t} \sum^{t}_{l=1} f_l^{*}(b_l^{\top}(\mu+\lambda)) +r^{*}(-\mu)+d_t^{\top}\lambda \right\}
	\end{equation}  }
\ENDFOR
\end{algorithmic}
\end{algorithm}

Our algorithm framework is free of optimizer, that is, we can select any optimizer to get the $\epsilon_t$-optimal solution to dual program \eqref{eq:adaptive_SAA}. Since the dual problem $\Bar{D}_t(\blambda,d_t)$ is generally convex with respect to $(\lambda,\nu)$, one favorable choice is stochastic gradient descent that is first order (recall that we assume the gradient of the dual problem, i.e., the primal variable, is easily attainable) and the computational complexity can be free of size $t$. This makes it possible to deal with large-scale dual optimization when the total running time $T$ is large.

More specifically, if the dual optimizer is selected as stochastic gradient descent where the accuracy is specified by $\epsilon_t:=ct^{-3/2}$, we end up with the following Algorithm \ref{alg:SGD} by our algorithm framework. Basically, it requires computing $O(t^3)$ stochastic gradients at time $t$.  Moreover, if the dual problem $\Bar{D}_t$ is further strongly convex or smooth, we can reduce the computational cost to   $O(t)$ for each resolving. See Section \ref{sec:application:str_cvx} for more discussions on the case of strongly convex objectives. In Section \ref{sec:regret}, we demonstrate that any optimization algorithm $\mathcal{B}_t$ that achieves the rate of dual convergence $\E\norm{\nu_t-\nu^{*}(d_t)}_2^2=O(t^{-1})$ or $O\big((T-t)^{-1}\big)$ suffices to guarantee the optimal logarithmic regret in the end.

\begin{algorithm}
\caption{Resolving with Stochastic Gradient Descent}
\label{alg:SGD}
\begin{algorithmic}
\REQUIRE regularizer $r$, iteration number $T$, start point $\bmu_0$, where $\boldsymbol{\mu}=[\lambda^\top,\mu^\top]^\top $ and initial resource $B_0:=dT$.
\FORALL{$t = 1,\dots,T$}
    \STATE{Receive $(f_t,b_t)\sim \cP$.  }
    \STATE{Calculate 
    $\tilde{x}_t :=\tilde{x}_t(\nu_{t-1}):= \mathop{\arg\max}_{x\in \cX}\{f_t(x)-(\lambda_{t-1}+\mu_{t-1})^\top b_t x\} = -\nabla f^*_t(b_t^\top(\lambda_{t-1}+\mu_{t-1}))$.}
    	\STATE{Select $ x_t:=\begin{cases} \tilde{x}_t&\text{ if }B_{t-1}\ge b_t x_t  \\ 0 &\text{ otherwise }  \end{cases}$ }
	\STATE{Update remaining resources: $B_t:= B_{t-1}-b_t x_t$  and average remaining resources: $d_t:= \frac{B_t}{T-t}$}
	\STATE{Set  $R:=\sqrt{m\left(2\frac{\Bar{f}+\Bar{r}}{\underline{d}} +G\right) }$, $L:=\sqrt{m\Bar{d}^2+2n\Bar{b}^2D^2+nG^2}$,  $K:=t^3$, and $\eta_t:=\frac{\sqrt{2}R }{L \sqrt{K}}$. Define $\bmu_t^0:=\bmu_{t-1}$
	  }	
    \FORALL{$k = 1,\dots,K$}
    \STATE{Randomly pick $\zeta$ from $[t]:=\left\{1,\dots,t\right\}$ with uniform distribution and calculate the gradient 
    	\begin{equation}
    	\label{eq:dual_gradient}
    	    \nabla D_\zeta (\bmu_{t}^{k-1}) := \left[\begin{array}{c} -b_\zeta \tilde{x}_{\zeta}(\nu_{t}^{k-1}) +d_t\\ -b_\zeta \tilde{x}_{\zeta}(\nu_{t}^{k-1}) + \tilde a(\mu_{t}^{k-1}) \end{array}\right]
    	\end{equation} 
    	}
	\STATE{Update dual variable via stochastic gradient descent: 
	\begin{equation}
	\label{eq:dual_update}
	    \bmu_t^{k}:= \argmin_{\bmu\in \Omega_{\lambda}^+\times\Omega_{\mu} } \left\{ \langle\bmu,\nabla D_\zeta (\bmu_{t}^{k-1})\rangle +\frac{1}{2\eta_t}\norm{\bmu-\bmu_{t}^{k-1}}^2_2 \right\} 
	\end{equation}
	}	
    \ENDFOR
    \STATE{Update dual variable by averaging: $\bmu_t:=\frac{\sum_{k=1}^{K} \bmu_t^k }{K}$}
\ENDFOR
\end{algorithmic}
\end{algorithm}

\section{Regret Analysis}
\label{sec:regret}

\subsection{Regret upper bound}
In this section, we apply dual convergence established in Section \ref{sec:dual_conv} to derive an upper bound of regret. The result is valid for our algorithm framework Algorithm \ref{alg:framework} with any dual optimizers. Without loss of generality, we focus on stochastic optimizers $\mathcal{B}_t$, which are independent of future arrivals $\left\{(f_j,b_j)\right\}_{j\ge t+1}$. As long as $\mathcal{B}_t$ delivers reasonably accurate dual solutions $\blambda_t$, based on past history $\mathcal{H}_{t-1}=\left\{f_j,b_j,x_j \right\}_{j=1}^{t-1}$, new arrival $(f_t, b_t)$ and updated constraint $d_t\in\Omega_d$, our adaptive framework Algorithm~\ref{alg:framework} achieves a logarithmic-order regret. Precisely, the accuracy of dual solutions shall satisfy the following condition.

\begin{Condition}(Accuracy of dual solutions)
\label{condition:conv}. 
Suppose the updated constraints $\left\{d_j\middle| 1\le j \le t\right\}\subseteq \Omega_{d}$. We say the algorithm $\{\mathcal{B}_j\}_{j\ge 1}$ satisfies dual convergence condition \ref{condition:conv} if 
\begin{equation}
\label{log-convergence}
\\
\E_{\mathcal{B},\mathcal{P}} \norm{\nu_{t}-\nu^*(d_t)}^2 \le C_4 \frac{1}{t+1}, \text{  or   } \ \  \E_{\mathcal{B},\mathcal{P}} \norm{\nu_{t}-\nu^*(d_t)}^2 \le C_4(\frac{1}{t+1} + \frac{1}{T-t})
\end{equation}
for some constant $C_4=\tilde{O}(nm )$. The expectation is taken with respect to all the $\{\mathcal{B}_j\}_{j\ge 1}$ and $\cP$.
\end{Condition}

Recall that the dual convergence established in Section \ref{sec:dual_conv} holds uniformly for any $d\in\Omega_{d}$. Therefore, any dual optimizers ensuring corresponding dual solution error $\epsilon_t=\Theta(t^{-1})$ or $\epsilon_t=\Theta\big((T-t)^{-1}\big)$ ( $\epsilon_t=\Theta(t^{-3/2})$ or $\epsilon_t=\Theta((T-t)^{-3/2}$ for stochastic dual optimizers) satisfy Condition \ref{condition:conv}. If Condition~\ref{condition:conv} holds, our adaptive framework Algorithm \ref{alg:framework} achieves the following optimal regret.

\begin{theorem}[Regret upper bound]
\label{thm:regret_ub}
Under Assumptions \ref{asm:basic}-\ref{asm:k-thmoment}, if the algorithm $\{\mathcal{B}_t\}_{t\ge 1}$ we choose satisfies Condition \ref{condition:conv}, then the regret of Algorithm \ref{alg:framework} has the following upper bound:
\begin{equation*}
    \text{Regret}(A)\leq \mathring{C}\cdot \log T
\end{equation*}
for some constant $\mathring{C}=O(m^2n^2\log m)$ depending on the values in Assumptions \ref{asm:basic}-\ref{asm:nondegeneracy:3}. 
\end{theorem}

Clearly, exact solutions to the SAA program \eqref{eq:adaptive_SAA} is a theoretically valid candidate for $\{\mathcal{B}_t\}_{t\geq 1}$, which is actually  the classic idea of re-solving heuristic. However, the computational cost can be high if we want to find an exact solution. Fortunately, by Theorem \ref{thm:regret_ub}, it suffices to approximately solve SAA program \eqref{eq:adaptive_SAA} as long as the accuracy meets conditions (\ref{log-convergence}). 
We shall show in Section~\ref{sec:regret:wth_dt} that the rate $O(\log T)$ is optimal for the number of periods $T$. The regret upper bound depends quadratically on the dimension of constraints $m$ and the dimension of action space $n$.

We now briefly sketch the proof of Theorem \ref{thm:regret_ub}. The proof begins with the decomposition of regret, which shows that regret can be controlled by the cumulative error of dual solutions $\nu_t-\nu^*(d_t)$ and by $\E[T-\tau]$ for some stopping time $\tau$.
Recall, given a certain distribution $\cP$, the definition of regret:
$$
\operatorname{Regret}\left(A\middle| \cP\right)=R^*(\mathcal{P})-R\left(A\middle| \mathcal{P}\right), 
$$
where $R^*(\mathcal{P})$ and $R\left(A\middle| \mathcal{P}\right)$ are defined in \eqref{eq:max_risk} and \eqref{eq:alg_risk}, respectively. To upper bound the regret, we need an upper bound of offline maximum reward. To that end, we define
$$
g (\nu):= \E \big[ f_t (\tilde x_t(\nu )) + r(\tilde a (\mu^*)) + (\tilde a(\mu^*)-b_t \tilde x_t({\nu}))^{\top}\mu^* +(d-b_t \tilde x_t(\nu))^{\top}\lambda^*\big].
$$
Here $g(\nu)$ serves as an upper bound for $R^*(\mathcal{P})$, characterized by the following lemma.

\begin{lemma}\label{lemma:offline_ub}
The offline maximum reward $R^*(\mathcal{P})$ satisfies
$
R^*(\mathcal{P}) \le T\cdot g(\nu^*).
$
\end{lemma}
Note that, as shown by \cite{bumpensanti2020re,vera2021bayesian},  as long as our problem is degenerate, there will be a possible $\Omega(\sqrt{T})$ gap between fluid problem $T\cdot g(\nu^*)$ and the optimal $R^*(\mathcal{P})$, which means that this upper bound may not be tight in the degenerate case.

Since $f_t$ and $r(a)$ have trivial upper bounds, we get $ R^*(\mathcal{P})\le T(\bar{f}+\bar{r})$. Thus, for a proper stopping time $\tau$, we have 
\begin{equation}
\label{eq:upp_1}
    R^*(\mathcal{P})\le \E \left[\tau g(\nu^*)+(T-\tau)(\bar{f}+\bar{r}) \right]. 
\end{equation}
Note that our strategy of dealing with the non-separable regularizer is to introduce an additional (i.e., variable split) primal variable $a$. While its actual value does not directly affect our algorithm framework, it is vital for our theoretical investigation. To this end, denote $a_t= \tilde{a}(\mu_t)$ the value of $a$ at $t$-th iteration. By Fenchel conjugate, we actually have $a_T=-\nabla r^{\ast}(-\mu_T)$. The second impact of variable splitting is an equality constraint between $a_T$ and $T^{-1}\sum_{t=1}^T b_tx_t$. It turns out that their difference can be measured by the difference between $\mu_T$ and the following quantity:
\begin{align}\label{eq:hat_muT}
 \hat{\mu}_T:=\arg\min\limits_{\mu} \left\{ r^*(-\mu)-\mu^\top\frac{\sum_{t=1}^{T}b_t x_t}{T}\right\}.
\end{align}
The above minimization is taken without constraints, implying that $T^{-1}\sum_{t=1}^T b_tx_t=-\nabla r^{\ast}(-\hat\mu_T)$. 



We can now describe the following regret decomposition for a general stopping time. 

\begin{Proposition}
\label{prop:regret_decomp}
Under Assumptions \ref{asm:basic}-\ref{asm:nondegeneracy:3}, for a proper stopping time $\tau$ ensuring that the resource is not depleted before $ t\le \tau$, the regret of our dual-based adaptive framework Algorithm~\ref{alg:framework} admits the following upper bound: 

\begin{equation}\label{eq:regret_decom}
    \begin{aligned}
        \operatorname{Regret}\left(A \middle| \cP \right)\le & \underbrace{\E\left[\sum_{t=1}^{\tau} g(\nu^*)-g(\nu_t)\right]}_{\textsf{R.1}} + \underbrace{\E \left[ 2(\bar{f}+\bar{r}+C_3)(T-\tau) + \left\langle \lambda^*,  \sum_{t=1}^{\tau}( d - b_t x_t) \right\rangle  \right] }_{\textsf{R.2}},\\
        & + \underbrace{\E\left[ \left\langle \mu^*-\hat\mu_T, \sum_{t=1}^{\tau}(\tilde{a}(
        \mu^*) -b_t x_t)\right\rangle \right]}_{\textsf{R.3}},
    \end{aligned}
\end{equation}
where $C_3:=\sqrt{mn}GD\Bar{b}$.
\end{Proposition}

It remains to bound the two parts in Proposition \ref{prop:regret_decomp}, respectively.  The key point is to carefully choose a stopping time that (1) avoids early stopping; (2) enforces the total resource constraints. The first term $\textsf{R.1}$ is contributed by the algorithm before stopping time, which can be controlled by the cumulative dual error $\E\sum_{t=1}^{\tau} \norm{\nu_{t-1}-\nu^*}^2$. 
The second term $\textsf{R.2}$ concerns the lost rewards due to resource depletion, which can be controlled by $\E(T-\tau)$.
To achieve an $O(\log T)$ regret, the stopping time shall be carefully designed so that $\E(T-\tau)=O(\log T)$. 
The term $\textsf{R.3}$ is contributed mainly by the variable splitting, which can be controlled jointly by the  cumulative dual error and $\E(T-\tau)$. 
The three terms capture different sources of regret induced by our adaptive framework Algorithm~\ref{alg:framework}.  It turns out that we shall bound $\E\sum_{t=1}^{\tau} \norm{\nu_{t-1}-\nu^*}^2$ and $\E(T-\tau)$, for which a smart design of stopping time becomes crucial. 

Our design of stopping time is inspired by the budget-ratio stopping time introduced and investigated by \cite{arlotto2019uniformly} and \cite{li2021online} for online linear allocation problems. At the core of this design is a smart strategy that ensures, as the updated constraint $d_t$ varies within a region $\mathcal{D}\subset \Omega_d$,  the binding and non-binding dimensions of the problem $D(\blambda, d_t)$ remain unchanged. The region $\mathcal{D}$ is usually a small neighbor of the original budget $d$. The following lemma dictates that such a region $\mathcal{D}$ exists for our regularized online convex allocation problem. Recall that $\blambda^{\ast}(d')$ denotes the optimal dual solution to $D(\blambda, d')$, and $I_{\textsf{B}}$ and $I_{\textsf{NB}}$ stand for the binding and non-binding dimension of $D(\blambda, d)$.

\begin{lemma}
\label{lemma:region_d}
Under Assumptions \ref{asm:basic}-\ref{asm:nondegeneracy:3},  there exists a constant $\delta_d>0$ such that for any $d'\in\Omega_d$, if
$$
  -\delta_d \le d_i'- d_i\le\delta_d\text{ if } i\in I_{\textsf{B}} \text{, and } d'_i-d_i\ge-\delta_d\text{ if } i\in I_{\textsf{NB}}, 
$$
then the dual problems $D(\blambda, d')$ and $D(\blambda, d)$ share the same binding and non-binding dimensions. 
\end{lemma}

With Lemma \ref{lemma:region_d}, we define the required region where  binding and non-binding dimensions remain unchanged during iterations by
$$
\mathcal{D}:=\left\{d'\in \Omega_d \middle| -\delta_d \le d_i'- d_i\le\delta_d\text{ if } i\in I_{\textsf{B}} \text{, and } d'_i-d_i\ge-\delta_d\text{ if } i\in I_{\textsf{NB}} \right\}.
$$ 
We thereby design the following stopping time.
\begin{equation}\label{eq:tau}
    \tau:=\min\limits_{t\in [T]} \left\{T-\left\lceil \frac{\sqrt{n}D\bar{b}}{\underline{d}} \right\rceil\right\}\bigcup \big\{ t | d_t \notin \mathcal{D} \big\}.
\end{equation}
Additionally, this stopping time also guarantees that resource depletion will not happen before $\tau$. We show that $\tau$ rules out early-stopping so that $\E[T-\tau]=O(\log T)$. The following lemmas bound the cumulative dual error and $\E[T-\tau]$.

\begin{lemma}
\label{lemma:alg_conv}
 Under Assumptions \ref{asm:basic}-\ref{asm:nondegeneracy:3}, Algorithm \ref{alg:framework} with selected dual optimizer $\{\mathcal{B}_t\}_{t\ge 1}$ satisfying Condition \ref{condition:conv} achieves
\begin{equation}
    \begin{aligned}
    \E\left[\sum_{t=1}^{\tau } \norm{\nu_{t-1}-\nu^*}^2\right]\le O( \log T) 
    \end{aligned}
\end{equation}
\end{lemma}

\begin{lemma}
\label{lemma:early_stop}
 Under Assumptions \ref{asm:basic}-\ref{asm:nondegeneracy:3}, the stopping time (\ref{eq:tau}) of Algorithm \ref{alg:framework} with selected dual optimizer $\{\mathcal{B}_t\}_{t\ge 1}$ satisfying Condition \ref{condition:conv} has 
\begin{equation}
    \begin{aligned}
    \E\left(T-\tau \right)\le  O( \log T) 
    \end{aligned}
\end{equation}
\end{lemma}

These two lemmas play a key role in our regret analysis. They are proved by investigating the dynamic behavior of constraints $d_{it}$ for binding and non-binding dimensions, respectively.  For binding dimensions, we investigate the recurrence relation of $d_{it}$ by leveraging the binding relations. For the non-binding dimensions, we exploit the $\delta_d$ gap between $d_{it}$ and average resource consumption. Since the regret is jointly controlled by $\E\sum_{t=1}^{\tau} \norm{\nu_{t-1}-\nu^*}^2$ and $\E(T-\tau)$, we conclude the $O(\log T)$ regret order. We defer the detailed proof to Appendix \ref{proof:thm_regret_ub}.

\subsection{Lower bound and algorithms without constraint update}\label{sec:regret:wth_dt}
\cite{bray2019does} and \cite{li2021online} have established the logarithmic regret lower bound for online multi-secretary problems and online linear programming, respectively. We establish a matching lower bound in this section to show the optimality of Theorem~\ref{thm:regret_ub}. 
 We note that there always exists a regularizer function that makes our regularized online allocation problem more challenging than the non-regularized one. For example, consider that $f_t(x)$ and $r$ are both monotonic increasing and the hindsight optimal strategy $\{x_t'\}_{t=1}^{T}$ that optimizes $\max_{x_t\in\cX}\big\{\sum_{t=1}^{T} f_t(x_t) \text{ s.t. } \sum_{t=1}^{T}b_tx_t\le dT\big\}$, it holds that $\sum_{t=1}^{T}b_tx_t'= dT$ and thus $r(T^{-1}\sum_{t=1}^{T}b_tx_t')\ge r(T^{-1}\cdot\sum_{t=1}^{T}b_tx_t)$ for any other $\{x_t\}_{t=1}^{T}$. This renders the regret lower bound of regularized problem larger than that of a non-regularized one. Therefore, we only focus on the non-regularized problems for the regret lower bound.
\begin{theorem}[Regret lower bound]\label{thm:regret_lb}
For any dual-based algorithm $A$, we have the worst-case regret lower bound:
\begin{equation*}
    \begin{aligned}
       \operatorname{Regret}(A)\ge \Omega(\log T).
    \end{aligned}
\end{equation*}
\end{theorem}

Theorem \ref{thm:regret_lb} justifies the optimality of our algorithm in terms of worst-case regret. The logarithmic regret also matches classic unrestricted online convex optimization \citep{hazan2007logarithmic}. Nevertheless, one may wonder how important the adaptive constraint update is in our adaptive framework Algorithm \ref{alg:framework} and whether achieving an optimal regret without the adaptive constraints update is possible. Here we only present a negative answer for algorithms that do not update constraints, partially for two specific but renowned algorithms. In this discussion, we call ``algorithms without constraint update" as dual algorithms seeking to approximately optimize the fluid problem $D(\blambda,d)$ in the whole process without changing $d$ and the objective. 
For concreteness, we investigate two similar algorithms (described in Algorithms~\ref{alg:ogd}) without constraints update that have been discussed in the literature for online dual gradient (mirror descent \citep{balseiro2021regularized,balseiro2022best} and dual SAA \citep{li2021online}. Algorithm \ref{alg:ogd} try to minimize $D(\blambda,d)$ by either Stochastic Approximation (SA) or by solving SAA. In the strongly convex case, it is clear that they all converge with $\E\norm{\blambda_t-\blambda^*}=O(1/\sqrt{t})$ by \cite{rakhlin2012making} and Section \ref{sec:dual_conv}. 

\begin{algorithm}
\caption{Online dual gradient (mirror) descent or SAA without constraint update}
\label{alg:ogd}
\begin{algorithmic}
\REQUIRE regularizer $r$, iteration number $T$, step size $\eta_t:=\Theta(\frac{1}{t})$ for $t\in [T]$, start point $\boldsymbol{\mu}_0=[\lambda^\top_0,\mu_0^\top]^\top$, 
and initial resource $B_0:=dT$.
\FORALL{$t = 1,\dots,T$}
    \STATE{Receive $(f_t,b_t)\sim \cP$.  }
    \STATE{Calculate 
    $\tilde{x}_t := -\nabla f^*_t(b_t^\top(\lambda_{t-1}+\mu_{t-1})), \ \tilde a_{t} :=-\nabla r^*(-\mu_{t-1})$.
}
    	\STATE{Select $ x_t:=\begin{cases} \tilde{x}_t &\text{ if }B_{t-1}\ge b_t x_t  \\ 0 &\text{ otherwise }  \end{cases}$ }
	\STATE{Update remaining resources: $B_t:= B_{t-1}-b_t x_t$  }
	\STATE{ Calculate the stochastic gradient $\nabla D_t (\bmu_{t-1},d) := \left[\begin{array}{c} -b_t \tilde{x}_t +d\\ -b_t \tilde{x}_t+a_{t}\end{array}\right]$ \\ }
	\STATE{ Update dual variable via online gradient descent: $$	    \bmu_t:= \argmin_{\bmu\in \Omega_{\lambda}^+\times\Omega_{\mu} } \left\{ \langle\bmu,\nabla D_t (\bmu_{t-1},d )\rangle +\frac{1}{2\eta_t}\norm{\bmu-\bmu_{t-1} }^2_2 \right\}  $$  }	
    	\STATE{Or via solving t-sample SAA: 
	$$
	 \bmu_t := \argmin_{\bmu\in \Omega_{\lambda}^+\times\Omega_{\mu} } \left\{ \frac{1}{t} \sum^{t}_{j=1} f_j^{*}(b_j^{\top}(\mu+\lambda)) +r^{*}(-\mu)+d^{\top}\lambda \right\}
	 $$  }	
\ENDFOR
\end{algorithmic}
\end{algorithm}

The following lemma establishes an $\Omega(T^{1/2})$ regret lower bound for these two algorithms.

\begin{theorem}\label{thm:ogd_lb}
Under Assumptions \ref{asm:basic}-\ref{asm:k-thmoment},  there exists a constant $c_2>0$ such that any dual-based algorithm $A$ attempting to approximate $\nu^*$ with $\E\norm{\nu_t-\nu^*}_2\le c_2D(t+1)^{-1/2}$ incurs a worst-case regret lower bound:
\begin{equation*}
    \operatorname{Regret}(A)\ge\Omega(T^{1/2})
\end{equation*}
\end{theorem}
We prove this theorem by constructing a one-dimensional strongly convex reward and bound the regret by leveraging the probability estimation of Binomial distribution.  Note that the lower bound can also be controlled by both dual approximate error $\E\sum_{t=1}^{\tau} \norm{\nu_{t-1}-\nu^*}^2_2$ and early stopping effect $\E(T-\tau)$. Here $\nu^{\ast}$ is the deterministic dual solution when the resource constraint is fixed at $d$. In sharp contrast, the dual solution $\nu_t$ in our adaptive framework Algorithm~\ref{alg:framework} aims to approximate $\nu^{\ast}(d_t)$ where $d_t$ is the updated constraint at time $t$. Intuitively, the rationale behind constraint update is that at time $t$, and the decision should be made considering the remaining resources $d_t$ at hand instead of the initial resource $d$. Algorithm \ref{alg:ogd}, however, without constraint update, suffers from early stopping $\E(T-\tau)\ge \Omega(\sqrt{T})$ and is thus not optimal.


\begin{remark}
Theorem  \ref{thm:ogd_lb} suggests that Algorithm \ref{alg:ogd} fails to reach the optimal regret under our assumptions because they all seek to approximate a deterministic $\blambda^*$. In fact, even if we know the exact distribution $\mathcal{P}$ and its optimal solution $\blambda^*$, we are still unable to make our dual-based algorithm optimal by just choosing $\blambda_t=\blambda^*$.  Theorem \ref{thm:ogd_lb} gives rigorous evidence that our constraint-update algorithm outperforms other prior ones without constraint update, such as the online gradient descent studied by \cite{balseiro2021regularized,balseiro2022best}.

Finally, we remark that our theorem pushes forward the understanding of adaptiveness for online algorithms to dual-based ones. In \cite{arlotto2019uniformly}, the authors established an $\Omega(\sqrt{T})$ regret lower bound only for non-adaptive strategies (without adaptively updating the dual solutions). However, our proof demonstrates that, even when the strategy is adaptive, it might still not be sufficient to deliver an optimal regret if the algorithm only focuses on dual updates but neglects the constraint update. Actually, as in Algorithm \ref{alg:ogd}, focusing on fixed constraints leads to a sub-optimal early stopping.
\end{remark}

\section{Infrequent Resolving and Fast algorithms}
Although Algorithm \ref{alg:framework} only requires inexact resolving, its frequency of solving convex programming is of order $O(T)$ for $T$ periods. This raises the question of whether it is possible to further reduce the computational burden through infrequent resolution or other faster algorithms. Here we extend our previous result to both infrequent resolving and fast algorithm design, showing that (i) we can still achieve optimal regret by infrequent resolving given a good initialization; (ii) we can reach sub-optimal $O(\log^2 T)$ regret under linear computational cost. These two algorithms significantly alleviate the computational cost while ensuring a good regret performance.

	


Define rate $\rho\in (0,1) $ for resolving that satisfies: $\rho\ge \left\lceil \frac{\sqrt{n}D\bar{b}}{\underline{d}} \right\rceil/T\vee (1-\delta_d/(\Bar{d}+\sqrt{n}\bar{b}D+\delta_d ) )$, i.e., $t\le T-\left\lceil \frac{\sqrt{n}D\bar{b}}{\underline{d}} \right\rceil$ and $d_{t }\in\mathcal{D}$ for any $t\le(1-\rho)T$.  We describe the infrequent resolving algorithm in Algorithm \ref{alg:infrequent} and the fast algorithm in Algorithm \ref{alg:fast}.
\begin{algorithm}
\caption{Infrequent resolving}
\label{alg:infrequent}
\begin{algorithmic}
\REQUIRE regularizer $r$, iteration number $T$, start point $(\lambda_0,\mu_0)$, and initial resource $B_0:=dT$.
\STATE{Set $J=\lceil \log_{\frac{1}{\rho}} T\rceil $, and $ T_j = T-\lceil\rho^j T  \rceil$ for $j\in [J]$ }
\FORALL{$t = 1,\dots,T$}
    \STATE{Receive $(f_t,b_t)\sim \cP$.  }
    \STATE{Calculate  $\tilde{x}_t := -\nabla f^*_t(b_t^\top(\lambda_{t-1}+\mu_{t-1}))$.
}
    	\STATE{Select $ x_t:=\begin{cases} \tilde{x}_t&\text{ if }B_{t-1}\ge b_t x_t  \\ 0 &\text{ otherwise }  \end{cases}$ }
	\STATE{Update remaining resources: $B_t:= B_{t-1}-b_t x_t$}
    \IF{ $t=T_{j}$ for some $j$ }
    \STATE{Compute average remaining resources: $d_t:= \frac{B_t}{T-t}$}
    	\STATE{Update dual variable $(\lambda_t,\mu_t)$ via solving the following dual problem by any approximation algorithm $\cB_t$ with accuracy $\epsilon_t$: 
	\begin{equation*}
	    \min_{(\lambda,\mu)\in \Omega_{\lambda}\times\Omega_{\mu} } \left\{ \Bar{D}_{t} (\lambda,\mu,d_t):=\frac{1}{t} \sum^{t}_{l=1} f_l^{*}(b_l^{\top}(\mu+\lambda)) +r^{*}(-\mu)+d_t^{\top}\lambda \right\}
	\end{equation*}  }
	\ELSE
	\STATE{Let $(\lambda_t,\mu_t)=(\lambda_{t-1},\mu_{t-1})$, $d_t=d_{t-1}$  with no update }
	\ENDIF
\ENDFOR
\end{algorithmic}
\end{algorithm}
The ideas behind Algorithm \ref{alg:infrequent} and \ref{alg:fast} are similar: we split the total period $T$ into $O(\log T)$ decreasing epochs, and only update the remaining constraints $d_t$ for resolving at the beginning of each epoch $t=T_j$. The distinction is that, Algorithm \ref{alg:infrequent} inexactly solve the convex programming at time $T_j$, but Algorithm \ref{alg:fast} exploits the following epoch from $T_j$ to $T_{j+1}$ to perform stochastic approximation as to minimize $D(\blambda,d_{T_j})$.
\begin{theorem}[Infrequent resolving]\label{thm:infreq} Suppose Assumption \ref{asm:basic}-\ref{asm:nondegeneracy:3} hold. Given the initial point $\blambda_0=[\nu_0^\top,\mu_0^\top]^\top$ satisfying $\E\norm{\nu_0-\nu^*}_2^2=O(1/T)$, and under Condition \ref{condition:conv}, Algorithm \ref{alg:infrequent} enjoys an optimal regret upper bound 
\begin{equation*}
    \text{Regret}(A)\leq \mathring{C}\cdot \log T,
\end{equation*}
for some constant $\mathring{C}=O(m^2n^2\log m)$ depending on the values in Assumptions \ref{asm:basic}-\ref{asm:nondegeneracy:3}. Here in the regret, the expectation is taken also with respect to $\nu_0$.
\end{theorem}

\begin{algorithm}
\caption{Fast algorithm}
\label{alg:fast}
\begin{algorithmic}
\REQUIRE regularizer $r$, iteration number $T$, start point $\boldsymbol{\mu}_0=[\lambda^\top_0,\mu_0^\top]^\top$, and initial resource $B_0:=dT$. Set $l=0$
\STATE{Set $J=\lceil \log_{\frac{1}{\rho}} T\rceil $, and $ T_j = T-\lceil\rho^j T  \rceil$ for $j\in [J]$ }
\FORALL{$t = 1,\dots,T$}
    \STATE{Receive $(f_t,b_t)\sim \cP$.  }
    \STATE{Calculate $\tilde{x}_t := -\nabla f^*_t(b_t^\top(\lambda_{t-1}+\mu_{t-1})), \ \tilde a_{t} :=-\nabla r^*(-\mu_{t-1})$.
}
    	\STATE{Select $ x_t:=\begin{cases} \tilde{x}_t&\text{ if }B_{t-1}\ge b_t x_t  \\ 0 &\text{ otherwise }  \end{cases}$ }
	\STATE{Update remaining resources: $B_t:= B_{t-1}-b_t x_t$}
    \IF{ $t=T_{j}$ for some $j$ }
    \STATE{Let $l=T_j$. Compute average remaining resources: $d_t:= \frac{B_t}{T-t}$}
	\ELSE
	\STATE{Let $d_t=d_{t-1}$ with no constraint update }
	\ENDIF
	\STATE{ Calculate the stochastic gradient with updated constraint  $\nabla D_t (\bmu_{t-1},d_t) := \left[\begin{array}{c} -b_t \tilde{x}_t +d_t\\ -b_t \tilde{x}_t+a_{t}\end{array}\right]$  }
	\STATE{Set step size $\eta_t:=\Theta(\frac{1}{t-l+1})$ and update dual variable via online gradient descent: $$	    \bmu_t:= \argmin_{\bmu\in \Omega_{\lambda}^+\times\Omega_{\mu} } \left\{ \langle\bmu,\nabla D_t (\bmu_{t-1},d_t )\rangle +\frac{1}{2\eta_t}\norm{\bmu-\bmu_{t-1} }^2_2 \right\}  $$ 
	}	
\ENDFOR
\end{algorithmic}
\end{algorithm}
\begin{theorem}[Sub-optimal fast algorithm]\label{thm:fast} Suppose Assumption \ref{asm:basic}-\ref{asm:nondegeneracy:3} hold. Algorithm \ref{alg:fast} achieves sub-optimal regret bound
\begin{equation*}
    \text{Regret}(A)\leq \widetilde{C}\log^2 T,
\end{equation*}
for some constant $\widetilde{C}=O(mn^2)$ depending on the values in Assumptions \ref{asm:basic}-\ref{asm:nondegeneracy:3}. 
\end{theorem}
\begin{remark} The initialization in Theorem \ref{thm:infreq} serves to ensure the performance in the first epoch. This initialization requirement can be met, for instance, by employing dual optimization based on previously collected data as side information. The fast algorithm presented in Algorithm \ref{alg:fast} updates the dual variable using one-step online gradient descent, resulting in a linear computational cost. In contrast to \cite{balseiro2022best}, our Algorithm \ref{alg:fast} incorporates constraint updates, which effectively helps us avoid the $\Omega(\sqrt{T})$ lower bound stated in Theorem \ref{thm:regret_lb} and achieve logarithmic regret. Compared with the previous results, our Algorithm \ref{alg:fast} only depends linearly on the number of constraints. Due to the reliance of dual convergence, the polynomial dependence of dimensions ($O(mn)$) is unavoidable in regret analysis of all the methods in our framework without further assumptions on the problem.

\end{remark}
\section{Applications}\label{sec:application}

\subsection{Strongly convex dual problems}\label{sec:application:str_cvx}
We consider a special but practical setting, in which our empirical dual problem $\bar{D}_t(\blambda,d_t)$  in  \eqref{eq:adaptive_SAA} is always $\underline{\cL}_D$-strongly convex. This assumption can be met if $f_t^*$ and $r$ are almost surely strongly convex. In this case, we only need to do SGD for  $O(t)$ times at time $t$ to make our algorithm theoretically optimal. 
Simply modify algorithm by setting  $K:=t$, and $\eta_{k}:=\frac{\underline{\cL}_D }{k }$, and take
$\bmu_t:= \bmu_t^K $ 
Notice that $\E\norm{\nu_t-\nu^*(d_t)}_2^2 \le 2\E\norm{\nu_t-\nu^*_t(d_t)}_2^2 + 2\E\norm{\nu^*_t(d_t)-\nu^*(d_t)}_2^2$ where $\nu^*_t(d_t)$ is the optimal solution to the empirical dual problem $\bar{D}_t(\blambda,d_t)$. The second term   $\E\norm{\nu^*_t(d_t)-\nu^*(d_t)}_2^2$ represents the dual convergence and can be bounded by $O(t^{-1})$ by Theorem \ref{thm:dual_conv}, while the first term accounts for the optimization error and can also be bounded by $O(t^{-1})$ (see, \cite{rakhlin2012making}). 

\vspace{-0.5cm}
\subsection{Online linear programming}
An instant application of our algorithm is the classical non-regularized online linear allocation problems, which finds applications in online ad-auction \citep{buchbinder2007online}, network revenue management \citep{jasin2012re}, multi-secretary problem \citep{kleinberg2005multiple}, etc. At time $t$, we make a decision $x_t\in\cX=\left[0,D\right]^n$ that returns a linear reward $v_t$ and bears a random cost $b_t\in \mathbb{R}^{m\times n}$ per unit. Online linear programming can be formalized as:
$$
\begin{aligned}
&\max_{x_t} \quad &&\sum_{t=1}^T v_t^\top x_t \\
&\text{s.t.} \quad &&\sum_{t=1}^T b_t x_t \preceq dT, \ d\in \mathbb{R}_+^{m}\\
& \quad &&  x_t \in \left[0,D\right]^n , \forall t\in [T].
\end{aligned}
$$
The empirical dual problem and its population version can be explicitly written as 
$$
\bar{D}_T(\lambda,d):=\frac{\sum_{t=1}^{T}\sum_{i=1}^{n}\left(v_{it}-b_{it}^\top\lambda \right)^{+} }{T}+d^\top\lambda , \text{ and } \ D(\lambda,d) := \E\sum_{i=1}^{n}\left(v_{it}-b_{it}^\top\lambda \right)^{+} +d^\top\lambda, 
$$
which is in line with \cite{li2021online}. Here the index $b_{it}$ means the $i$-column of $b_t$. For a given dual variable $\lambda$, we make the primal decision by $x_{it}:= D\mathbb{I}(v_{it}-b_{it}^\top\lambda>0)$ if the resource constraints are not violated. Then, under the same locally strongly convex and non-degeneracy assumptions, we can make optimal decisions by choosing approximate solution $\lambda_t$. 
Towards that end, an $O(\log T)$ regret is attainable, which improves prior result \citep{li2021online}. Assumption \ref{asm:k-thmoment} seems stronger than the smoothness of expected primal solutions because the former implies the latter. However, in the case of linear programming they are actually equivalent because, for any $\delta>0$, we have
$$\E\left[\sup_{\lambda: \norm{b_t^\top(\lambda-\lambda^*) }\le \delta} D\mathbb{I}(b_{it}^\top\lambda^*\le v_{it}\le b_{it}^\top\lambda)\middle| b_t\right]\le  D\PP(b_{it}^\top\lambda^*\le v_{it}\le b_{it}^\top\lambda+\delta)\le O(\delta),
$$ 
i.e., the smooth of the expected primal decision also implies Assumption 4.

\vspace{-0.3cm}
\subsection{Online max-min fairness and load-balancing allocation}
Our algorithm framework applies to online max-min fairness allocation, which is a well-accepted fairness criterion used in various real-world problems, including bandwidth allocation \citep{salles2008lexicographic}, routing and load-balancing \citep{nace2006computing}, and classroom allocations \citep{kurokawa2015leximin}. To satisfy the condition on regularizer in Assumption \ref{asm:nondegeneracy:3}, we set $r(a) = \kappa \min_i (a_i/d_i)$ and require the (scaled) expected optimal constraint $\left\{d^*_i/d_i:=\E\left(b_t\tilde{x}_t(\nu^*)\right)_i/d_i, \ i\in[ m]\right\}$ to have one unique minimum element.

We can also apply our algorithm framework to online load-balancing problems, which have been studied in various fields such as bandwidth allocation \citep{bejerano2004fairness}, network design \citep{radunovic2007unified}, distributed system design \citep{xu2011chemical}, and other various scenarios. To avoid over-exploitation, we select {negative max loss}: $r(a):=-\kappa\max_i ({a_i}/{d_i})$ as our regularizer. To satisfy the regularizer condition, we only need the (scaled) expected optimal constraint $\left\{d^*_i/d_i:=\E\left(b_t\tilde{x}_t(\nu^*)\right)_i/d_i, \ i\in[ m]\right\}$ to have one unique maximum element.

Other regularizers like $\ell_1$-loss, hinge loss \citep{balseiro2021regularized}, and entropy loss are also applicable as long as the optimal resource consumption $d^*$ is located in a small smooth region of regularizer $r$ only for non-binding dimensions.

\section{Numerical Experiments}\label{sec:experiment}
We implement Resolving with SGD as a showcase for our proposed algorithmic framework. The performance is assessed under 4 different stochastic input models. Due to limited space, the numerical results are relegated to Appendix \ref{appdx:exp}. The results show the superiority of our algorithms in regret control compared with other aforementioned algorithms. Our algorithm exhibits $O(\log T)$ regret within all different input models, which corroborates our Theorem \ref{thm:regret_ub} and \ref{thm:regret_lb}. We also display the resource consumption dynamic to visualize our algorithmic framework's optimal resource control. To compare the impact of different regularizers, we plot the regrets on different regularization levels and show the trade-off between maximizing the reward and penalizing the regularization.  The regret performance is very robust to different regularization levels.


\section{Discussion}\label{sec:discuss}
This paper investigated regularized online convex allocation problems with a non-separable regularizer. While a polynomial-time adaptive algorithm framework is proven optimal in controlling regret,  several interesting yet challenging questions remain open. One is the necessity of the non-degeneracy assumption. Recently, \cite{bumpensanti2020re} showed that the non-degeneracy assumption is unnecessary for re-solving heuristics to reach a low regret under linear settings. Can a similar optimal result be achieved without the non-degeneracy assumption on constraints in the online convex allocation? 
Another challenge is the input model. Throughout this paper, we only discussed online allocation problems under the stochastic input model. The behavior of re-solving algorithms for other input models like random permutation or adversarial inputs remains largely unknown.

\end{sloppypar}

\newpage
\bibliography{reference}
\bibliographystyle{apalike}

\newpage

\begin{center}
{\bf\LARGE Supplement to `` Optimal Regularized Online Allocation by Adaptive
Re-Solving "}
\end{center}
\smallskip

\section{Proofs of Main Results}

\subsection{Proof of Lemma \ref{lemma:bounded_solution}} 
We prove the bound of the deterministic optimal solution. Consider $\Omega'_\mu=\left\{ -\nabla r (a)\middle| a \in \cZ \right\}$ . The bounded subgradient in Assumption \ref{asm:regularizer} suggests that the dual variable region $\Omega_\mu'$ we defined is bounded by $G$. We explain this definition by the optimal conditions of stochasic programming. Note that for problem \eqref{eq:dual_expected}, $\mu$ is unconstrained. The optimal condition suggests that 
\begin{equation*}
    \nabla r^*(-\mu^*) = \E b_t \nabla f^*_t(b_t^{\top}(\lambda^*+\mu^*))
\end{equation*}
if we assume fubini theorem holds. Then by the Fenchel conjugate, we have $\mu^*\in -\nabla r (\E b_t\tilde{x}_t )$.
This shows that by defining $\Omega'_\mu$ we indeed define the possible region that contains optimal solution $\mu^*$, i.e., $\mu^*\in \Omega_\mu$. Thus we have $\norm{\mu}_{\infty}\le G$.

For the second bound of $\norm{\lambda^*}_{\infty}$, we only need to check that $d^\top\lambda^*\le  2(\bar{f}+\bar{r})$ always holds. Otherwise if $d^\top \lambda^* >2(\bar{f}+\bar{r})$, we have 
\begin{equation*}
    \begin{aligned}
    D(\blambda^*,d) & =\E \sup_{x}\{f_t(x) - (\lambda^*+\mu^*)^\top b_t x_t \} +\sup_{a}\{r(a)+a^{\top}\mu^* \}+d^\top \lambda^* \ge \E f_t(0)+r(0)+ d^\top\lambda^* \\
    & > (\bar{f}+\Bar{r}) \ge D(\boldsymbol{0},d),
    \end{aligned}
\end{equation*}
which suggests that $\blambda^*$ is not optimal. Thus we have $d^\top\lambda^*\le  2(\bar{f}+\bar{r})$, i.e., $\norm{\lambda^*}_{\infty} \le  \frac{2(\bar{f}+\Bar{r})}{\underline{d}}$. The bound of empirical optimal solution $\blambda^*_T$ follows exactly the same argument.

The following proposition on the growth of second order term $s(\nu,d)$ will be useful in the development of our theory.
\begin{Proposition}
\label{prop:quadratic_growth}
Under Assumptions \ref{asm:basic}-\ref{asm:nondegeneracy}, the second order objective function $s(\nu,d)$ in stochastic program satisfies the following growth condition:
\begin{equation}
\underline{\cL}_s \norm{\nu-\nu^*}_2^2 \le s(\nu) \le \bar{\cL}_s\norm{\nu-\nu^*}_2^2,  
\end{equation}
where the constant $\underline{\cL}_s:= {\sigma_{\min}\underline{\cL}_f}/2$, $\bar{\cL}_s:= \Bar{b}^2\Bar{\cL}_f/2$ .
\end{Proposition}
\proof{}Proof:
By definition, we have
$$
\begin{aligned}
s(\nu) = D(\nu,\mu^*,d)- D(\nu^*,\mu^*,d)-\nabla_{\nu} D(\nu^*,\mu^*,d)^\top (\nu-\nu^*)\\
= \int_{0}^{1} \left[\nabla D_\nu(z(\nu- \nu^*)+\nu^*,\mu^* ,d)- \nabla D_\nu (\nu^*,\mu^*,d)\right]^\top (\nu - \nu^*)dz,  
\end{aligned}
$$
where $\nabla_\nu D(\nu,d) = \E b_t f^*_t(b_t^\top v )$. Then for any $z$, we have

\begin{equation*}
    \begin{aligned}
    &  \left[\nabla D_\nu(z(\nu- \nu^*)+\nu^*,\mu^* ,d)- \nabla D_\nu (\nu^*,\mu^*,d)\right]^\top (\nu - \nu^*) \\ 
  \le  &  \norm{\E b_t\nabla f_t^{*}(b_t^{\top}(z(\mu+\lambda-\mu^*-\lambda^*)+\mu^*+\lambda^*) -\E b_t\nabla f_t^{*}(b_t^{\top}(\mu^*+\lambda^*)) }_2\norm{\nu - \nu^*}_2\\
    \le & \norm{z \bar{\cL}_f \bar{b}   \E  \left[ b_t^\top(\mu+\lambda-\mu^*-\lambda^*)  \right]}_2\norm{\nu-\nu^*} \\
    \le & z\bar{\cL}_f \bar{b}^2 \norm{\nu-\nu^*}^2,
    \end{aligned}
\end{equation*}
where the second inequality is by Assumption \ref{asm:nondegeneracy:1}  when conditioned on $b_t$. By the integral of $z$ we have $ s(\nu) \le \bar{\cL}_s\norm{\nu-\nu^*}_2^2$. For the next direction, it is also clear that
\begin{equation*}
    \begin{aligned}
    & \left[\nabla D_\nu(z(\nu- \nu^*)+\nu^*,\mu^* ,d)- \nabla D_\nu (\nu^*,\mu^*,d)\right]^\top (\nu - \nu^*) \\ 
  =  &  \left(\E b_t\nabla f_t^{*}(b_t^{\top}(z(\mu+\lambda-\mu^*-\lambda^*)+\mu^*+\lambda^*) -\E b_t\nabla f_t^{*}(b_t^{\top}(\mu^*+\lambda^*)) \right)^\top (\nu - \nu^*) \\ 
  =  & \E \left[\E\left[ \langle   \nabla{f^*_t}(b_t^\top (z(\mu+\lambda-\mu^*-\lambda^*)+\mu^*+\lambda^*))-\nabla{f^*_t}(b_t^\top (\mu^*+\lambda^*)),  b_t^\top (\mu+\lambda-\mu^*-\lambda^*)\rangle\right] \middle| b_t\right] \\
    \ge &  z \underline{\cL}_f  \E \norm{ b_t^\top(\mu+\lambda-\mu^*-\lambda^*) }_2^2 \ge z \underline{\cL}_f \sigma_{\min}\norm{\mu+\lambda-\mu^*-\lambda^*}_{2}^2.
    \end{aligned}
\end{equation*}
\endproof

\subsection{Proof of Proposition \ref{prop:rademacher-prob}: Empirical Risk Minimization}\label{sec:erm}
In this proof, we generalize our discussion to a broader setting: we consider the convergence of classical ERM method. In many areas of statistics and machine learning research, we aims to solve the following problem, named Empirical Risk Minimization (ERM):  given $T$ empirical convex risk functions $\ell(\boldsymbol{\lambda},\xi_t): \mathbb{R}^m\to \mathbb{R},t\in[T] $ where $\xi_t$ are i.i.d realizations from an unknown distribution, with its population version $\mathfrak{L} (\boldsymbol{\lambda})= \E_{\xi} \ell(\boldsymbol{\lambda},\xi_t)$,  we seek to find a good parameter $\widehat{\boldsymbol{\lambda}} $ by minimizing the empirical risk 

\begin{equation*}
    \widehat{\boldsymbol{\lambda}} =\argmin_{\lambda \in \mathbb{R}^m} \Bar{\ell}(\boldsymbol{\lambda},\{\xi_t\}_{t=1}^{T} ) =\argmin_{\lambda \in \mathbb{R}^m} \frac{1}{T}\sum_{t=1}^{T}  \ell(\boldsymbol{\lambda},\xi_t)
\end{equation*}
as a proxy of the parameter that minimize the population risk $\boldsymbol{\lambda}^*=\argmin_{\boldsymbol{\lambda}} \mathfrak{L} (\boldsymbol{\lambda}) $. In statistic, such $\widehat{\boldsymbol{\lambda}}$ is also called M-estimator.

The following approach will show that, under second order growth condition of $\mathfrak{L} (\boldsymbol{\lambda})$, the ERM method provides a estimate $\widehat{\boldsymbol{\lambda}}$ with optimal convergence rate $\E \norm{\widehat{\boldsymbol{\lambda}} -{\boldsymbol{\lambda}}^* }^2=O(\frac{1}{T}) $


\begin{assumption}[Lipschitz continuity]\label{asm:erm-lip} Suppose the subgradient of each $\ell(\boldsymbol{\lambda},\xi_t)$ satisfies 
\begin{equation*}
\norm{\nabla\ell(\boldsymbol{\lambda},\xi_t)}\le L
\end{equation*}
\end{assumption}

\begin{assumption}[Second order growth]\label{asm:erm-growth} The population risk satisfies the following second order growth
\begin{equation*}
\begin{gathered}
\left\langle   \nabla\mathfrak{L} (\boldsymbol{\lambda})-\nabla\mathfrak{L} (\boldsymbol{\lambda}^*), \boldsymbol{\lambda}-\boldsymbol{\lambda}^* \right\rangle \ge \underline{\cL}_{\ell} \norm{\boldsymbol{\lambda}-\boldsymbol{\lambda}^*}^2_2 
\end{gathered}   
\end{equation*}
\end{assumption}

\begin{assumption}[Smoothness of the first moment]\label{asm:erm-moment} For any $\boldsymbol{\lambda}\in \mathbb{R}^m$, with $\delta =\norm{\boldsymbol{\lambda}-\boldsymbol{\lambda}^*} $, we have
\begin{equation*}
\E\sup_{\boldsymbol{\lambda}'\in \mathbb{B} ({\boldsymbol{\lambda}}^* , \delta) }  \norm{\nabla\ell(\boldsymbol{\lambda}',\xi_t)- \nabla\ell(\boldsymbol{\lambda}^*,\xi_t) }\le M \delta,
\end{equation*}
where $\mathbb{B}(\boldsymbol{\lambda}^*,r)=\left\{ \boldsymbol{\lambda}' : \norm{\boldsymbol{\lambda}'-\boldsymbol{\lambda}^*}\le r \right\} $.
\end{assumption}

To investigate the convergence of $\widehat{\boldsymbol{\lambda}}$, one classical approach is to compute the statistical complexity of function group $\left\{\ell(\boldsymbol{\lambda},\xi_t), \boldsymbol{\lambda}\in\Theta \right\}$ to get an uniform bound of $\sup_{\boldsymbol{\lambda}\in\Theta} \left[\Bar{\ell}(\boldsymbol{\lambda},\{\xi_t\}_{t=1}^{T} )- \mathfrak{L} (\boldsymbol{\lambda})\right]$. However, this approach may fail to reach the optimal convergence rate because it neglects the second order information. Instead, we consider the second order part of losses and improve our analyses by a localized argument near the optimal solution $\boldsymbol{\lambda}^*$. Equipped with localized Rademacher complexity which shares a similar idea as \cite{bartlett2005local}, we are able to derive a sharp local probabilistic bound of $\norm{\boldsymbol{\lambda}-\boldsymbol{\lambda}^*}$. To fix this idea, we define the second order part of our loss function:

$$
s(\boldsymbol{\lambda},\xi_t)= \ell(\boldsymbol{\lambda},\xi_t)- \langle \nabla \ell(\boldsymbol{\lambda}^*,\xi_t),  \boldsymbol{\lambda}-\boldsymbol{\lambda}^* \rangle - \ell(\boldsymbol{\lambda}^*,\xi_t), \ \bar{s}(\boldsymbol{\lambda})= \frac{1}{T}\sum_{t=1}^{T}s(\boldsymbol{\lambda},\xi_t)
$$
with its population version $S(\boldsymbol{\lambda}) = \E s(\boldsymbol{\lambda},\xi_t) =\mathfrak{L} (\boldsymbol{\lambda}) - \langle \nabla \mathfrak{L} (\boldsymbol{\lambda}),  \boldsymbol{\lambda}-\boldsymbol{\lambda}^* \rangle - \mathfrak{L} (\boldsymbol{\lambda}^*)$. Define localized Rademacher complexity of $s$ within a small neighbourhood of $\boldsymbol{\lambda}^*$ as

\begin{equation*}
    \mathcal{R}_\varepsilon = \E_{\xi} \E_\sigma \left[ \sup_{\boldsymbol{\lambda}\in \mathbb{B}(\boldsymbol{\lambda}^*,\varepsilon) } \frac{1}{T}\sum_{t=1}^{T} \sigma_t s(\boldsymbol{\lambda},\xi_t)  \right],
\end{equation*}
where $\sigma_t$ are independent Rademacher random variables. We have the following result:
\begin{Proposition}\label{prop:rademacher}
Under Assumption \ref{asm:erm-lip}-\ref{asm:erm-moment}, the following inequality holds
\begin{equation*}
     \mathcal{R}_\varepsilon \le \sqrt{2m \log(3K) }\frac{2L\varepsilon}{\sqrt{T}} +\frac{{M} \varepsilon^2 }{K},
\end{equation*}
for any constant $K>0$. Consequently, if $\varepsilon \ge \frac{64\sqrt{2}L }{\sqrt{T} \underline{\cL}_{\ell} } \sqrt{\log \frac{100{M} }{\underline{\cL}_{\ell}} }$, we have the following probabilistic bound:

\begin{equation*}
    \PP\left( \norm{ \widehat{\boldsymbol{\lambda}} -\boldsymbol{\lambda}^*}\ge \varepsilon \right)\le m\exp\left({-\frac{T\underline{\cL}_{\ell}^2\varepsilon^2  }{8L^2m}}\right)+ \exp\left({-\frac{T\underline{\cL}_{\ell}^2\varepsilon^2  }{5000L^2}}\right)
\end{equation*}

\end{Proposition}
\begin{proof}
Define $\mathcal{K}$ as a $\frac{\varepsilon}{K}$-cover of the set $\boldsymbol{\lambda}\in \mathbb{B}(\boldsymbol{\lambda}^*,\varepsilon)$, then by the covering number of a ball, it is valid that $\log \abs{\mathcal{K}} \le m\log(3 K)  $. Define a projection $\mathcal{K}(\boldsymbol{\lambda})$ that project each $\boldsymbol{\lambda}\in \mathbb{B}(\boldsymbol{\lambda}^*,\varepsilon)$ onto the closest element in the cover $\mathcal{K}$. By Assumption \ref{asm:erm-lip}, we have a uniform bound $\abs{s(\boldsymbol{\lambda},\xi_t)}\le 2L \varepsilon$. Then it follows that:

\begin{equation*}
    \begin{aligned}
     \mathcal{R}_\varepsilon = &\E_{\xi} \E_\sigma \left[ \sup_{\boldsymbol{\lambda}\in \mathbb{B}(\boldsymbol{\lambda}^*,\varepsilon) } \frac{1}{T}\sum_{t=1}^{T} \sigma_t s(\boldsymbol{\lambda},\xi_t)  \right] \\
     \le &\E_{\xi} \left[\E_\sigma \sup_{\boldsymbol{\lambda}\in \mathcal{K}}  \frac{1}{T}\sum_{t=1}^{T} \sigma_t s(\boldsymbol{\lambda},\xi_t) +\E_\sigma\sup_{\boldsymbol{\lambda}\in \mathbb{B}(\boldsymbol{\lambda}^*,\varepsilon),\boldsymbol{\lambda}'= \mathcal{K}(\boldsymbol{\lambda})}  \frac{1}{T}\sum_{t=1}^{T} \sigma_t (s(\boldsymbol{\lambda},\xi_t)-s(\boldsymbol{\lambda}',\xi_t))  \right] \\
      \le & \sqrt{2m \log(3K) }\frac{2L\varepsilon}{\sqrt{T}}+ \E_{\xi} \left[E_\sigma\sup_{\boldsymbol{\lambda}\in \mathbb{B}(\boldsymbol{\lambda}^*,\varepsilon),\boldsymbol{\lambda}'= \mathcal{K}(\boldsymbol{\lambda})}  \frac{1}{T}\sum_{t=1}^{T} \sigma_t (s(\boldsymbol{\lambda},\xi_t)-s(\boldsymbol{\lambda}',\xi_t))  \right],
    \end{aligned}
\end{equation*}
where the second inequality is by Massart's finite class lemma. We focus on controlling the second term by computing the first order moment of $\abs{s(\boldsymbol{\lambda},\xi_t)-s(\boldsymbol{\lambda}',\xi_t)}$:

\begin{equation*}
    \begin{aligned}
     \E &\sup_{\boldsymbol{\lambda}\in \mathbb{B}(\boldsymbol{\lambda}^*,\varepsilon),\boldsymbol{\lambda}'= \mathcal{K}(\boldsymbol{\lambda})}\abs{s(\boldsymbol{\lambda},\xi_t)-s(\boldsymbol{\lambda}',\xi_t)} =  \E\sup_{\boldsymbol{\lambda}\in \mathbb{B}(\boldsymbol{\lambda}^*,\varepsilon),\boldsymbol{\lambda}'= \mathcal{K}(\boldsymbol{\lambda})} \abs{ \ell(\boldsymbol{\lambda},\xi_t) - \ell(\boldsymbol{\lambda}',\xi_t)- \langle \nabla\ell(\boldsymbol{\lambda}^*,\xi_t)  , \boldsymbol{\lambda}-\boldsymbol{\lambda}'\rangle } \\
     &= \E \sup_{\boldsymbol{\lambda}\in \mathbb{B}(\boldsymbol{\lambda}^*,\varepsilon),\boldsymbol{\lambda}'= \mathcal{K}(\boldsymbol{\lambda})}
     \left(\int_{0}^1 v^\top(\nabla\ell(\boldsymbol{\lambda}'+vz,\xi_t)- \nabla\ell(\boldsymbol{\lambda}^*,\xi_t))  dz\right), \text{ where }v=\boldsymbol{\lambda}-\boldsymbol{\lambda}' \\
     &\le \sup \norm{v} \cdot \E\sup_{\boldsymbol{\lambda}\in \mathbb{B}(\boldsymbol{\lambda}^*,\varepsilon),\boldsymbol{\lambda}'= \mathcal{K}(\boldsymbol{\lambda})}
       \int_{0}^1 \norm{\nabla\ell(\boldsymbol{\lambda}'+vz,\xi_t)- \nabla\ell(\boldsymbol{\lambda}^*,\xi_t)} dz \\
       &\le \frac{\varepsilon}{K} \int_{0}^1 \E\sup_{\boldsymbol{\lambda}\in \mathbb{B}(\boldsymbol{\lambda}^*,\varepsilon),\boldsymbol{\lambda}'= \mathcal{K}(\boldsymbol{\lambda})}\norm{\nabla\ell(\boldsymbol{\lambda}'+vz,\xi_t)- \nabla\ell(\boldsymbol{\lambda}^*,\xi_t)} dz 
       \le \frac{{M}\varepsilon^2}{K},
    \end{aligned}
\end{equation*}
where the last inequality we use the Assumption \ref{asm:erm-moment}. Then it follows that 

\begin{equation*}
    \begin{aligned}
      \E_{\xi} &\left[E_\sigma\sup_{\boldsymbol{\lambda}\in \mathbb{B},\boldsymbol{\lambda}'= \mathcal{K}(\boldsymbol{\lambda})}  \frac{1}{T}\sum_{t=1}^{T} \sigma_t (s(\boldsymbol{\lambda},\xi_t)-s(\boldsymbol{\lambda}',\xi_t))  \right] \le \frac{ { \sum_{t=1}^T \E_\xi \sup_{\boldsymbol{\lambda}\in \mathbb{B},\boldsymbol{\lambda}'= \mathcal{K}(\boldsymbol{\lambda})} \abs{s(\boldsymbol{\lambda},\xi_t)-s(\boldsymbol{\lambda}',\xi_t)}  } }{{T}}\\
      &\le \frac{{M}\varepsilon^2}{K},
    \end{aligned}
\end{equation*}
which proves the first statement. We then prove the second statement by choosing a suitable $K$.

If $\widehat{\boldsymbol{\lambda}}$ which minimizes the empirical risk satisfies  $\norm{ \widehat{\boldsymbol{\lambda}} -\boldsymbol{\lambda}^*}\ge \varepsilon$, then by the convexity of $\bar{\ell}$, there exists a $\boldsymbol{\lambda}\in \mathbb{B}(\boldsymbol{\lambda}^*,\varepsilon) $ such that $\bar{\ell}(\boldsymbol{\lambda})-\bar{\ell}(\boldsymbol{\lambda}^*)\le 0$. Together with
the second order growth Assumption \ref{asm:erm-growth},  we have

\begin{equation*}
   \begin{aligned}
    \bar{s}(\boldsymbol{\lambda}) -S(\boldsymbol{\lambda}) & = \bar{\ell}(\boldsymbol{\lambda})-\bar{\ell}(\boldsymbol{\lambda}^*)- (\mathfrak{L} (\boldsymbol{\lambda}) - \mathfrak{L} (\boldsymbol{\lambda}^*)) - \langle \nabla\bar{\ell}(\boldsymbol{\lambda}^*)- \nabla \mathfrak{L} (\boldsymbol{\lambda}^*), \boldsymbol{\lambda} -\boldsymbol{\lambda}^* \rangle \\
    &\le -\frac{\underline{\cL}_{\ell}}{2}\varepsilon^2+ \norm{\nabla\bar{\ell}(\boldsymbol{\lambda}^*)- \nabla \mathfrak{L} (\boldsymbol{\lambda}^*)} \varepsilon.
    \end{aligned}
\end{equation*}
By Hoeffding's concentration inequality, we have

\begin{equation*}
    \PP\left( \norm{\nabla\bar{\ell}(\boldsymbol{\lambda}^*)- \nabla \mathfrak{L} (\boldsymbol{\lambda}^*)} \ge \frac{\underline{\cL}_{\ell}}{4}\varepsilon \right) \le m\exp\left({-\frac{T\underline{\cL}_{\ell}^2\varepsilon^2  }{8L^2m}}\right)
\end{equation*}
Define the event that inequality $\norm{\nabla\bar{\ell}(\boldsymbol{\lambda}^*)- \nabla \mathfrak{L} (\boldsymbol{\lambda}^*)} \ge \frac{\underline{\cL}_{\ell}}{4}\varepsilon$ holds as $\mathcal{E}_1$. Then under $\{\norm{ \widehat{\boldsymbol{\lambda}} -\boldsymbol{\lambda}^*}\ge \varepsilon\}\cap\mathcal{E}_1^c  $, we have 

\begin{equation*}
     \sup_{\boldsymbol{\lambda}\in \mathbb{B}(\boldsymbol{\lambda}^*,\varepsilon)  } \abs{\bar{s}(\boldsymbol{\lambda}) -S(\boldsymbol{\lambda})} \ge  \frac{\underline{\cL}_{\ell}}{4}\varepsilon^2.
\end{equation*}
Choosing $K=\frac{32 {M} }{\underline{\cL}_{\ell} }$. When $\varepsilon \ge \frac{64\sqrt{2}L }{\sqrt{T} \underline{\cL}_{\ell} } \sqrt{\log \frac{100{M} }{\underline{\cL}_{\ell}} } $, we have

\begin{equation*}
    2 \mathcal{R}_\varepsilon \le \frac{\underline{\cL}_{\ell}}{8}\varepsilon^2,
\end{equation*}
thus we have the following inequality:
\begin{equation*}
    \sup_{\boldsymbol{\lambda}\in \mathbb{B}(\boldsymbol{\lambda}^*,\varepsilon)  } \abs{\bar{s}(\boldsymbol{\lambda}) -S(\boldsymbol{\lambda})} \ge 2\mathcal{R}_\varepsilon + \frac{\underline{\cL}_{\ell}}{8}\varepsilon^2.
\end{equation*}
By the convergence theory of empirical process \citep{koltchinskii2011oracle,boucheron2005theory}, 
\begin{equation*}
     \PP\left(\sup_{\boldsymbol{\lambda}\in \mathbb{B}(\boldsymbol{\lambda}^*,\varepsilon)  } \abs{\bar{s}(\boldsymbol{\lambda}) -S(\boldsymbol{\lambda})} \ge 2\mathcal{R}_\varepsilon + \frac{6L\varepsilon z}{\sqrt{T}} \right)\le \exp{(-\frac{z^2}{2})},
\end{equation*}
thus we conclude that $\PP(\{\norm{ \widehat{\boldsymbol{\lambda}} -\boldsymbol{\lambda}^*}\ge \varepsilon\}\cap\mathcal{E}_1^c  )\le \exp\left({-\frac{T\underline{\cL}_{\ell}^2\varepsilon^2  }{5000L^2}}\right) $, i.e.,

\begin{equation*}
    \PP\left( \norm{ \widehat{\boldsymbol{\lambda}} -\boldsymbol{\lambda}^*}\ge \varepsilon \right)\le m\exp\left({-\frac{T\underline{\cL}_{\ell}^2\varepsilon^2  }{8L^2m}}\right)+ \exp\left({-\frac{T\underline{\cL}_{\ell}^2\varepsilon^2  }{5000L^2}}\right).
\end{equation*}
\end{proof}

\begin{theorem}
The following bound holds for the convergence rate of $\widehat{\boldsymbol{\lambda}}$:

\begin{equation*}
    \E\norm{ \widehat{\boldsymbol{\lambda}} -\boldsymbol{\lambda}^*}^2\le \left(\frac{512L^2 }{\underline{\cL}_{\ell}^2 }\log \frac{100{M} }{\underline{\cL}_{\ell}} + \frac{8 m^2 L^2 +5000L^2}{\underline{\cL}_{\ell}^2 }\right)\frac{1}{T}
\end{equation*}
\end{theorem}

\begin{proof}
This is a direct consequence of Proposition \ref{prop:rademacher}. By the integral formula of expectation, we have
\begin{equation*}
    \begin{aligned}
      &\E\norm{ \widehat{\boldsymbol{\lambda}} -\boldsymbol{\lambda}^*}^2 = \int_{0}^\infty \PP(\norm{ \widehat{\boldsymbol{\lambda}} -\boldsymbol{\lambda}^*}\ge \sqrt{z} )dz \\
      &\le \left(\frac{512L^2 }{\underline{\cL}_{\ell}^2 }\log \frac{100{M} }{\underline{\cL}_{\ell}}\right)\frac{1}{T} + \int_{c}^\infty \PP(\norm{ \widehat{\boldsymbol{\lambda}} -\boldsymbol{\lambda}^*}^2\ge z )dz, \text{ where } c=\left(\frac{512L^2 }{\underline{\cL}_{\ell}^2 }\log \frac{100{M} }{\underline{\cL}_{\ell}}\right)\frac{1}{T} \\
      &\le \left(\frac{512L^2 }{\underline{\cL}_{\ell}^2 }\log \frac{100{M} }{\underline{\cL}_{\ell}}\right)\frac{1}{T}+ \int_{0}^\infty m\exp\left({-\frac{T\underline{\cL}_{\ell}^2 z }{8L^2m}}\right)+ \exp\left({-\frac{T\underline{\cL}_{\ell}^2 z }{5000L^2}}\right) dz \\
      &\le \left(\frac{512L^2 }{\underline{\cL}_{\ell}^2 }\log \frac{100{M} }{\underline{\cL}_{\ell}} + \frac{8 m^2 L^2 +5000L^2}{\underline{\cL}_{\ell}^2 }\right)\frac{1}{T},
    \end{aligned}
\end{equation*}
which finishes the proof.
\end{proof}

By simply equating $\ell(\boldsymbol{\lambda},\xi_t)$ with Fenchel conjugate $f^*_t(\boldsymbol{\lambda})$ in online convex allocation, we are able to prove the optimal dual convergence rate $O(\frac{1}{T})$. Notice that, here our Assumption \ref{asm:k-thmoment} is equivalent to the Assumption \ref{asm:erm-moment} we used in the proof.

\subsection{Proof of Proposition \ref{prop:dual_conv_prob}}\label{app:dual_conv}
For any given $\varepsilon>0$, we define the neighbourhood of $\nu^*$ for given $\varepsilon$ as $$\Omega_{\nu}(\varepsilon):=\left\{\nu: \norm{\nu-\nu^*}_\infty \le 4H \varepsilon \right\}.$$ We then construct a good event $\mathcal{E}(\varepsilon)$ with pro only depends on $\varepsilon$ that under this good event, the convex function $\bar s_T(\nu,d) $ is larger than a quadratic function in $ \Omega_{\nu}(\varepsilon)$, which serves as a lower bound of dual function.
The construction of this good event $\mathcal{E}(\varepsilon)$ is based on the following splitting scheme and concentration of objective function:
\begin{enumerate}
    \item We first split $\Omega_{\nu}(\varepsilon)$ into multiple cubes layer by layer and in each single cube, we control the difference of second order terms between all the $\nu$ in the cube and the central point of the cube.
    \item Then we uniformly control the deviation of second order terms for all central points.
\end{enumerate}




We now discuss the second order term $\Bar{s}_T(\nu,d)$ defined in \eqref{eq:decomp}. To derive an uniform lower bound of $\bar s_T(\nu,d)$, we do the following split on $\Omega_{\boldsymbol{\nu}}(\varepsilon)$ according to \cite{huber1967under}. 

Define set $\Omega^k_{\nu}(\varepsilon)=\left\{\nu|\norm{\nu-\nu^*}_{\infty}\le q^k 4H \varepsilon,\norm{\mu-\mu^*}_{\infty}\le q^k 4H \varepsilon\right\}, 0\le k\le N  $, where $q\in (0,1)$ and $N\in \mathbb{N}_+$ will be identified later. This split divides $\Omega_{\nu}(\varepsilon)$ into $N$ layers $\{\Omega^{k-1}_{\nu}(\varepsilon)\backslash \Omega^k_{\nu}(\varepsilon)\}_{k=1}^{N}$ and a center cube $\Omega^N_{\nu}(\varepsilon)$. We then split each layer into disjoint cubes $\{\bar\Omega^{kl}(\varepsilon)\}_{l=1}^{l_k}$ with edges of length $(1-q)q^{k-1}4H \varepsilon$, and denote the center cube by $\bar\Omega^{N1}(\varepsilon)$. \cite{huber1967under} shows that there are at most $(2N)^{m}$ cubes. This split is not unique to get the desired convergence order but it makes our result tighter. The center of each cube $\bar\Omega^{kl}(\varepsilon)$ is defined as $\nu_{kl}$.
Define $\bar\nu_{kl}=\arg\max_{\nu\in \bar\Omega^{kl}(\varepsilon)}\norm{\nu-\nu^*}_2$, and 
\begin{equation}
  \begin{aligned}
     \Gamma^{kl}_t &=\max_{\nu\in \bar\Omega^{kl}(\varepsilon)}\left[s_t(\nu_{kl},d)-s_t(\nu,d)\right] \\
  \end{aligned}  
\end{equation}
Then for $k \in \{0,\dots,N-1\}$, and $\forall \nu\in \bar\Omega^{kl}(\varepsilon)$, $\bar{s}_T$ can be decomposed as 
\begin{equation}
    \label{eq:sec_decomp}
    \begin{aligned}
       \bar{s}_T(\nu,d)=&\frac{1}{T}\sum_{t=1}^{T}s_t(\nu,d)- \frac{1}{T}\sum_{t=1}^{T}s_t(\nu_{kl},d)  
       + \frac{1}{T}\sum_{t=1}^{T}s_t(\nu_{kl},d)\\
       \ge & \underbrace{\E s_t(\nu_{kl},d) -\E \Gamma^{kl}_t  }_{\ref{eq:sec_decomp}.1} 
       + \underbrace{-\frac{1}{T}\sum_{t=1}^{T}\Gamma^{kl}_t +\E \Gamma^{kl}_t}_{\ref{eq:sec_decomp}.2}
       + \underbrace{\frac{1}{T}\sum_{t=1}^{T}s_t(\nu_{kl},d)-\E s_t(\nu_{kl},d)}_{\ref{eq:sec_decomp}.3}
    \end{aligned}
\end{equation}
We study the lower bounds of these 3 terms in \eqref{eq:sec_decomp} respectively.

\textbf{Lower bound of \ref{eq:sec_decomp}.1}:
\begin{equation}
\label{gamma_1}
\begin{aligned}
   \E \Gamma^{kl}_t=&\E \max_{\nu\in \bar\Omega^{kl}(\varepsilon)} \left[f_t^*(b_t^\top(\lambda_{kl}+\mu_{kl}))-f_t^*(b_t^\top(\lambda_{}+\mu_{})) -\nabla f_t^*(\lambda^*+\mu^*)^\top b_t^\top (\lambda_{kl}+\mu_{kl}-\lambda-\mu)\right]\\
   = & \E \max_{\nu\in \bar\Omega^{kl}(\varepsilon)} \left[\int_0^1v_1^\top(\nu)\left[\nabla f_t^*(b_t^\top(\lambda+\mu)+v_1\cdot z)-\nabla f_t^*(b_t^\top(\lambda^*+\mu^*))\right]dz  \right]\\
   \le & \Bar{b}\max_{\nu\in \bar\Omega^{kl}(\varepsilon)} \norm{\nu-\nu_{kl} }_2\cdot \left[\int_0^1\E \max_{\nu\in \bar\Omega^{kl}(\varepsilon)}\norm{\nabla f_t^*(b_t^\top(\nu)+v_1\cdot z)-\nabla f_t^*(b_t^\top(\nu^*))}dz  \right] \\
   \le & L_1\bar b^2 (\max_{\nu\in \bar\Omega^{kl}(\varepsilon)}\norm{\nu_{kl}-\nu}_2)\cdot \norm{\bar\nu_{kl}-\nu^*}_2
\end{aligned}
\end{equation}
where $v_1(\nu)=b_t^\top(\lambda_{kl}+\mu_{kl}-\lambda_{}-\mu_{})$ is the direction vector, and the second inequality is obtained by Assumption \ref{asm:k-thmoment}.

According to Proposition \ref{prop:quadratic_growth}, we have
\begin{equation*}
    \begin{aligned}
       \E s_t(\nu_{kl},d) \ge \underline{\cL}_{s}\norm{\nu_{kl}-\nu^*}_2^2 
    \end{aligned}
\end{equation*}
So for the first term, it is clear that
\begin{equation}
\label{eq:sec_term_1}
    \begin{aligned}
        -\E \Gamma^{kl}_t+\E s_t(\nu_{kl},d) \ge  \underline{\cL}_{s}(\norm{\nu_{kl}-\nu^*}_2^2-L_1\bar b^2 (\max_{\nu\in \bar\Omega^{kl}(\varepsilon)}\norm{\nu_{kl}-\nu}_2)\cdot \norm{\bar\nu_{kl}-\nu^*}_2 \\
    \end{aligned}
\end{equation}

\textbf{Lower bound of \ref{eq:sec_decomp}.2}: Since the gradients $\norm{\nabla f_t^*}_{\infty}$ is bounded by $D$, by the integral form of $\Gamma_{kl}$ in the second equality of \ref{gamma_1}, we also have:

\begin{equation*}
    \begin{aligned}
       \norm{\Gamma^{kl}_t}_2\le& 2\sqrt{n} \bar b D \max_{\nu\in \bar\Omega^{kl}(\varepsilon)} \norm{\nu-\nu_{kl}}_2,
    \end{aligned}
\end{equation*}
for any $t\in [T]$. Define event 
\begin{equation}
\label{eq:event_gamma}
    \mathcal{E}_{kl,1}(\varepsilon_1)=\left\{-\frac{1}{T}\sum_{t=1}^{T}\Gamma^{kl}_t+ \E \Gamma^{kl}_t< -2\varepsilon_1 \sqrt{n} \bar b D \max_{\nu\in \bar\Omega^{kl}(\varepsilon)} \norm{\nu-\nu_{kl}}_2 \right\}.
\end{equation}

Then according to Hoeffding's inequality, $\PP(\mathcal{E}_{kl,1}(\varepsilon_1))\le \exp(-\frac{T\varepsilon^2_1 }{2})$

\textbf{Lower bound of \ref{eq:sec_decomp}.3}:
We calculate the norm of each $s_t(\nu_{kl},d)$:

\begin{equation}
\label{eq:sec_term_3}
    \begin{aligned}
    \norm{s_t(\nu_{kl},d)}_2=&\left\| \left[\int_0^1v_2^\top\left[\nabla f_t^*(b_t(\lambda^*+\mu^*)+v_2\cdot z)-\nabla f_t^*(\lambda^*+\mu^*)\right]dz \right]dz\right\|_2 \\
   \le &2\sqrt{n} \bar b D \norm{\bar\nu_{kl}-\nu^*}_2,
\end{aligned}
\end{equation}
for any $t\in [T]$, where $v_2=b_t^\top(\nu_{kl}-\nu^*_{})$ is the direction vector. Define event 
\begin{equation}\label{eq:event_s}
    \mathcal{E}_{kl,2}(\varepsilon_2)=\left\{ \frac{1}{T}\sum_{t=1}^{T}s_t(\nu_{kl},d)-\E s_t(\nu_{kl},d)<-2\varepsilon_2 \sqrt{n}\bar b D \norm{\bar\nu_{kl}-\nu^*}_2\right\}.
\end{equation}
Then we have $\PP(\mathcal{E}_{kl,2})\le\exp(-\frac{T\varepsilon^2_2}{2})$ by Hoeffding's inequality. Now we would like to make all the quantities in the lower bound uniform by leveraging the splitting scheme. From the split, we have 

\begin{equation*}
    \begin{aligned}
        \max _{\nu \in \bar\Omega^{k l}(\varepsilon)}\left\|\nu-\nu_{k l}\right\|_{2}&=\sqrt{m}(1-q) q^{k-1} 4H\varepsilon,\\
         \left\|\nu^{*}-\nu_{k l}\right\|_{2} &\geq q^{k} 4H\varepsilon.
    \end{aligned}
\end{equation*}
And also 
\begin{equation*}
    \begin{aligned}
\left\|\nu^{*}-\bar{\nu}_{k l}\right\|_{2} & \leq\left\|\nu^{*}-\nu_{k l}\right\|_{2}+\max _{\bar\Omega^{k l}(\varepsilon)}\left\|\nu-\nu_{k l}\right\|_{2} \\
\max _{\nu \in \bar\Omega^{k l}(\varepsilon)}\left\|\nu-\nu_{k l}\right\|_{2}&\le \frac{\sqrt{m}(1-q)}{q}\norm{\nu^*-\nu_{kl}}_2\le \frac{\sqrt{m}(1-q)}{q}\norm{\nu^*-\bar\nu_{kl}}_2.
    \end{aligned}
\end{equation*}

Thus  we have the following result for the \ref{eq:sec_decomp}.1 term in \eqref{eq:sec_term_1}.

$$
\begin{aligned}
        -\E \Gamma^{kl}_t+\E s_t(\nu_{kl},d) \ge & \underline{\cL}_{s}(\norm{\nu_{kl}-\nu^*}_2^2-L_1\bar b^2 (\max_{\nu\in \bar\Omega^{kl}(\varepsilon)}\norm{\nu_{kl}-\nu}_2)\cdot \norm{\bar\nu_{kl}-\nu^*}_2\\
        \ge & \frac{\underline{\cL}_{s}}{\left(1+\frac{\sqrt{m}(1-q)}{q}\right)^2} \norm{\Bar{\nu}_{kl}-\nu^*}_2^2- \frac{\sqrt{m}(1-q)}{q} \cdot L_1\bar b^2\norm{\Bar{\nu}_{kl}-\nu^*}_2^2
\end{aligned}
$$

So there exists $\underline{q}=\frac{\sqrt{m}}{\sqrt{m}+ 1 \wedge \frac{\underline{\cL}_s}{4L_1\bar b^2} } $ such that when $q\ge \underline{q}$, $\frac{\sqrt{m}(1-q)}{q}\le 1 \wedge \frac{\underline{\cL}_s}{4L_1\bar b^2}$, and  
$$ \frac{\underline{\cL}_{s}}{(1+\frac{\sqrt{m}(1-q)}{q})^2}-\frac{\sqrt{m}(1-q)}{q} \cdot L_1\bar b^2 \ge \underline{\cL}_{s}/2.
$$

Choose $q=\underline{q}\vee\frac{1}{2}$. Then for the \ref{eq:sec_decomp}.1 term in \eqref{eq:sec_term_1} we have 
\begin{equation}
    \label{eq:sec_lb_1}
   -\E \Gamma^{kl}_t+\E s_t(\nu_{kl},d)\ge \frac{\underline{\cL}_{s}}{2}\norm{\bar\nu_{kl}-\nu^*}_2^2.
\end{equation}

Let $\varepsilon_1=\varepsilon_2=\sqrt{m\log m}\varepsilon$. For \ref{eq:sec_decomp}.2, under event $\mathcal{E}_{kl,1}^c(\varepsilon_1)$ in \eqref{eq:event_gamma} we have

\begin{equation}\label{eq:sec_lb_2}
    \begin{aligned}
         -\frac{1}{T}\sum_{t=1}^{T}\Gamma^{kl}_t+ \E \Gamma^{kl}_t &\ge -2\varepsilon \sqrt{nm\log m} \bar b D  (\max_{\nu\in \bar\Omega^{kl}(\varepsilon)}\norm{\nu-\nu^*}_2) \\
         & \ge -2 \varepsilon  \sqrt{nm\log m} \bar b D \frac{\sqrt{m}(1-q)}{q} \norm{\bar\nu_{kl}-\nu^*}_2.
    \end{aligned}
\end{equation}

For \ref{eq:sec_decomp}.3, under event $\mathcal{E}_{kl,2}^c(\varepsilon)$ in \eqref{eq:event_s} we have

\begin{equation}\label{eq:sec_lb_3}
    \frac{1}{T}\sum_{t=1}^{T}s_t(\nu_{kl},d)-\E s_t(\nu_{kl},d)\ge-2\varepsilon \sqrt{nm\log m}\bar b D \norm{\bar\nu_{kl}-\nu^*}_2.
\end{equation}

Now we combine second order lower bounds in \eqref{eq:sec_lb_1}, \eqref{eq:sec_lb_2}, \eqref{eq:sec_lb_3} together under the desired good event
\begin{equation*}
    \begin{aligned}
         \mathcal{E}(\varepsilon)= \cap_{k=1}^{N}\cap_{l}(\mathcal{E}_{kl,1}^c(\varepsilon)\cap\mathcal{E}^c_{kl,2}(\varepsilon)),
    \end{aligned}
\end{equation*}
where we choose $N$ by setting the radius of $\bar\Omega^{N1}(\varepsilon)$: $\sqrt{m} q^{N} 4H\varepsilon \le 2H\varepsilon$, i.e., $$N= \lceil\log_{q}(\frac{1}{2\sqrt{m}}) \rceil\le \frac{4L_1\bar b^2 }{ \underline{\cL}_s}\sqrt{m}\log\sqrt{m}.$$ 
Under $ \mathcal{E}(\varepsilon)$, for any $\nu\in\Omega_{\nu}(\varepsilon)$ satisfying $\norm{\nu-\nu^*}_2> 2H\varepsilon$, there exists $k=\{0,\dots,N-1\}$ and $l$ such that $\nu\in \bar\Omega^{kl}(\varepsilon)$, and 
\begin{equation*}
    \begin{aligned}
             \bar s_T(\nu,d) \ge& \frac{\underline{\cL}_{s}}{2}\norm{\bar\nu_{kl}-\nu^*}_2^2 -2\varepsilon \sqrt{n}\Bar{b}D(1+\frac{\sqrt{m}(1-q)}{q})\norm{\bar\nu_{kl}-\nu^*}_2 \\
             \ge & \frac{\underline{\cL}_{s}}{2}\norm{\bar\nu_{kl}-\nu^*}_2^2 - 4\varepsilon \sqrt{n m\log m}\Bar{b}D \norm{\bar\nu_{kl}-\nu^*}_2
     \\
    \end{aligned}
\end{equation*}

Compute the probability of $\mathcal{E}(\varepsilon)$ we can show that 
$$
\begin{aligned}
     \PP(\mathcal{E}(\varepsilon))\ge & 1 - \sum_{0\le k \le N-1,\ l} \left(\PP(\mathcal{E}_{kl,1}(\varepsilon))+ \PP(\mathcal{E}_{kl,2}(\varepsilon))\right) \\
    \ge & 1- 2(2\lceil\log_{q}(\frac{1}{2\sqrt{m}}) \rceil)^{m} \exp(-\frac{m\log m T\varepsilon^2}{2})\ge 1- 2\exp({-\frac{m\log m(T\varepsilon^2-1) }{4}})
\end{aligned}
$$
The following Lemma can show the concentration of first order term:

\begin{lemma}\label{lemma:first_order}
Under Assumptions \ref{asm:basic}-\ref{asm:nondegeneracy}, the concentration of the gradient in the first order term $\bar \phi_{T,\nu}(\nu^*,d)$ satisfies
\begin{equation}
    \begin{aligned}
    \PP(\norm{\bar \phi_T(\nu^*,d)-\nabla D_{\nu}(\nu^*,d)}_2>\epsilon)\le 2m\exp (-\frac{T\epsilon^2}{2m \sqrt{n}\bar b D}),
    \end{aligned}
\end{equation}
for any $\epsilon>0$.
\end{lemma}
\proof{Proof:} According to Hoeffding's inequality, we have 
\begin{equation*}
\begin{aligned}
     \PP (\abs{\left(\bar  \phi_T(\nu^*,d)\right)_i-\left(\nabla D_{\nu}(\nu^*,d)\right)_i}>\epsilon/\sqrt{m})\le 2 \exp (-\frac{T\epsilon^2}{2m\sqrt{n}\Bar{b}D })
    \end{aligned}
\end{equation*}
for $\forall i \in [m]$. Combining all $m$ dimensions together we conclude that
$$
\PP(\norm{\frac{1}{T}\sum_{t=1}^{T}\phi_t(\blambda^*,d)-\nabla D(\blambda^*,d)}>\varepsilon)\le 2m\exp (-\frac{T\varepsilon^2}{2m\sqrt{n}\Bar{b}D}).
$$
\endproof

For the first order term, denote event $\mathcal{E}_0(\varepsilon_0)=\{\norm{\Bar{\phi}_T(\blambda,d)-\nabla D(\blambda^*,d)}_2>\varepsilon_0\}$. Take $\varepsilon_0=\varepsilon \sqrt{n m\log m}\Bar{b}D$ Then by Lemma \ref{lemma:first_order}, we have $$\PP(\mathcal{E}_0(\varepsilon_0))\le 2m\exp (-\frac{T\sqrt{n}\log m\Bar{b}D\varepsilon^2 }{2 }).$$ 
Under event $\mathcal{E}_0^c(\varepsilon)\cap \mathcal{E}(\varepsilon)$, we have 
\begin{equation}\label{eq:first_order_lb}
    \begin{aligned}
     \Bar{D}_T(\blambda,d)-\Bar{D}_T(\blambda^*,d)
     &\ge \Bar{s}_T(\nu,d)+ \left\langle\Bar{\phi}_{T,\nu}(\nu^*,d)-\nabla_{\nu} D(\blambda^*,d), \nu -\nu^*\right\rangle \\
     &\ge \frac{\sigma_{\min} \underline{\cL}_f}{4} \norm{ \nu'-\nu^*}_2^2 - 5\varepsilon \sqrt{n m\log m}\Bar{b}D \norm{\bar\nu_{kl}-\nu^*}_2 \\
     &=\frac{\sigma_{\min} \underline{\cL}_f}{4} \norm{ \nu'-\nu^*}_2^2 -\frac{\sigma_{\min} \underline{\cL}_f}{4}\cdot 2H\varepsilon \norm{ \nu'-\nu^*}_2,
 \end{aligned}
\end{equation}
where we define $H=10\sqrt{n m\log m}\Bar{b}D/(\sigma_{\min} \underline{\cL}_f)$.

We now show how the first inequality leads to the probabilistic bound of $\norm{\nu-\nu^*}_2$. By the definition of $\Bar{s}_T(\nu,d)$, if the first inequality holds, an argument similar to \eqref{eq:sT-ub} will lead to 
\begin{equation}\label{eq:dual-conv-partition}
    \begin{aligned}
     \Bar{D}_T(\blambda,d)-\Bar{D}_T(\blambda^*,d)&\ge \Bar{s}_T(\nu,d)+ \left\langle\Bar{\phi}_T(\blambda^*,d), \blambda-\blambda^*\right\rangle \\
     &\ge \Bar{s}_T(\nu,d)+ \left\langle\Bar{\phi}_{T,\nu}(\nu^*,d)-\nabla_{\nu} D(\blambda^*,d), \nu -\nu^*\right\rangle \\
     &\ge \frac{\sigma_{\min} \underline{\cL}_f}{4}\left( \norm{ \nu'-\nu^*}_2^2 -2 H\varepsilon \norm{ \nu'-\nu^*}_2\right)
 \end{aligned}
\end{equation}
 with probability at least $1-  2m\exp (-\frac{T\sqrt{n}\log m\Bar{b}D\varepsilon^2 }{2 })$,
where we use the optimality of $\blambda^*$ and concentration of gradient. If $\nu_T^{\ast}$ is part of the optimal solution, then we claim that, the dual optimal solution $\nu_T^{\ast}$ must have $\norm{\nu^*_T-\nu^*}_2\le 2H\varepsilon$. Otherwise:
\begin{enumerate}
    \item If has $2H\varepsilon<\norm{\nu^*_T-\nu^*}_2\le 4H\varepsilon$, then there will be a ${\nu^*_T}'$ such that $\norm{{\nu^*_T}'-\nu^*}_2\ge \norm{\nu^*_T-\nu^*}_2 > 2H\varepsilon$, and 
    \begin{equation*}
    \begin{aligned}
         \Bar{D}_T(\blambda^*_T,d)-\Bar{D}_T(\blambda^*,d)\ge \frac{\sigma_{\min} \underline{\cL}_f}{4}\left( \norm{{\nu^*_T}'-\nu^*}_2^2 -2 H\varepsilon \norm{{\nu^*_T}'-\nu^*}_2\right)>0,
    \end{aligned}
\end{equation*}
which contradicts the optimality of $\blambda^*_T$.
    \item If  $\norm{\nu^*_T-\nu^*}_2> 4H\varepsilon$, since $\Bar{D}_T(\blambda^*,d)-\Bar{D}_T(\blambda^*,d)=0$ and $\Bar{D}_T(\blambda^*_T,d)-\Bar{D}_T(\blambda^*,d)\le 0$, by the convexity of $\Bar{D}_T$ we can always find a $\tilde{\blambda}$ such that  $2H\varepsilon<\norm{\tilde{\nu}-\nu^*}_2\le 4H\varepsilon$ and $\Bar{D}_T(\tilde{\blambda},d)-\Bar{D}_T(\blambda^*,d)\le 0$.
    However, according to  \eqref{eq:dual-conv-partition}, we have $$\Bar{D}_T(\tilde\blambda,d)-\Bar{D}_T(\blambda^*,d)\ge \frac{\sigma_{\min} \underline{\cL}_f}{4}\left( \norm{{\tilde{\nu}}'-\nu^*}_2^2 -2 H\varepsilon \norm{{\tilde{\nu}}'-\nu^*}_2\right)>0,$$
    which also ends up with a contradiction.
\end{enumerate}
To better present our idea to readers, we draw a figure here. This clearly shows that when our empirical function is lower bounded by a quadratic function with $H\varepsilon$ as the axis of symmetry, we essentially have $\norm{\nu^*_T-\nu^*}_2\le 2H\varepsilon$. 
\begin{figure}[!htbp]
    \centering
    \includegraphics[width = 0.7\linewidth]{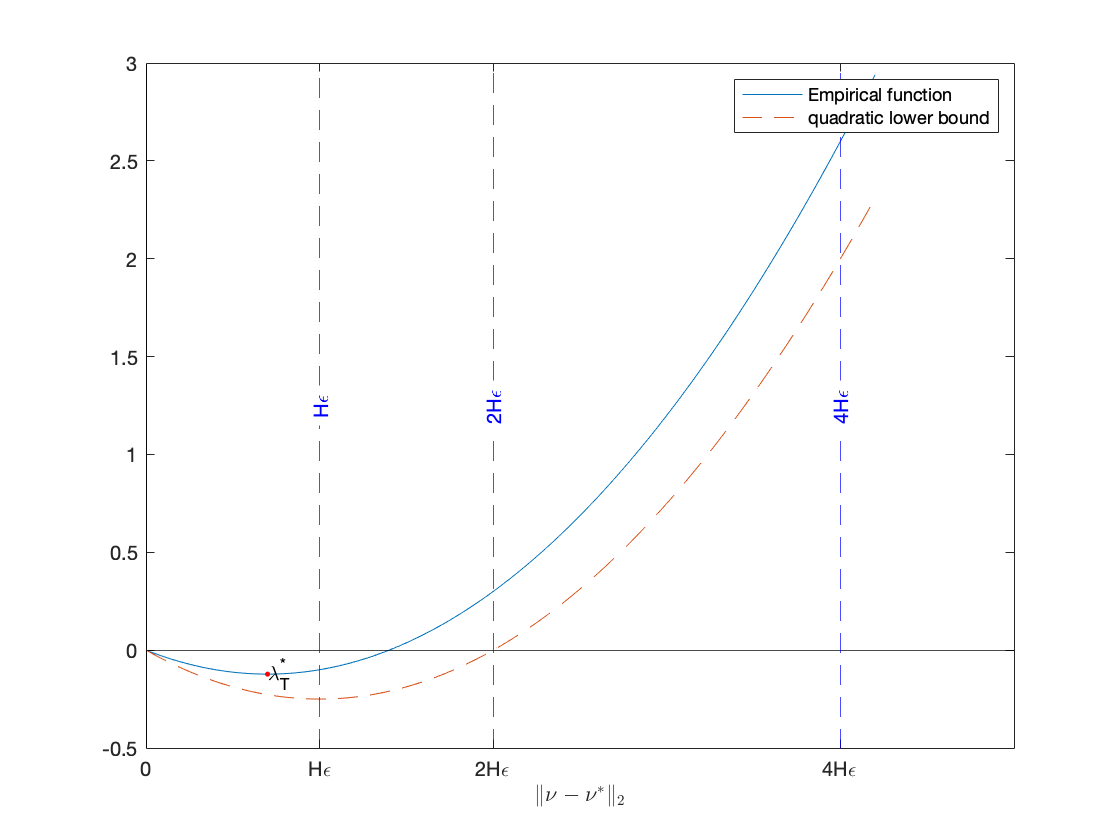}
    \caption{The value of $\Bar{D}_T$ with respect to $\norm{{{\nu}}-\nu^*}_2$. Since the quadratic lower bound has $H\varepsilon$ as the axis of symmetry, we have $ \Bar{D}_T(\blambda,d)-\Bar{D}_T(\blambda^*,d)>0$ for $\norm{\nu-\nu^*}_2> 2H\varepsilon$. }
    \label{fig:draw-quad}
\end{figure}

\subsection{Proof of Theorem \ref{thm:dual_conv}}\label{proof:dual_conv}
By the tail expectation formula, for constant $H>0$, we have
$$
\E \norm{\nu^*_{T}-\nu^*}^2_2=4H^2\int_{0}^{\infty}\PP(\norm{\nu^*_{T}-\nu^*}^2_2 > 4H^2 z)d z
$$
According to the probabilistic bound in Proposition~\ref{prop:dual_conv_prob}, for any $z>0$,
$$\PP(\norm{\nu^*_{T}-\nu^*}^2_2 > 4H^2 z)\le 2m\exp \left(-\frac{T\sqrt{n}\log m\Bar{b}D\varepsilon^2 }{2 }\right)\vee 1 + 2\exp\left({-\frac{m\log m(T\varepsilon^2-1) }{4}}\right)\vee 1.$$
Then, calculating the integral,  we get
$$
\begin{aligned}
    \E &(\norm{\nu^*_{T}-\nu^*}^2_2) = H^2\int_{0}^{\infty}\PP(\norm{\nu^*_{T}-\nu^*}^2_2\ge  4H^2 z)dz \\
 \le &\frac{400 n m\log m\Bar{b}^2D^2}{\sigma_{\min}^2 \underline{\cL}_f^2}  \int_{2/(T\sqrt{n}\Bar{b}D)}^{\infty} \left[ 2\exp \left(-\frac{T\sqrt{n}\log m\Bar{b}D z }{2 }+\log m\right)+\frac{2}{T\sqrt{n}\Bar{b}D} \right] d z   \\
& + \frac{400 n m\log m\Bar{b}^2D^2}{\sigma_{\min}^2 \underline{\cL}_f^2}  \int_{ \frac{1}{T} }^{\infty} \left[2\exp\left({-\frac{m\log m(Tz-1) }{4}}\right)+\frac{1}{T}\right] d z  \\ 
\le & C_1 \frac{\Bar{b}^2 D^2}{\sigma_{\min}^2 \underline{\cL}_f^2} \frac{nm\log m }{T}  \\
\end{aligned}
$$

For the optimality of the $O(T^{-1})$ order, let us consider a non-regularized case when $x\in[0,1]$ and $f_t(x):=f(x,\xi_t):=-(x-2\xi_t)^2/4+\xi_t^2$, with the single constraint $d=1/2$ and cost $b_t=1$. The dual problem is 
\begin{equation*}
    D_t(\lambda)=\begin{cases}
    \frac{1}{2}\lambda & \text{ if } \lambda>\xi_t \\
    -\frac{1}{4}+\xi_t-\frac{1}{2}\lambda &\text{ if }  \lambda<\xi_t-\frac{1}{2} \\
    \lambda^2-2(\xi_t-\frac{1}{4})\lambda+\xi_t^2 &\text{ if } \xi_t-\frac{1}{2}\le\lambda\le\xi_t.
    \end{cases}
\end{equation*}
Let $\xi_t$ be any distribution varies within $[1/2,\ 3/4]$ with variance $\sigma^2_{\xi}>0$. Then, for any $t$, we have $\xi_t-1/4\in [1/4,\ 1/2]\subseteq [\xi_t-1/2,\xi_t]$. Thus, for the sample average $\Bar{D}_T(\lambda):=T^{-1}\sum_{t=1}^{T}D_t(\lambda)$, when $\lambda\in[1/4,\ 1/2]$, $\Bar{D}_T(\lambda):=\lambda^2-2(\Bar{\xi_T}-1/4)\lambda+\Bar{\xi^2_T}$ with the optimal solution being $\lambda^*_T:=\Bar{\xi_T}-1/4$. We have $\E (\lambda^*_T-\lambda^*)^2\ge \operatorname{Var}(\Bar{\xi_T})=\sigma^2_{\xi}/T$. This shows that our $O(T^{-1})$ dual convergence rate is indeed optimal.

\subsection{Proof of  Theorem \ref{thm:dual_conv_approx}}

Recall that, by the proof of Proposition \ref{prop:dual_conv_prob}, the convex function $\Bar{D}_T$ is larger than a quadratic function in a neighborhood of $\nu^*$ with a high probability claimed there. Then, for any $\epsilon$ satisfying $\epsilon< 2H^2\varepsilon^2{\sigma_{\min} \underline{\cL}_f}$, with the same high probability, the $\epsilon$-optimal solution must belong to $\Omega_{\nu}(\varepsilon)$, because, for all the points in the border $\norm{\nu-\nu^*}_2=4H\varepsilon$, we already have $\Bar{D}_T(\blambda,d)-\Bar{D}_T(\blambda^*,d)\ge 2H^2\varepsilon^2{\sigma_{\min} \underline{\cL}_f}$. Then, with the same high probability, it follows that 
\begin{equation*}
    \begin{aligned}
         \epsilon\ge \Bar{D}_T(\blambda_T^{\epsilon},d)-\Bar{D}_T(\blambda^*,d)\ge \frac{\sigma_{\min} \underline{\cL}_f}{4} \norm{ \nu_T^{\epsilon\prime }-\nu^*}_2^2 -\frac{\sigma_{\min} \underline{\cL}_f}{4}\cdot 2H\varepsilon \norm{ \nu_T^{\epsilon\prime}-\nu^*}_2,
    \end{aligned}
\end{equation*}
which suggests that $\norm{\nu_T^{\epsilon}-\nu^*}_2\le\norm{ \nu_T^{\epsilon\prime }-\nu^*}_2\le H\varepsilon+\big(H^2\varepsilon^2+4\epsilon/(\sigma_{\min}\underline{\cL}_f) \big)^{1/2}$. Still, applying the tail expectation formula, we get
$$
\begin{aligned}
    \E (\norm{\nu^{\epsilon}_{T}-\nu^*}^2_2)& =4H^2\int_{0}^{\frac{2\epsilon}{H^2\sigma_{\min} \underline{\cL}_f}}\PP(\norm{\nu^{\epsilon}_{T}-\nu^*}_2\ge  2H \sqrt{z})dz \\
    &+4H^2\int_{\frac{2\epsilon}{H^2\sigma_{\min} \underline{\cL}_f}}^{\infty}\PP(\norm{\nu^{\epsilon}_{T}-\nu^*}_2\ge  2H \sqrt{z})dz \\
& \le  \frac{8\epsilon}{\sigma_{\min}\underline{\cL}_f } +4H^2\int_{\frac{\epsilon}{H^2\underline{\cL}_D}}^{\infty}\PP(\norm{\nu^{\epsilon}_{T}-\nu^*}_2\ge  2H \sqrt{z})dz.
\end{aligned}
$$
Let $2H\sqrt{z}=H\varepsilon + \sqrt{H^2\varepsilon^2+\frac{4\epsilon}{\sigma_{\min}\underline{\cL}_f}}$.  When $z>\frac{\epsilon}{H^2\sigma_{\min}\underline{\cL}_f}$, we have $\epsilon< 2H^2\varepsilon^2\sigma_{\min}\underline{\cL}_f$, thus  $\PP(\norm{\nu^{\epsilon}_{T}-\nu^*}_2\ge  2H \sqrt{z})$ can be bounded by Proposition \ref{prop:dual_conv_prob}. Also when  $2H\sqrt{z}=H\varepsilon + \sqrt{H^2\varepsilon^2+\frac{2\epsilon}{\sigma_{\min}\underline{\cL}_f}}$, we have $\varepsilon^2\ge z-\frac{2\epsilon}{H^2\sigma_{\min}\underline{\cL}_f}$. By the integral of $z$, we get the second part of the bound.

\subsection{Proof of Corollary \ref{cor:stocha_dual_conv}}

Recall the proof of Theorem \ref{thm:dual_conv_approx} that when $\varepsilon$ satisfying $\epsilon< 2H^2\varepsilon^2{\sigma_{\min} \underline{\cL}_f}$, with high probability the deterministic $\epsilon$-optimal solution must be in $\Omega_{\nu}(\varepsilon)$. Similarly, for the stochastic  $\epsilon$-optimal solution, we try to confine it in a larger region so that with high probability $ \E\left[\norm{\nu_T^{\epsilon}-\nu_T^{*}}_2^2\middle|\Bar{D}_T \right]$ can still be bounded by $\varepsilon$. Notice that, although our Proposition \ref{prop:dual_conv_prob} only focus on $\Omega_{\nu}(\varepsilon)$, it also bring us information outside $\Omega_{\nu}(\varepsilon)$. For any $\varepsilon$ and $\epsilon$,  under the event when Proposition \ref{prop:dual_conv_prob} holds, and any  $\bar D_T $ we have:

\begin{enumerate}
    \item If $\Bar{D}_T(\blambda_T^{\epsilon},d)-\Bar{D}_T(\blambda^*,d) \le 2H^2\varepsilon^2{\sigma_{\min} \underline{\cL}_f}$, then $\norm{\nu_T^{\epsilon}-\nu_T^{*}}_2\le 4H\varepsilon$.
    \item If $\Bar{D}_T(\blambda_T^{\epsilon},d)-\Bar{D}_T(\blambda^*,d) > 2H^2\varepsilon^2{\sigma_{\min} \underline{\cL}_f}$, then we have $\norm{\nu_T^{\epsilon}-\nu_T^{*}}_2\le \frac{2}{H\varepsilon\sigma_{\min}\underline{\cL}_f}(\Bar{D}_T(\blambda_T^{\epsilon},d)-\Bar{D}_T(\blambda^*,d))$. Because the convex function $\Bar{D}_T(\blambda,d)-\Bar{D}_T(\blambda^*,d)=0$ when $\blambda=\blambda^*$, and when $\norm{\nu-\nu^*}_2=4H\varepsilon$,  $\Bar{D}_T(\blambda,d)-\Bar{D}_T(\blambda^*,d)\ge 2H^2\varepsilon^2{\sigma_{\min} \underline{\cL}_f}$.
\end{enumerate}

We conclude that under the event when Proposition \ref{prop:dual_conv_prob} holds, for any  $\epsilon< 2H^2\varepsilon^2{\sigma_{\min} \underline{\cL}_f}$, 
$$ \E_{\cB}\left[\norm{\nu_T^{\epsilon}-\nu^{*}}_2^2
\middle|\Bar{D}_T \right]\le 16H^2\varepsilon^2+ \frac{4\sqrt{m\left(2\frac{\Bar{f}+\Bar{r}}{\underline{d}} +G\right) }}{H\varepsilon\sigma_{\min}\underline{\cL}_f}\cdot \epsilon
$$ 
because $\norm{\nu_T^{\epsilon}-\nu^{*}}_2\le 4\sqrt{m\left(2\frac{\Bar{f}+\Bar{r}}{\underline{d}} +G\right) }$. The RHS term has a minimum value 
$$ z_0= 3\cdot 8\epsilon^{\frac{2}{3}}\left(m\left(2\frac{\Bar{f}+\Bar{r}}{\underline{d}} +G\right)\right)^{\frac{1}{3}}/ (\sigma_{\min}\underline{\cL}_f )^{\frac{2}{3}} $$ 
when $\varepsilon_0= \epsilon^{\frac{1}{3}} \left(m\left(2\frac{\Bar{f}+\Bar{r}}{\underline{d}} +G\right)\right)^{\frac{1}{6}} / (2 H  (\sigma_{\min}\underline{\cL}_f )^{\frac{1}{3}})$. When the RHS term is larger than is minimum value, we can always take the corresponding $\varepsilon$ at the right side where $\varepsilon>\varepsilon_0$ and it follows that 
$$z= 16H^2\varepsilon^2+ \frac{4\sqrt{m\left(2\frac{\Bar{f}+\Bar{r}}{\underline{d}} +G\right) }}{H\varepsilon\sigma_{\min}\underline{\cL}_f}\cdot \epsilon \le 48H^2\varepsilon^2.$$   

Then by the tail expectation formula we have
$$
\begin{aligned}
    \E_{\cB,\cP} \norm{\nu^{\epsilon}_{T}-\nu^*}^2_2 & =\int_{0}^{z_0}\PP(\E_{\cB}\left[\norm{\nu_T^{\epsilon}-\nu^{*}}_2^2\middle|\Bar{D}_T \right]\ge z)dz  +\int_{z_0}^{\infty}\PP(\E_{\cB}\left[\norm{\nu_T^{\epsilon}-\nu^{*}}_2^2\middle|\Bar{D}_T \right]\ge   z)dz \\
& \le  z_0 +\int_{z_0}^{\infty} \left[2m\exp \left(-\frac{T\sqrt{n}\log m\Bar{b}D z/(48H^2) }{2 }\right)\vee 1\right]dz \\
& +\int_{z_0}^{\infty} \left[2\exp\left({-\frac{m\log m(Tz/(48H^2)-1) }{4}}\right)\vee 1\right]dz. \\
& \le   z_0 + C_2 \frac{\Bar{b}^2 D^2}{\sigma_{\min}^2 \underline{\cL}_f^2} \frac{nm\log m }{T} .
\end{aligned}
$$

\subsection{Proof of Lemma \ref{lemma:offline_ub}}
Recall the Lagrangian of program \eqref{eq:lagrangian}. By duality, we have

\begin{equation*}
    \begin{aligned}
        R^*(\mathcal{P}) &:= \E_{\mathcal{P}} \left[\max_{x_t\in \cX } \sum_{t=1}^{T} f_t(x_t) +T\cdot r(\frac{\sum_{t=1}^T b_t x_t}{T}), \ s.t. \ \sum_{t=1}^T b_t x_t \preceq dT\right]\\
        &\le \E \sum_{t=1}^{T}\left[ f_t (\tilde x_t(\nu^*)) + r(\tilde a (\mu^*)) + (\tilde a(\mu^*)-b_t \tilde x_t(\nu^*))^{\top}\mu^* +(d-b_t \tilde x_t(\nu^*))^{\top}\lambda^*\right] \\
        & = T\cdot g(\nu^*)
    \end{aligned}
\end{equation*}

\subsection{Proof of Proposition \ref{prop:regret_decomp}}

Since $r$ is proper, by Fenchel conjugate, the definition of $\hat{\mu}_T$ implies 
$$
\begin{aligned}
    r(\frac{\sum_{t=1}^{T}b_t x_t}{T})+\hat{\mu}_T^{\top}\frac{\sum_{t=1}^{T}b_t x_t}{T} = & r^*(-\hat{\mu}_T) -\hat\mu_T^\top \E b_t \tilde{x}_t(\nu^*)+ \hat\mu_T^\top \E b_t \tilde{x}_t(\nu^*) \\
    \ge & r(\tilde{a}(\mu^*) )+\hat{\mu}_T^\top \tilde{a}(\mu^*) \\
\end{aligned}
$$

Combined with $R\left(A\middle| \mathcal{P}\right)=\E_{A,\mathcal{P}} \left[ \sum_{t=1}^{T} f_t(x_t) +T\cdot r(\frac{\sum_{t=1}^T b_t x_t}{T}) \right]$, we have

\begin{equation*}
    \begin{aligned}
        R\left(A\middle| \mathcal{P}\right) & \ge \E\left[\sum_{t=1}^{T}f_t(x_t)+Tr(\tilde{a}(\mu^*))+T\hat{\mu}_T^{\top}\tilde{a}(\mu^*)-T\hat{\mu}_T^{\top}\frac{\sum_{t=1}^{T}b_t x_t}{T}\right]. \\
    \end{aligned}
\end{equation*}

The Assumption \ref{asm:regularizer} suggests that
$$
\norm{\hat{\mu}_T}_2 \le \sqrt{m}G , \text{ and } \norm{a}_2=\norm{\nabla r^*(-\mu)}_2\le \sqrt{n}D\Bar{b}.
$$
Thus 
\begin{equation*}
    \begin{aligned}
        R\left(A\middle| \mathcal{P}\right)  \ge  \E & \left[  \sum_{t=1}^{T} \left[ f_t(\tilde x_t(\nu_{t-1}))+r(\tilde{a}(\mu^*))\right]  + \left\langle \hat{\mu}_T,\sum_{t=1}^{T}(\tilde{a}(\mu^*)-b_t x_t) \right\rangle\right] \\
        = &  \E \left[ \sum_{t=1}^{T}g(\nu_{t-1}) + \left\langle \hat{\mu}_T-\mu^*,\sum_{t=1}^{T}(\tilde{a}(\mu^*)-b_t x_t) \right\rangle - \left\langle \lambda^* , \sum_{t=1}^{T}(d- b_t x_t)
        \right\rangle  \right] \\
    \end{aligned}
\end{equation*}
Combined with \eqref{eq:upp_1}, we can show that
\begin{equation*}
    \begin{aligned}
        R^*(\mathcal{P}) -  R\left(A\middle| \mathcal{P}\right)  \le \E \left[ \sum_{t=1}^{T}g(\nu_{}^*)-g(\nu_{t-1})  +\left\langle \lambda^* , \sum_{t=1}^{T}(d- b_t x_t)
        \right\rangle  \right],
    \end{aligned}
\end{equation*}
or, equivalent, in a two-phase form:
\begin{equation*}
    \begin{aligned}
        R^*(\mathcal{P}) -  R\left(A\middle| \mathcal{P}\right)  \le \E \left[ \sum_{t=1}^{\tau}g(\nu_{}^*)-g(\nu_{t-1})  +\left\langle \lambda^* , \sum_{t=1}^{\tau}(d- b_t x_t)
        \right\rangle  \right]+ 2(\Bar{f}+\Bar{r}+\sqrt{mn}GD\Bar{b})(T-\tau).
    \end{aligned}
\end{equation*}
We conclude the proof.

\subsection{Proof of Lemma \ref{lemma:region_d}}

For technical convenience, we assume that, for each non-binding dimensions $i\in I_{\textsf{NB}}$, the updated constraint $d_{it}$ never exceeds the threshold $\bar{d}$ (the uniform bound defined in Assumption \ref{asm:basic}) at all iterations. This is a mid assumption both for theory and in practice. Indeed, if $d_{it}$ is larger than the $\bar{d}$, this means that the constraint $d_{it}$ is very loose so that its impact to the optimization problem is negligible. In this case, such a constraint can essentially be discarded. We start with a lemma on the continuity of dual optimal solution to prove Lemma \ref{lemma:region_d}.

\begin{lemma}[Continuity of dual optimal solution]
\label{lemma:continuity_lambda}
Under Assumption \ref{asm:basic}, \ref{asm:regularizer}, \ref{asm:nondegeneracy}, for the stochasitc program $\min\limits_{\mu,\lambda\succeq 0} D(\blambda,d')=\mathbb{E} f_t^{*}(b_t^{\top}(\mu+\lambda)) +r^{*}(-\mu)+{d'}^{\top}\lambda$, let $d'$ be $d'_1, d'_2 \in \Omega_d$ separately,  then the corresponding optimal solution $\nu^*(d'_1),\nu^*(d'_2)$ satisfies 
$$
\norm{\nu^*(d'_1)-\nu^*(d'_2)}^2_2 \le \frac{1}{\sigma_{\min}^2 \underline{\cL}_f^2} \norm{d'_1-d'_2}^2_2.
$$
If further $d'_1, d'_2$ identify the same binding/non-binding dimensions, then 
$$
\norm{\nu^*(d'_1)-\nu^*(d'_2)}^2_2 \le  \frac{1}{\sigma_{\min}^2 \underline{\cL}_f^2} \sum_{i\in I_{\textsf{B}} }(d'_{1i}-d'_{2i})^2,
$$
where the binding dimension $I_{\textsf{B}}$ is with respect to $d_1'$ and $d'_2$.

\end{lemma}

\proof{Proof of Lemma \ref{lemma:continuity_lambda}}
 By Proposition \ref{prop:quadratic_growth} and the uniform assumption on $d$, we have
$$
\begin{aligned}
 D(\nu^*(d_2'),\mu^*(d_1'), d_1')-D(\blambda^*(d_1'),d_1') &\ge \frac{1}{2}\sigma_{\min}\underline{\cL}_f \norm{\nu^*(d_2')-\nu^*(d_1')}_2^2 \\
 D(\nu^*(d_1'),\mu^*(d_2'),d_2')-D(\blambda^*(d_2'),d_2') &\ge \frac{1}{2}\sigma_{\min}\underline{\cL}_f \norm{\nu^*(d_1')-\nu^*(d_2')}_2^2.
\end{aligned}
$$

Summing up two inequality we have
\begin{equation}\label{eq:dual_opt_continuity_1}
    \begin{aligned}
    (d_1'-d_2')^\top (\nu^*(d_2')- \nu^*(d_1')) \ge  \sigma_{\min}\underline{\cL}_f \norm{\nu^*(d_2')-\nu^*(d_1')}_2^2,
\end{aligned}
\end{equation}
or equivalently, $\sum_{i\in I_{\textsf{B}} }(d'_{1i}-d'_{2i})  (\nu^*_i(d_2')- \nu^*_i(d_1')) \ge  \sigma_{\min}\underline{\cL}_f \norm{\nu^*(d_2')-\nu^*(d_1')}_2^2$ if further $d'_1, d'_2$ share the same binding/non-binding dimensions. From \eqref{eq:dual_opt_continuity_1} we can show that 
$$
\begin{aligned}
  \sigma_{\min}\underline{\cL}_f \norm{\nu^*(d_2')-\nu^*(d_1')}_2^2  & \le (d_1'-d_2')^\top (\nu^*(d_2')- \nu^*(d_1')) \le \norm{d_1'-d_2'}_2 \norm{\nu^*(d_2')- \nu^*(d_1')}_2 \\
\norm{\nu^*(d_2')-\nu^*(d_1')}_2 & \le \frac{1}{\sigma_{\min}\underline{\cL}_f } \norm{d_1'-d_2'}_2.
\end{aligned}
$$
Thus we get the first statement. For the second statement, we focus on the binding dimensions:
$$
\begin{aligned}
\sigma_{\min}\underline{\cL}_f \norm{\nu^*(d_2')-\nu^*(d_1')}_2^2 & \le \sum_{i\in I_{\textsf{B}} }(d'_{1i}-d'_{2i})  (\nu^*_i(d_2')- \nu^*_i(d_1')) \\ 
& \le \sqrt{\sum_{i\in I_{\textsf{B}} }(d'_{1i}-d'_{2i})^2  } \sqrt{\sum_{i\in I_{\textsf{B}} }(\nu^*_i(d_2')- \nu^*_i(d_1'))^2 } \\
& \le  \sqrt{\sum_{i\in I_{\textsf{B}} }(d'_{1i}-d'_{2i})^2  } \norm{\nu^*(d_2')-\nu^*(d_1')}_2,
\end{aligned}
$$
which completes the proof of Lemma \ref{lemma:continuity_lambda}.
\endproof
Then, we return to Lemma \ref{lemma:region_d} and consider the original constraints $d$ and the its binding/non-binding dimensions:  $I_{\textsf{B}}= \left\{ i\middle| d_i-\E\left(b_t\tilde{x}_t(\nu^*)\right)_i=0\right\}$, and $I_{\textsf{NB}}= \left\{ i\middle| d_i-\E\left(b_t\tilde{x}_t(\nu^*)\right)_i>0\right\}$. Here we write the corresponding optimal solution to $\min\limits_{\mu,\lambda\succeq 0} D(\blambda,d) $ as $\blambda^*$, and write $\blambda^*(d')$ if we change $d$ to $d'$. Then if $i\in I_{\textsf{B}}$  and $i$ changes to non-binding dimensions for  $d'$, by Lemma \ref{lemma:continuity_lambda} and Assumption \ref{asm:nondegeneracy:3-2}, for any $\norm{d'-d}_2\le \delta_0\wedge \frac{ \underline{\lambda} }{2 \Bar{\cL}_r }  $ we have

\begin{equation}\label{eq:nb_to_b}
\norm{d-d'}_2 \ge \sigma_{\min}\underline{\cL}_f \norm{\nu^*(d')-\nu^*}_2 \ge \sigma_{\min}\underline{\cL}_f\left(\abs{\lambda_i^*-0}-\abs{\mu_i-\mu_i(d') } \right)\ge \frac{1}{2}\sigma_{\min}\underline{\cL}_f \underline{\lambda},
\end{equation}
where $\underline{\lambda}= \min\left\{\lambda^*_i\middle|  i \in I_{\textsf{B}}  \right\}$. If on the other hand, $i\in I_{\textsf{NB}}$  and $i$ changes to binding dimensions for  $d'$, by Assumption \ref{asm:k-thmoment}, we have

\begin{equation*}
    \begin{aligned}
    \E \norm{\nu^*(d') - \nu^*}_2  & \ge  \frac{1}{2\Bar{b}^2 L_1 } \left| \E\left(b_t\tilde{x}_t(\nu^*(d'))\right)_i -\E\left(b_t\tilde{x}_t(\nu^*)\right)_i \right|=  \frac{1}{2\Bar{b}^2 L_1 }  \left| d'_i -\E\left(b_t\tilde{x}_t(\nu^*)\right)_i \right|  \\
    & \ge \frac{1}{2\Bar{b}^2 L_1 } \left( \left| d_i-\E\left(b_t\tilde{x}_t(\nu^*)\right)_i \right| - \abs{ d'_i -d_i} \right).
    \end{aligned}
\end{equation*}

Denote the minimum of remaining resources in non-binding dimensions by $$\gamma =\min_{i\in I_{\textsf{NB}}} \left\{ d_i-\E\left(b_t\tilde{x}_t(\nu^*)\right)_i \right\}.$$ By Lemma \ref{lemma:continuity_lambda} we have

\begin{equation*}
    \begin{aligned}
     \norm{d-d'}_2 \ge \sigma_{\min}\underline{\cL}_f \E \norm{\nu^*(d') - \nu^*}_2  & \ge \frac{\sigma_{\min}\underline{\cL}_f}{2 \Bar{b}^2 L_1 } \left( \gamma - \abs{ d'_i -d_i} \right)\ge \frac{\sigma_{\min}\underline{\cL}_f}{2 \Bar{b}^2 L_1 } \left( \gamma - \norm{d-d'}_2 \right),
    \end{aligned}
\end{equation*}
i.e., $\norm{d-d'}_2 \ge \frac{\gamma \sigma_{\min}\underline{\cL}_f}{ \sigma_{\min}\underline{\cL}_f+2\Bar{b}^2 L_1 }$. Combined with \eqref{eq:nb_to_b}, taking $$\delta_d= \frac{1}{\sqrt{m}} \cdot \left( \frac{\gamma \sigma_{\min}\underline{\cL}_f}{ \sigma_{\min}\underline{\cL}_f+2\Bar{b}^2 L_1 }\right) \wedge  \left( \frac{1}{2}\sigma_{\min}\underline{\cL}_f\underline{\lambda}\right)\wedge\delta_0\wedge \frac{ \underline{\lambda} }{2 \Bar{\cL}_r } ,$$ 
we can conclude that when $\abs{d_i-d'_i}\le \delta_d$, the binding/non-binding dimensions will never change. Moreover, enlarging the constraint in a non-binding dimension  will never change this constraint to the binding dimension. So, for the non-binding dimensions, $d_i'-d_i$ can be any large. This finishes the proof.

\subsection{Proof of lemma \ref{lemma:alg_conv}}
Proof of this lemma under frequent resolving is similar in spirit as that in \cite{li2021online}, but with different induction and dual convergence rate. Here we focus on the infrequent re-solving scheme in Algorithm \ref{alg:infrequent}. All the proof for the infrequent re-solving scheme is valid for the frequent resolving case. we define $d_t=d_{T_j}$ if $T_j\le t<T_{j+1}$. Then $d_t$ will only update when $t=T_j$ for some $j$. Without loss of generality, we assume $\frac{1}{\rho}$ is a integer and $T=\left(\frac{1}{\rho}\right)^K$ for some integer $K$. 
Denote the ratio $C_\rho=\rho/(1-\rho)$. 
Also assume $C_4=O(mn\log m)$ in Condition \ref{condition:conv}. Decomposing the perturbation of dual variables will lead to
\begin{equation*}
    \begin{aligned}
    \E\left[\sum_{t=1}^{\tau } \norm{\nu_{t-1}-\nu^*}^2_2\right]\le 2 \E\left[\sum_{t=1}^{\tau } \norm{\nu_{t-1}-\nu^*(d_{t-1})}^2_2+\norm{\nu^*(d_{t-1})-\nu^*}^2_2\right]
    \end{aligned}
\end{equation*}

By the definition of stopping time $\tau$ and Condition \ref{condition:conv}, the first term in the RHS has
\begin{equation*}
    \E\left[\sum_{t=1}^{\tau } \norm{\nu_{t-1}-\nu^*(d_{t-1})}^2_2\right] \le  \sum_{j=1}^{J} 2C_4\frac{ 1 }{T_j}\cdot(T_{j+1}-T_j ) \text{ or } C_4\frac{ 1 }{T-T_j }\cdot(T_{j+1}-T_j ) \\
    \le  2C_4 (\log T+1).
\end{equation*}
Notice that, we need a good initialization: $\E\norm{\nu_0-\nu^*}^2=O(1/T)$ to reach this condition for the first $(1-\rho)T$ terms in infrequent resolving case. For the second term, we apply lemma \ref{lemma:continuity_lambda} to it.
\begin{equation*}
    \begin{aligned}
    2 \E\left[\sum_{t=1}^{\tau}\norm{\nu^*(d_{t-1})-\nu^*}^2_2\right] \le  \frac{2}{\sigma_{\min}^2\underline{\cL}_f^2}  \E\left[\sum_{t=1}^{\tau} \sum_{i\in I_{\textsf{B}} }(d_{it}-d_{i})^2\right].
    \end{aligned}
\end{equation*}
Thus we transform the perturbation of $\nu^*(d_t)$ into the deviation of $d_t$ in the binding dimensions.

To ease our analysis, we define a new sequence $d_{t}'$
$$
d_{t}'=\begin{cases}d_t, &\text { if } t\le \tau \\ d_{t-1}, &\text { if } t> \tau\end{cases}
$$
which shares the same stopping time with $d_{t}$ and define $\tau_i=\min\limits_{t} 
\{T-\left\lceil \frac{m\bar{b}}{\underline{d}} \right\rceil\}\cup\{t|d_{it}'\notin \mathcal{D}_i\}$ for $i\in [m]$ as the stopping time on each dimension with $\tau = \min\{\tau_1,...,\tau_m\}$.  We follow a similar approach as \cite{li2021online} to bound the stopping time. The distinction is that our analysis is more refined and we use a martingale argument that is different from \cite{li2021online} to study the infrequent resolving scheme.

We first consider the binding dimensions. For any $i\in I_{\textsf{B}}$, we  derive:

\begin{equation}\label{eq:dt-induction}
    \begin{aligned}
     d_{i,T_{j+1} }'  & = d_{i, T_{j}}' + \frac{\sum_{k=T_{j}+1}^{T_{j+1}}\left[d_{i,T_{j}}' - \left(b_k\tilde{x}_{k}(\nu_{T_{j}} )\right)_i\right]  }{T-T_{j+1}}\mathbb{I}(\tau> T_{j}) \\
     \E \left(d_{i,T_{j+1}}' - d_{i}\right)^2 & =   \E \left(d_{i,T_{j}}' - d_{i}\right)^2 + \underbrace{\E  \frac{\left(\sum_{k=T_{j}+1}^{T_{j+1}}\left[d_{i,T_{j}}' - \left(b_k\tilde{x}_{k}(\nu_{T_{j}} )\right)_i\right] \right)^2 }{(T-T_{j+1})^2} \mathbb{I}(\tau> T_{j})  }_{A'}  \\ 
    & + \underbrace{2 \E\frac{\left(d_{i,T_{j}}' - d_{i}\right) \left(\sum_{k=T_{j}+1}^{T_{j+1}}\left[d_{i,T_{j}}' - \left(b_k\tilde{x}_{k}(\nu^*(d_{T_{j}}) )\right)_i\right]  \right) }{T-T_{j+1}}\mathbb{I}(\tau> T_j )}_{B'}   \\
    & + \underbrace{ 2 \E\frac{\left(d_{i,T_{j}}' - d_{i}\right) \left(\sum_{k=T_{j}+1}^{T_{j+1}}\left(b_k\tilde{x}_{k}(\nu_{T_{j}} )\right)_i-\left(b_k\tilde{x}_{k}(\nu^*(d_{T_{j}}) )\right)_i \right) }{T-T_{j+1}} \mathbb{I}(\tau> T_j)}_{C'}
    \end{aligned}
\end{equation}

For the term $A'$ we have $$
\begin{aligned}
   A' & = \E  \frac{\left(\sum_{k=T_{j}+1}^{T_{j+1}}\left[d_{i,T_{j}}' -\E\left[\left(b_k\tilde{x}_{k}(\nu_{T_{j}} )\right)_i|\cH_{T_j}\right] +\E\left[\left(b_k\tilde{x}_{k}(\nu_{T_{j}} )\right)_i|\cH_{T_j}\right]- \left(b_k\tilde{x}_{k}(\nu_{T_{j}} )\right)_i\right] \right)^2 }{(T-T_{j+1})^2} \mathbb{I}(\tau> T_{j})\\
   & \le 2\E\frac{\left(\sum_{k=T_{j}+1}^{T_{j+1}} d_{i,T_{j}}' -\E\left[\left(b_k\tilde{x}_{k}(\nu_{T_{j}} )\right)_i|\cH_{T_j}\right] \right)^2 }{(T-T_{j+1})^2}  \\
   & \ +  2\E\frac{   \left(\sum_{k=T_{j}+1}^{T_{j+1}}\E\left[\left(b_k\tilde{x}_{k}(\nu_{T_{j}} )\right)_i|\cH_{T_j}\right]- \left(b_k\tilde{x}_{k}(\nu_{T_{j}} )\right)_i\right)^2 }{(T-T_{j+1})^2}\\
   & \le 2\E\frac{\sum_{k=T_{j}+1}^{T_{j+1}}\left(\E\left[ \left(b_k\tilde{x}_{k}(\nu^*(d_{T_{j}}) )\right)_i-\left(b_k\tilde{x}_{k}(\nu_{T_{j}} )\right)_i \middle| \cH_{T_j} \right]\right)^2 }{T-T_{j+1} } \\ 
   & \le 2 \Bar{b}^4 L_1^2\E \norm{\nu^*(d_{T_{j}})-\nu_{T_{j}} }_2^2 \\ 
   & \le \frac{C_4 \Bar{b}^4 L_1^2+ 4nD^2\Bar{b}^2 }{T-T_j}+ \frac{C_4 \Bar{b}^4 L_1^2}{T_j}
\end{aligned}
$$ 
For the second inequality, since $i\in I_{\textsf{B}} $ and $d_t\in \sigma(\cH_t)$, conditioned on past history $\cH_{T_j}$, we always have $$d_{i,T_{j}}' -\E\left[ \left(b_k\tilde{x}_{k}(\nu^*(d_{T_{j}}) )\right)_i \middle| \cH_{T_j} \right]=0, \text{ for any } k\ge T_j+1.$$ This also indicates that $B'=0$. For the third inequality, we use Assumption \ref{asm:k-thmoment}. For the term $C'$, we apply Assumption \ref{asm:k-thmoment} and Condition \ref{condition:conv}:
\begin{equation*}
    \begin{aligned}
   C' & = 2 \E\left[\E\frac{\left(d_{i,T_{j}}' - d_{i}\right) \left(\sum_{k=T_{j}+1}^{T_{j+1}}\left(b_k\tilde{x}_{k}(\nu_{T_{j}} )\right)_i-\left(b_k\tilde{x}_{k}(\nu^*(d_{T_{j}}) )\right)_i \right) }{T-T_{j+1}} \mathbb{I}(\tau> T_j)\middle| \cH_{T_j} \right]\\
   & \le 2 \Bar{b}^2 L_1\E \frac{\abs{d_{i,T_{j}}' - d_{i}} \sum_{k=T_{j}+1}^{T_{j+1}}\norm{\nu_{T_{j}}-\nu^*(d_{T_{j}}) } }{T-T_{j+1} }\\
   &  \le {2 \sqrt{C_4} \Bar{b}^2 L_1\ \sqrt{\mathbb{E}\left(d_{i ,T_{j}}^{\prime}-d_{i}\right)^{2}} \sqrt{\frac{1}{T-T_j} + \frac{1}{T_j}}  }.
    \end{aligned}
\end{equation*}
Here the first inequality is because of Assumption \ref{asm:k-thmoment}, and the second inequality is from Condition \ref{condition:conv} and Cauchy inequality. Here in the derivation, we can treat $\{\nu_t\}$ as a new sequence generated by $\{d_{t}'\}$, which has the same value with the original one when $t\le \tau$, and takes $\nu_{t}=\mathcal{B}_{t}(\cH_t,d_{t}')$ when $t> \tau$. We then get the recurrence relation of $d_{i,T_{j} }^{\prime}-d_{i}$:
\begin{equation*}
    \begin{aligned}
    \mathbb{E}\left(d_{i, T_{j+1}}^{\prime}-d_{i}\right)^{2} \leq & \mathbb{E}\left(d_{i ,T_{j}}^{\prime}-d_{i}\right)^{2}
     +\frac{C_4 \Bar{b}^4 L_1^2+ 4nD^2\Bar{b}^2 }{T-T_j}+ \frac{C_4 \Bar{b}^4 L_1^2}{T_j}\\
     &+{2 \sqrt{C_4 }\Bar{b}^2 L_1\ \sqrt{\mathbb{E}\left(d_{i t}^{\prime}-d_{i}\right)^{2}} \sqrt{\frac{1}{T-T_j} + \frac{1}{T_j}}  }.
    \end{aligned}
\end{equation*}

Since $d_0=d$, assigning $C_4=O(mn\log m)$ and by induction we have $$\mathbb{E}\left(d_{i ,T_{j}}^{\prime}-d_{i}\right)^{2}\le C_5 C_\rho mn \log m D^2 \Bar{b}^4 L_1^2 \frac{\rho^{-j}-1}{T} $$


So, we have
$$
\begin{aligned}
2 \E\left[\sum_{t=1}^{\tau}\norm{\nu^*(d_{t-1})-\nu^*}^2_2\right] 
\le  & 2 \E\left[\sum_{t=1}^{\tau}\sum_{i\in I_{\textsf{B}}} (d_{i,{t-1}}-d_i)^2\right] \\
\le & 2m \mathbb{E}\sum_{j=1}^{J} (T_j-T_{j-1}) \left[\left(d_{i, T_j}^{\prime}-d_{i}\right)^{2}\right] \le C_5 C_\rho m^2n \log m D^2 \Bar{b}^4 L_1^2 \log T +C, \\
 \text{ and } \E\left[\sum_{t=1}^{\tau} \norm{\nu_{t-1}-\nu^*}^2_2\right]\le  &  \left( 2C_2 +\frac{mC_3}{\underline{\cL}_D ^2}  \right) \log T+  2C_2,
\end{aligned}
$$
Notice that, by Condition 1, $C_4$ can take as small as $O(mn\log m/(\sigma_{\min}^2\underline{\cL}_f^2 ) )$, which completes the proof that
\begin{equation*}
    \E\left[\sum_{t=1}^{\tau}\norm{\nu_t-\nu^*}^2_2\right]  \le C\cdot C_\rho \frac{ m^2n\log m D^2 \Bar{b}^4 L_1^2 \log T}{\sigma_{\min}^2\underline{\cL}_f^2} +C
\end{equation*}
In the following discussion, we will treat $C_\rho$ as a constant and thus can be omitted. In the case of frequent resolving, we only need to substitute the recurrence relation of $d_{it}'-d_i$ with:
\begin{equation*}
    \begin{aligned}
    \mathbb{E}\left(d_{i, t+1}^{\prime}-d_{i}\right)^{2} \leq & \mathbb{E}\left(d_{i t}^{\prime}-d_{i}\right)^{2}
     +\frac{\left(\Bar{d}+\sqrt{n}D\Bar{b}\right)^2 }{(T-t-1)^{2}}  +\frac{2\sqrt{2C_2L_2 }\Bar{b}^2 \sqrt{ \frac{1}{t+1}+\frac{1}{T-t} } \sqrt{\mathbb{E}\left(d_{i t}^{\prime}-d_{i}\right)^{2}}   }{T-t-1},
    \end{aligned}
\end{equation*}
which can be derived by the same analysis as the infrequent resolving case.
Since $d_0=d$, by induction we have $\mathbb{E}\left(d_{it}^{\prime}-d_{i}\right)^{2} \le C_3\frac{t+1 }{(T+1)(T-t) } $, where $C_3= \left(2\cdot \left(\Bar{d}+\sqrt{n}D\Bar{b}\right)^2 \vee \left(2\sqrt{2C_2L_2 }\Bar{b}^2\right)+1  \right)^2$. The lemma still holds.

\subsection{Proof of lemma \ref{lemma:early_stop}}
We prove this Lemma under an infrequent resolving setting. The frequent resolving can be handled analogously. Since $\tau = \min\{\tau_1,...,\tau_m\}$, we only need to show $\E (T-\tau_i) \le C \log T $ for any $i$ in binding dimensions and non-binding dimensions. By the definition of $\rho$, and $\tau$, the algorithm will only hit the stopping time after the first update of $d_t$, i.e, $\tau> T_1$. For the binding dimensions, applying Chebyshev's inequality, we have
\begin{equation}
\label{eq:rem_1}
    \begin{aligned}
\E (T-\tau_i) \le& \sum_{i=1}^{T} \PP\left( \tau_i \le t\right) \le 1+\frac{\sqrt{n}D\bar{b}}{\underline{d}}+\sum_{j=2}^{J} \PP\left( \tau_i \le T_j\right)(T_{j}-T_{j-1}) \\
\le & 1+ \frac{\sqrt{n}D\bar{b}}{\underline{d}} +\sum_{j=2}^{J}\PP\left( \abs{d_{i,T_j}-d_i}\ge \delta_d\right)(T_{j}-T_{j-1})  \\
\le & 1+ \frac{\sqrt{n}D\bar{b}}{\underline{d}} +\sum_{j=2}^{J}\left( \frac{\E \left(d_{i, T_{j+1}}^{\prime}-d_{i}\right)^{2}  }{\delta_d^2}\right)(\rho^{j-1}(1-\rho)T ) \\ 
 \le & C+ \frac{\sqrt{n}D\bar{b}}{\underline{d}} + \frac{C_5 m^2n D^2 \Bar{b}^4 L_1^2}{\delta_d^2} \log T 
\end{aligned}
\end{equation}

For the non-binding dimensions, $\mathcal{D}$ ensures that binding/non-binding dimensions remain unchanged when $d'\in \mathcal{D}$. Then for $d'\in \mathcal{D}$, we define
$$
\tilde{d}'_i= \begin{cases}d_i', &\text{ if }i\in I_{\textsf{B}} \\ d_i-\delta_d, &\text{ if }i\in I_{\textsf{NB}}\end{cases}
$$
We know that $\nu^*(d')=\nu^*(\tilde{d}')$ because the non-binding constraints are loose, then 
$$
\E \left(b_t \tilde{x}_t(\nu^*(d'))\right)_i  = \E \left(b_t \tilde{x}_t(\nu^*(\tilde{d}'))\right)_i   < d_i-\delta_d
$$
Recall that $\tilde{x}_k(\cdot)\indep \cH_{t}$, thus $\E \left[\left(b_k \tilde{x}_k(\nu^*(d_{t})\right)_i\middle| \cH_{t} \right]  < d_i-\delta_d$ for any $k\ge t+1$, $i\in I_{\textsf{NB}}$ and $d_{t}\in\mathcal{D}$. This implies that 

$$
\begin{aligned}
\PP\left(\tau_{i} \leq T_j\right) 
&\le  \PP\left(\sum_{t=1}^{t^{\prime}} (b_t \tilde{x}_t\left(\nu_{t-1} \right))_i \geq t'\left(d_{i}-\delta_{d}\right)+T \delta_{d} \text { for some } 1 \leq t^{\prime} \leq T_j\right)   \\
& \le  \PP\left(\sum_{t=1}^{t^{\prime}} \left[(b_t \tilde{x}_t\left(\nu_{t-1} \right))_i - \E \left[\left(b_t \tilde{x}_t(\nu^*(d_{t-1})\right)_i\middle| \cH_{t-1} \right]\right] \geq T \delta_{d} \text { for some } 1 \leq t^{\prime} \leq T_j\right) \\
& \le  \PP\left(\sum_{t=1}^{t^{\prime}} \left[(b_t \tilde{x}_t\left(\nu_{t-1} \right))_i -\E\left[(b_t \tilde{x}_t\left(\nu_{t-1} \right))_i|\cH_{t-1}\right]\right] + \right.\\ 
& \left. \sum_{t=1}^{t^{\prime}} \abs{\E \left[\left(b_t \tilde{x}_t(\nu^*(d_{t-1})\right)_i\middle| \cH_{t-1} \right]-\E\left[(b_t \tilde{x}_t\left(\nu_{t-1} \right))_i|\cH_{t-1}\right] }  \geq T \delta_{d} \text { for some } 1 \leq t^{\prime} \leq T_j \right) \\
& \le \PP \left( \sum_{t=1}^{t^{\prime}} \left[(b_t \tilde{x}_t\left(\nu_{t-1} \right))_i -\E\left[(b_t \tilde{x}_t\left(\nu_{t-1} \right))_i|\cH_{t-1}\right]\right] \ge \frac{T\delta_d}{2} \text { for some } 1 \leq t^{\prime} \leq T_j\right) \\
& +\PP \left(  \sum_{t=1}^{t^{\prime}} \abs{\E \left[\left(b_t \tilde{x}_t(\nu^*(d_{t-1})\right)_i\middle| \cH_{t-1} \right]-\E\left[(b_t \tilde{x}_t\left(\nu_{t-1} \right))_i|\cH_{t-1}\right] }    \ge \frac{T\delta_d}{2} \text { for some } 1 \leq t^{\prime} \leq T_j\right)\\
\end{aligned}
$$
Since sequences in the last two lines are martingales/sub-martingales, we use Doob's martingale inequality and get the following derivation:
\begin{equation}\label{eq:non-bind-stop}
    \begin{aligned}
\PP\left(\tau_{i} \leq T_j\right)  \le & \frac{4}{T^2\delta_d^2}\sum_{t=1}^{T_j}\E\left[ (b_t \tilde{x}_t\left(\nu_{t-1} \right))_i -\E\left[(b_t \tilde{x}_t\left(\nu_{t-1} \right))_i|\cH_{t-1}\right]\right]^2  \\
& + \frac{4}{T^2\delta_d^2}  \E\sum_{t=1}^{T_j } \left[ \abs{\E \left[\left(b_t \tilde{x}_t(\nu^*(d_{t-1})\right)_i\middle| \cH_{t-1} \right]-\E\left[(b_t \tilde{x}_t\left(\nu_{t-1} \right))_i|\cH_{t-1}\right] }  \right]^2\\
    \end{aligned}
\end{equation}


By Assumption \ref{asm:k-thmoment}, we have
\begin{equation*}
    \begin{aligned}
      \PP\left(\tau_{i} \leq T_j \right) 
\le &\frac{16n\Bar{b}^2D^2 T_j }{T^2\delta_d^2}  + \frac{8L_1^2\Bar{b}^4  }{T^2\delta_d^2} \sum_{t=1}^{T_j} \E\norm{\nu_{t-1}-\nu^*(d_{t-1})}_2^2 \\
\le &\frac{16n\Bar{b}^2D^2 T_j}{T^2\delta_d^2} + \frac{C_5 m^2n\log m D^2 \Bar{b}^4 L_1^2 \log T  }{T^2\delta_d^2},
    \end{aligned}
\end{equation*}
where for $j=1$, we use the assumption on initialization. We now go back to calculate the $\E (T-\tau_i)$:

\begin{equation}\label{eq:rem_2}
    \begin{aligned}
  \E (T-\tau_i) \le&2+ \frac{\sqrt{n}D\bar{b}}{\underline{d}} + \sum_{j=2}^{J} \PP\left( \tau_i \le T_j\right)(T_{j}-T_{j-1})\\
  & \le C+ \frac{\sqrt{n}D\bar{b}}{\underline{d}}+ \sum_{j=2}^{J}\frac{16n\Bar{b}^2D^2 T_j(T_j-T_{j-1}) }{T^2\delta_d^2} + \frac{C_5 (T_j-T_{j-1})}{T^2\delta_d^2}  ( m^2n \log m D^2 \Bar{b}^4 L_1^2 \log T) \\
& \le C+ \frac{\sqrt{n}D\bar{b}}{\underline{d}} + \frac{C_5n\Bar{b}^2D^2\log T}{\delta_d^2}+\frac{C_5 m^2n\log m\Bar{b}^4 D^2 L_1^2\log T}{\delta_d^2}
\end{aligned}
\end{equation}
Notice that, by Condition 1, $C_4$ can take as small as $O(m\log mn/(\sigma_{\min}^2\underline{\cL}_f^2 ) )$. Putting together \eqref{eq:rem_1} and \eqref{eq:rem_2}  we conclude the proof of lemma \ref{lemma:early_stop}:
\begin{equation*}
     \E (T-\tau) \le \frac{C m^2n\log m\Bar{b}^4 D^2 L_1^2\log T }{\delta_d^2\sigma_{\min}^2\underline{\cL}_f^2}
\end{equation*}


\subsection{Proof of Theorem \ref{thm:regret_ub}}\label{proof:thm_regret_ub}

Equipped with Lemma~\ref{lemma:alg_conv} and ~\ref{lemma:early_stop}, we now continue sketching the proof of Theorem~\ref{thm:regret_ub}. By Proposition~\ref{prop:regret_decomp}, it suffices to bound the two terms there. 

\begin{proof}{Proof of Theorem~\ref{thm:regret_ub}}
The proof continues from Proposition~\ref{prop:regret_decomp}.

\noindent\textit{Step 1: bounding $\textsf{R.1}$}.  By Fenchel conjugate, we re-write the bridging function $g(\nu)$ by
\begin{equation*}
    \begin{aligned}
    g (\nu)= & \E \left[ f_t (\tilde x_t(\nu)) + r(\tilde a (\mu)) + (\tilde a(\mu)-b_t \tilde x_t(\nu))^{\top}\mu^* +(d-b_t \tilde x_t(\nu))^{\top}\lambda^*\right] \\
    = & \E\left[f^*_t(b_t^\top(\lambda+\mu) ) +r^*(-\mu)\right] + \E(\mu^*-\mu)^\top (\tilde{a}(\mu)-b_t\tilde{x}_t(\nu))+\E(\lambda^*-\lambda)^\top (d-b_t\tilde{x}_t(\nu)) \\
    = & \E\left[f^*_t(b_t^\top(\lambda+\mu) ) +r^*(-\mu)\right] -  \E\left[\nabla f_t^*(b_t^\top(\lambda+\mu))^\top b_t^\top (\lambda+\mu-\lambda^*-\mu^*) \right.\\
    & \left. -\nabla r^*(-\mu)^\top (\mu-\mu^*)+d^\top (\lambda-\lambda^*) \right] \\
    = & s(\nu,d) -\left\langle\nabla_{\nu} D (\nu,\mu^*,d)- \nabla_{\nu} D (\nu^*,\mu^*,d),  (\nu-\nu^*) \right\rangle 
    \end{aligned}
\end{equation*}

By Assumption \ref{asm:regularizer} and \ref{asm:nondegeneracy},  we get
\begin{equation*}
\begin{aligned}
   g(\nu^*)-g(\nu)= &s(\nu^*,d)-s(\nu,d) + \left\langle\nabla_{\nu} D (\nu,\mu^*,d)- \nabla_{\nu} D (\nu^*,\mu^*,d),  (\nu-\nu^*) \right\rangle  \\ 
    & + \left\langle\nabla s(\nu,d)-\nabla s(\nu^*,d), \nu-\nu^* \right\rangle \le \left(2\Bar{\cL}_s-\underline{\cL}_s \right)\norm{\nu-\nu^*}^2_2 \\
    & = (\Bar{b}^2\Bar{\cL}_f-\frac{1}{2}\sigma_{\min} \underline{\cL}_f )\norm{\nu-\nu^*}_2^2
\end{aligned}
\end{equation*}

Then Lemma \ref{lemma:alg_conv} gives rise to the following bound. 
$$
\E\left[\sum_{t=1}^{\tau} g(\nu^*)-g(\nu_{t-1})\right] \le O( \log T) .
$$

\vspace{2mm}
\noindent\textit{Step 2: bounding $\textsf{R.2}$}.  This term can be controlled by the definition of stopping time and Lemma~\ref{lemma:early_stop}. 

\begin{equation*}
    \begin{aligned}
    \E &\left[  2(\bar{f}+\bar{r}+C_3)(T-\tau) + \left\langle \lambda^*,  \sum_{t=1}^{\tau}( d - b_t x_t) \right\rangle   \right] \\
   & =\E\left[2(\bar{f}+\bar{r}+C_3)(T-\tau)+ \left\langle \lambda^*,  d_{\tau} (T-\tau) -d(T-\tau) \right\rangle \right]\\
   & \le \E\left[2(\bar{f}+2\bar{r}+C_3)(T-\tau) +\sum_{i\in I_{\textsf{B}}}\lambda^*_i (d_i+\delta_d) (T-\tau)\right] \\
   & \le  (2\bar{f}+2\bar{r}+2C_3+ (\norm{d}+\sqrt{m}\delta_d) \frac{2(\bar f+\bar r)}{\underline{d}} ) \E (T-\tau) = O(\log T).
    \end{aligned}
\end{equation*}
\end{proof}
Thus we finish the proof.
\vspace{2mm}

\noindent \textit{Step 3: bounding $\textsf{R.3}$}. This term requires the most effort. It concerns the combined effects of variable splitting and complementary slackness. The following lemma is important for bounding this term. 

\begin{lemma}\label{lemma:c_s}
 Suppose $i\in I_{\textsf{NB}}$. Under Assumptions \ref{asm:basic}-\ref{asm:nondegeneracy:3}, Algorithm \ref{alg:framework} with selected dual optimizer $\{\mathcal{B}_t\}_{t\ge 1}$ satisfying Condition \ref{condition:conv} and stopping time (\ref{eq:tau}) ensures
 \begin{equation*}
 \begin{aligned}
     \E&\norm{\hat{\mu}_{I_{\textsf{NB} }, T}-\mu^*_{I_{\textsf{NB} } }  }^2_2 \le O(\frac{\log T}{T} ),\text{ and  } \ \E \norm{\sum_{t=1}^{\tau}(\tilde{a}(\mu^*)-b_t x_t)}^2_2 \le O(T\log T).   \\
 \end{aligned}
 \end{equation*}
 
\end{lemma}

The proof Lemma \ref{lemma:c_s} exploits the local smoothness of $r$ and $\tilde{x}_t$  with the help of the optimality of $\mu^*$, i.e., $\tilde{a}(\mu^*)=\E b_t \tilde{x}_t(
\nu^*)$ . We first check the binding dimensions: for all the binding dimensions $i\in I_{\textsf{B}} $, since $\E (b_t \tilde{x}_t(
\nu^*))_i=d_i$, by the definition of stopping time, it can be shown that
\begin{equation}
    \begin{aligned}
     \E\left\langle \hat{\mu}_{I_{\textsf{B}},T}-\mu^*_{I_{\textsf{B}}}, \sum_{t=1}^{\tau} \tilde{a}(\mu^*)_{I_{\textsf{B}}}- (b_t x_t)_{I_{\textsf{B}}} \right\rangle=  & \E\left\langle \hat{\mu}_{I_{\textsf{B}},T}-\mu^*_{I_{\textsf{B}}}, \sum_{t=1}^{\tau} d_{I_{\textsf{B}}}- (b_t x_t)_{I_{\textsf{B}}} \right\rangle \\
     & \le \E\sqrt{mn}GD\Bar{b}(d+\delta_d)(T-\tau) \\
     & =O(\log T)
    \end{aligned}
\end{equation}
We then go back to the non-binding dimension $i\in I_{\textsf{NB}} $ with the help of Lemma \ref{lemma:c_s}. By Cauchy–Schwarz inequality, we get
\begin{equation}
\label{eq:regret_decomp_p3}
    \begin{aligned}
    \E\left[  \left\langle\hat{\mu}_{I_{\textsf{NB}},T}-\mu^*_{I_{\textsf{NB}}} , \sum_{t=1}^{\tau}\tilde{a}(\mu^*)_{I_{\textsf{NB}}}-(b_t x_t)_{I_{\textsf{NB}}} \right\rangle \right]\le & \bigg( \E\norm{ \hat{\mu}_{I_{\textsf{NB}},T}-\mu^*_{I_{\textsf{NB}}} }^2_2\E \norm{\sum_{t=1}^{\tau}(\E b_t \tilde{x}_t(
\nu^*)-b_t x_t)_{I_{\textsf{NB}}}}^2_2 \bigg)^{1/2}.\\
    \end{aligned}
\end{equation}
Thus $\textsf{R.3}$ can be controlled by $\log T $. The proof is concluded. The constant factor is as large as 
$$\mathring{C}\lesssim \frac{ m^2n\log m\Bar{b}^6 D^4 L_1^2  }{\delta_d^2\sigma_{\min}^2\underline{\cL}_f^2}\cdot\left( (\bar{f}+\bar{r} + \sqrt{mn}GD\Bar{b} )\vee \frac{n\Bar{b}^2G^2 + m\Bar{\cL}_r^2 \delta_0^2}{\delta_0^2 }\right) $$ 
from the proof of Lemma \ref{lemma:c_s}. Generally, taking $m\lesssim n$ we have the order of $\mathring{C}=$ $O(m^2n^2\log m)$

\subsection{Proof of Lemma \ref{lemma:c_s}}\label{sec:proof_c_s}
For the $\E\norm{\hat{\mu}_{I_{\textsf{NB}},T}-\mu^*_{I_{\textsf{NB}}}}^2_2$, $i\in I_{\textsf{NB}}$, the optimality of $\mu^*$  implies $\tilde{a}(\mu^*)=\E b_t \tilde{x}_t(\nu^*)$, thus by conjugate we have
\begin{equation*}
    \begin{aligned}
    \E&\norm{\hat{\mu}_{I_{\textsf{NB}}, T}-\mu^*_{I_{\textsf{NB}}} }^2_2 =  \E \norm{\nabla_{I_{\textsf{NB}}} r(\tilde{a}(\mu^*)) -\nabla_{I_{\textsf{NB}}} r(\frac{\sum_{t=1}^{T}b_t x_t}{T}) }^2_2 \\
    & \le n\Bar{b}^2G^2\PP\left(\norm{ \E b_t \tilde{x}_t(\nu^*) -\frac{\sum_{t=1}^{T}b_t x_t}{T}}\ge \delta_0\right)  + \E  m\Bar{\cL}_r^2\norm{\tilde{a}(\mu^*)-\frac{\sum_{t=1}^{T}b_t x_t}{T}}^2_2 \\
    &\le \frac{n\Bar{b}^2G^2 + m\Bar{\cL}_r^2 \delta_0^2}{\delta_0^2 }  \E\norm{\frac{\sum_{t=1}^{T}b_t x_t-\E b_t \tilde{x}_t(\nu^*)}{T}-\frac{\sum_{t=1}^{ \tau }\E\left[ b_t \tilde{x}_t(\nu_{t-1})\middle | \cH_{t-1}\right]  }{T}+ \frac{\sum_{t=1}^{ \tau }\E\left[ b_t \tilde{x}_t(\nu_{t-1})\middle | \cH_{t-1}\right]}{T} }^2_2 \\
    &\le  3 C_0 \left. \E\norm{\frac{\sum_{t=1}^{\tau}b_t \tilde{x}_t(\nu_{t-1})- \E\left[ b_t \tilde{x}_t(\nu_{t-1})\middle | \cH_{t-1}\right] }{T}}^2_2 \right. \text{(part \ref{sec:proof_c_s}.1)}\\
    & + 3 C_0 \left.  \E\norm{\frac{\sum_{t=1}^{\tau}\E\left[ b_t \tilde{x}_t(\nu_{t-1})\middle | \cH_{t-1}\right]- \E b_t\tilde{x}_t(\nu^*) }{T}}^2_2\right. \text{(part \ref{sec:proof_c_s}.2)}\\
    & + 3 C_0 \left.  \E\norm{\frac{\sum_{t=\tau+1}^{T}b_t x_t- \E b_t \tilde{x}_t(\nu^*) }{T}}^2_2\right. \text{(part \ref{sec:proof_c_s}.3)}.
    \end{aligned}
\end{equation*}

For the part \ref{sec:proof_c_s}.1, notice that this is a martingale series. Thus we have 
\begin{equation*}
    \begin{aligned}
\E\norm{\frac{\sum_{t=1}^{\tau}b_t \tilde{x}_t(\nu_{t-1})- \E\left[ b_t \tilde{x}_t(\nu_{t-1})\middle | \cH_{t-1}\right] }{T}}^2_2 & \le \frac{\sum_{t=1}^T \operatorname{Var}(b_t \tilde{x}_t(\nu_{t-1})- \E\left[ b_t \tilde{x}_t(\nu_{t-1})\middle | \cH_{t-1}\right])  }{T^2} \\
&\le \frac{ n D^2 \bar{b}^2 }{T}
    \end{aligned}
\end{equation*}

For the part \ref{sec:proof_c_s}.2, applying Assumption \ref{asm:k-thmoment} we can yield 
\begin{equation*}
    \begin{aligned}
    \E &\norm{\frac{\sum_{t=1}^{\tau}\E\left[ b_t \tilde{x}_t(\nu_{t-1})\middle | \cH_{t-1}\right]- b_t \tilde{x}_t(\nu^*) }{T}}^2_2  \le  \E \frac{\tau\sum_{t=1}^{\tau}\norm{ \E\left[ b_t \tilde{x}_t(\nu_{t-1})\middle | \cH_{t-1}\right]- b_t \tilde{x}_t(\nu^*)}^2_2 }{T^2} \\
    & \le \frac{\E \sum_{t=1}^{T} \norm{\E\left[ b_t \tilde{x}_t(\nu_{t-1})\middle | \cH_{t-1}\right]- b_t \tilde{x}_t(\nu^*)}^2_2\mathbb{I}(t\le \tau)}{T} \\
    &= \frac{\sum_{t=1}^{T}\E \norm{\left[\E\left[b_t \tilde{x}_t(\nu_{t-1})- b_t \tilde{x}_t(\nu^*) \middle| \cH_{t-1},b_t\right]\mathbb{I}(t\le \tau)\right]}^2_2 }{T} \\
    & \stackrel{\text { (a) }}{\leq} L_1^2 \bar{b}^2\frac{\sum_{t=1}^{T}\E \left[\norm{\nu_{t-1}-\nu^*}^2_2 \mathbb{I}(t\le \tau)\right] }{T}  =  L_1^2  \bar{b}^2\frac{\E \left[\sum_{t=1}^{\tau}\norm{\nu_{t-1}-\nu^*}^2_2\right] }{T} \\
    &  \stackrel{\text { (b) }}{\leq} \frac{ L_1^2  \bar{b}^2 m^2n\log m D^2 \Bar{b}^4 L_1^2 \log T}{\sigma_{\min}^2\underline{\cL}_f^2 T} 
    \end{aligned}
\end{equation*}

(a) is by Assumption \ref{asm:k-thmoment} and the fact $\{t\le \tau\}\in \sigma(\cH_{t-1})$, and $\tilde{x}_t(\cdot)\indep \nu_{t-1}$. (b) is by Lemma \ref{lemma:alg_conv}.




For the part \ref{sec:proof_c_s}.3, since $\norm{b_t x_t-\E b_t \tilde{x}_t(\nu^*)}^2_2\le n D^2 \bar{b}^2$, by Lemma \ref{lemma:early_stop}, we have
\begin{equation*}
    \begin{aligned}
    \E\norm{\frac{\sum_{t=\tau+1}^{T}b_t x_t-\E b_t  \tilde{x}_t(\nu^*) }{T}}^2_2 \le & \frac{\E\left[(T-\tau)\sum_{t=\tau+1}^{T}\norm{b_t x_t- \E b_t  \tilde{x}_t(\nu^*)}^2\right]}{T^2} \\
    \le & n D^2 \bar{b}^2\frac{\E(T-\tau)}{T}
    \le \frac{C m^2n^2\log m\Bar{b}^6 D^4 L_1^2 }{\delta_d^2\sigma_{\min}^2\underline{\cL}_f^2}  \frac{\log T}{T}
    \end{aligned}
\end{equation*}

We then go back to control the next term $\E \norm{\sum\limits_{t=1}^{\tau}(\tilde{a}(\mu^*)-b_t x_t)}^2_2$. This can be bounded exactly by part \ref{sec:proof_c_s}.1 and \ref{sec:proof_c_s}.2. From the argument above, we show that  $\E \norm{\sum\limits_{t=1}^{\tau}(\tilde{a}(\mu^*)-b_t x_t)}^2_2$ is controlled by $ O(T \log T) $ . Thus we finish the proof.


\subsection{Proof of Theorem \ref{thm:regret_lb} and \ref{thm:ogd_lb}}
 We specify a non-regularized case where $f_t(x)=-\frac{1}{4}(x-2\xi_t)^2+\xi_t^2$, with  fixed cost $b_t=1$, average resource capacity $d=\frac{1}{2}D$, and $\xi_t$ following two-point distribution $\PP(\xi_t=\frac{1}{2}D)=\PP(\xi_t=\frac{3}{4}D)=\frac{1}{2} $. Then the dual problem is
\begin{equation*}
    D_t(\lambda)=\begin{cases}
    \frac{1}{2}D\lambda & \text{ if } \lambda>\xi_t \\
    -\frac{1}{4}D+\xi_t-\frac{1}{2}D\lambda &\text{ if }  \lambda<\xi_t-\frac{1}{2}D \\
    \lambda^2-2(\xi_t-\frac{1}{4}D)\lambda+\xi_t^2 &\text{ if } \xi_t-\frac{1}{2}D\le\lambda\le\xi_t.
    \end{cases}
\end{equation*}

 Suppose $\lambda^*$ is the optimal solution to the deterministic problem $\min_{\lambda\ge 0} D(\lambda)=\E D_t(\lambda)$. Without loss of generality, we assume that our dual variable $\lambda$ is taken within $\left[\frac{1}{4}D,\frac{1}{2}D\right]$ since we know that $\lambda^*=\E\xi_t-\frac{1}{4}D=\frac{3}{8}D \in\left[\frac{1}{4}D,\frac{1}{2}D\right]$.
\begin{equation*}
    D_t(\lambda)=f^*_t(\lambda)+d^\top\lambda =   \lambda^2-2(\xi_t-\frac{1}{4}D)\lambda+\xi_t^2.
\end{equation*}
For the dual-based police $\left\{\lambda_t\right\}_{t=0}^{T-1}$, the corresponding primal variable is $x_t=\tilde{x}_t(\lambda_{t-1})= 2\xi_t-2\lambda_{t-1}$ or void if the resource is depleted. We have the following regret:


$$
\begin{aligned}
     \operatorname{Regret}(A) = & R^*(\mathcal{P}) -R\left(A\middle| \mathcal{P}\right) \\
     = & \E\left[ \max_{x_t\in[0,D]} \left\{ \sum_{t=1}^{T} f_t(x_t) \text{  s.t. }\sum_{t=1}^T x_t \le \frac{1}{2}D T \right\} \right] - \E\left[ \sum_{t=1}^{T} f_t(x_t) \right] \\
     = & \E\left[ \min_{\lambda\ge0} \left\{ \sum_{t=1}^{T} D_t(\lambda) \right\} \right] - \E\left[ \sum_{t=1}^{T} f_t(x_t) \right] \\
     = & \E\left[ \sum_{t=1}^{T} D_t(\lambda^*_T)  \right] - \E\left[ \sum_{t=1}^{T} f_t(x_t) \right]
\end{aligned}
$$
Define the corresponding 
$$g(\lambda)=\E\left[f_t(\tilde x_t(\lambda))+\langle d-b_t\tilde{x}_t(\lambda),\lambda^*\rangle \right]= D(\lambda)-\langle \nabla D (\lambda),\lambda-\lambda^* \rangle$$
We have $g(\lambda^*)=D(\lambda^*)$ and $g(\lambda^*)-g(\lambda)= (\lambda^*-\lambda)^2$. For the quadratic function $D_t$, we always have $D_t(\lambda_1)-D_t(\lambda_2)=\nabla D_t(\lambda_2)(\lambda_1-\lambda_2)+(\lambda_1-\lambda_2)^2$. Thus it follows that

\begin{equation*}
    \begin{aligned}
         \operatorname{Regret}(A) & =  \E\left[ \sum_{t=1}^{T} D_t(\lambda^*_T)  \right] - T D(\lambda^*) + T D(\lambda^*)- \E\left[ \sum_{t=1}^{T} f_t(x_t) \right] \\
         & = \E\left[ \sum_{t=1}^{T} D_t(\lambda^*_T)-D_t(\lambda^*)  \right] +  T g(\lambda^*)- \E\left[ \sum_{t=1}^{T} f_t(x_t) \right]\\
         & = -\E \left[ \sum_{t=1}^{T} \left[\nabla D_t(\lambda^*_T)(\lambda^*-\lambda^*_T)\right]+T(\lambda^*-\lambda^*_T)^2\right] + T g(\lambda^*)- \E\left[ \sum_{t=1}^{T} f_t(x_t) \right]\\
         & = -T\E(\lambda^*-\lambda^*_T)^2 + T g(\lambda^*)- \E\left[ \sum_{t=1}^{T} f_t(x_t) \right].
    \end{aligned}
\end{equation*}

By the dual convergence in Theorem \ref{thm:dual_conv}, we know that the first term $T\E(\lambda^*-\lambda^*_T)^2$ can be bounded by a constant. Now we handle the second term by controlling the stopping time. Define the stopping time $\tau_0=\min\left\{ t\in[T]\middle| \sum_{i=1}^{t}x_t\ge \frac{1}{2}D T-D \right\}\cup\{T\}$. Then when $t\le \tau_0$, we always have $x_t=\tilde{x}_t(\lambda_{t-1})= 2\xi_t-2\lambda_{t-1}$, and $0\le \sum_{t=\tau_0+1}^{T}x_t\le D$ for $t>\tau$. Then we have

\begin{equation*}
    \begin{aligned}
        \E\left[ \sum_{t=1}^{T} f_t(x_t) \right] & \le   \E\left[ \sum_{t=1}^{\tau_0} f_t(\tilde x_t(\lambda_{t-1}))  +\langle \frac{1}{2}D-\tilde{x}_t(\lambda_{t-1}),\lambda^*\rangle\right]   + \E\left[ \sum_{t=\tau_0+1}^{T} f_t(x_t) +\langle \frac{1}{2}D-x_t(\lambda),\lambda^*\rangle \right]  \\
         & \le \E \sum_{t=1}^{\tau} g(\lambda_{t-1}) +   \E \left[\sum_{t=\tau_0+1}^{T} \frac{3}{4}D x_t +\frac{1}{2}D\lambda^*\right] \\
         & \le  \E \sum_{t=1}^{\tau} g(\lambda_{t-1}) + \frac{3}{16}D^2\E[T-\tau_0]+\frac{3}{4}D^2.
    \end{aligned}
\end{equation*}
The first inequality is because of the resource constraint, and the second one is because $f_t(x)\le f_t'(0)(x-0)\le \frac{3}{4}Dx$. If we specify $\lambda_{t-1} =\xi_t$ when the resource constraints are violated, we also have $ \E\left[ \sum_{t=1}^{T} f_t(x_t) \right]\le \E \sum_{t=1}^{T} g(\lambda_{t-1})$. Then

\begin{equation}\label{eq:one_d_lb_stopping}
    \begin{aligned}
        T g(\lambda^*)- \E\left[ \sum_{t=1}^{T} f_t(x_t) \right] & \ge \E \left[ \sum_{t=1}^{\tau_0} g(\lambda^*)-g(\lambda_{t-1}) \right]+ \E (g(\lambda^*)-\frac{3}{16}D^2)\E[T-\tau_0]-\frac{3}{4}D^2\\
       &  = \E\left[ \sum_{t=1}^{\tau_0} (\lambda^*-\lambda_{t-1})^2 \right] + \frac{5}{64}D^2\E[T-\tau_0] -\frac{3}{4}D^2,
    \end{aligned}
\end{equation}
or $T g(\lambda^*)- \E\left[ \sum_{t=1}^{T} f_t(x_t) \right]  \ge \E\left[ \sum_{t=1}^{T} (\lambda^*-\lambda_{t-1})^2 \right]$ . Applying van Trees inequality to the estimation of $\lambda^*$ \citep{li2021online}, we can prove the Theorem \ref{thm:regret_lb}. To prove the Theorem \ref{thm:ogd_lb}, we only need to show the stopping time $\E[T-\tau_0]\ge \Omega(\sqrt{T})$ given the convergence condition. This proof is inspired by \cite{arlotto2019uniformly}. Denote $t'=\lfloor T-\sqrt{T}\rfloor$. We show that $\PP(\tau_0\le t')$ is larger that a constant $c$ so that $\E\tau_0\le (1-c)T+c(T-\sqrt{T})\le T-c\sqrt{T}$.

\begin{equation*}
    \begin{aligned}
        \PP(\tau_0\le t') & =\PP\left(\sum_{t=1}^{t'}2(\xi_t-\lambda_{t-1})\ge \frac{DT}{2}-D \right)\\
        & \ge \PP\left(\left\{\sum_{t=1}^{t'}2(\xi_t-\lambda^*)\ge \frac{DT}{2}-D +\varepsilon D\sqrt{ t'} \right\} \cap \left\{ \sum_{t=1}^{t'} |\lambda_{t-1}-\lambda^*|< \varepsilon D\sqrt{ t'} \right\} \right) \\
        & \ge \PP \left(\left\{\sum_{t=1}^{t'}2(\xi_t-\lambda^*)\ge \frac{DT}{2}-D +\varepsilon D\sqrt{ t'} \right\}\right) -\PP\left(\sum_{t=1}^{t'} |\lambda_{t-1}-\lambda^*|\ge \varepsilon D\sqrt{ t'}  \right)
    \end{aligned}
\end{equation*}

With the condition $\E|\lambda_t-\lambda^*|\le c_2 D/\sqrt{t+1}$, we have $\PP\left(\sum_{t=1}^{t'} |\lambda_{t-1}-\lambda^*|\ge \varepsilon D\sqrt{ t'}  \right)\le \frac{2c_2 }{\varepsilon}$ by Chebyshev's inequality. Then it holds that

\begin{equation*}
    \begin{aligned}
       \PP(\tau_0\le t') 
       & \ge \PP \left(\left\{\sum_{t=1}^{t'}2(\xi_t-\lambda^*)\ge \frac{DT}{2}-D +\varepsilon D\sqrt{ t'} \right\}\right) -\frac{2c_2}{ \varepsilon} \\
        & = \PP \left(\left\{\sum_{t=1}^{t'}\frac{4}{D}(\xi_t-\frac{D}{2})\ge \frac{t'}{2}+(T-t')-2 +2\varepsilon\sqrt{t'} \right\}\right) -\frac{2c_2}{ \varepsilon} \\
        & \ge \PP \left(\left\{\sum_{t=1}^{t'}\frac{4}{D}(\xi_t-\frac{D}{2})\ge \frac{t'}{2}+(1+2\varepsilon)\sqrt{t'} \right\}\right) -\frac{2c_2}{ \varepsilon} ,
    \end{aligned}
\end{equation*}
where $\sum_{t=1}^{t'}\frac{4}{D}(\xi_t-\frac{D}{2})$ follows the binomial distribution $B(t',\frac{1}{2})$, with mean $\mu=\frac{t'}{2}$ and standard deviation $\sigma =\frac{\sqrt{t'}}{2}$. The second inequality is because $T-t'\le \sqrt{T}+1$ and $\sqrt{T}-\sqrt{t'}\le \sqrt{T}-\sqrt{T- \sqrt{T}} = \frac{\sqrt{T}}{\sqrt{T- \sqrt{T}}+\sqrt{T}}\le 1$. For the binomial distribution, $\PP(X\ge \mu+x\sigma)$ converge to $\Phi(-x)$ for any $x$ with known $O(\frac{1}{\sqrt{n}}) $ speed by Berry-Esseen CLT where $\Phi(x)$ is the distribution function of standard normal distribution. We let $c_2=\sup_{\varepsilon>0}\varepsilon \Phi(-2-4\varepsilon)/{4} $. Then there exists $\varepsilon_0>0$ such that when $T$ is large enough, $\PP\left(\left\{\sum_{t=1}^{t'}\frac{4}{D}(\xi_t-\frac{D}{2})\ge \frac{t'}{2}+(1+2\varepsilon_0)\sqrt{t'} \right\}\right)\ge\frac{3c_2}{\varepsilon_0}$, which indicates that $\PP(\tau_0\le t') \ge \frac{c_2}{\varepsilon_0}$. This makes our proof complete.

\subsection{Proof of Theorem \ref{thm:infreq} }
Theorem \ref{thm:infreq} can be proved following the same path as in the proof of Theorem \ref{thm:regret_ub}. The key is that, with a good initialization, Lemma \ref{lemma:alg_conv} and \ref{lemma:early_stop} still hold. This has actually been mentioned in the proof of Lemma \ref{lemma:alg_conv} and \ref{lemma:early_stop} and thus omitted here.
\subsection{Proof of Theorem \ref{thm:fast}}
To prove the $O(\log^2 T)$ bound, we revise the previous proof of Lemma \ref{lemma:alg_conv} and \ref{lemma:early_stop} and re-compute some important rates. For online gradient descent, without loss of generality, suppose $d_t\in \Omega_d$ before the stopping time $\tau$. Then By Assumption \ref{asm:nondegeneracy} and the stochastic gradient descent approach in \cite{rakhlin2012making}, for $t\ge T_j$ ,we have 
\begin{equation*}
\begin{aligned}
       \E &\left( \norm{\nu_{t+1} -\nu^*(d_{T_j}) }_2^2 + \norm{\mu_{t+1} -\mu^*(d_{T_j}) }_2^2\right) \le \E \norm{\nu_{t}-\eta_{t+1} \nabla D_{t+1,\nu}(\nu_{t},\mu_{t},d_{T_j} ) -\nu^*(d_{T_j}) }_2^2 \\
        & \ \ + \E\norm{\mu_{t+1} -\mu^*(d_{T_j})-\eta_{t+1} \nabla D_{t+1,\mu}(\nu_{t},\mu_{t},d_{T_j} ) }_2^2, \\
\end{aligned}
\end{equation*}
and we also have
\begin{equation*}
    \begin{aligned}
         \E &\norm{\nu_{t+1} -\nu^*(d_{T_j}) }_2^2 \le \E \norm{\nu_{t}-\eta_{t+1} \nabla D_{\nu}(\nu_{t},\mu_{t},d_{T_j} ) -\nu^*(d_{T_j}) }_2^2, \text{ by projection} \\
     \E & \norm{\nu_{t+1} -\nu^*(d_{T_j}) }_2^2 \le (1-2\eta_{t+1} \sigma_{\min} \underline{\cL}_f) \E \norm{\nu_{t} -\nu^*(d_{T_j}) }_2^2 + \eta_{t+1}^2 n \Bar{b}^2D^2.
    \end{aligned}
\end{equation*}
And hence, by choosing $\eta_{t}=1/( \sigma_{\min} \underline{\cL}_f  (t-T_j+1))$, we have
\begin{equation}\label{eq:fast-converge}
   \E\left [ \norm{ \nu_t -\nu^*(d_{T_j})  }_2^2 \middle| \cH_{T_j} \right]\le \frac{4 n \Bar{b}^2 D^2 }{\sigma_{\min}^2 \underline{\cL}_f^2  (t-T_j+1)},  
\end{equation}
where $\nu^*(d_{T_j})$ is part of the optimal solution to $D(\blambda,d_{T_j} )$. See  \cite{rakhlin2012making} for a detailed approach. By this convergence rate, we have the following rates
\begin{equation*}
    \E\left[\sum_{t=1}^{\tau } \norm{\nu_{t-1}-\nu^*(d_{t-1})}^2_2\right] \le  \sum_{j=1}^{J} \frac{4 n \Bar{b}^2 D^2(\log (T_{j+1}-T_j ) +1)  }{\sigma_{\min}^2 \underline{\cL}_f^2 } \le \frac{C n \Bar{b}^2 D^2(\log^2 T)  }{\sigma_{\min}^2 \underline{\cL}_f^2 }
\end{equation*}
We now revisit the previous induction of $d_t$ in \eqref{eq:dt-induction}. For the three parts $A'$, $B'$, $C'$, we have the new rates:
\begin{equation*}
    \begin{aligned}
   A' & = \E  \frac{\left(\sum_{k=T_{j}+1}^{T_{j+1}}\left[d_{i,T_{j}}' -\E\left[\left(b_k\tilde{x}_{k}(\nu^*(d_{T_j}) )\right)_i|\cH_{T_j}\right] +\E\left[\left(b_k\tilde{x}_{k}(\nu^*(d_{T_j}) )\right)_i|\cH_{T_j}\right]- \left(b_k\tilde{x}_{k}(\nu_{k-1 } )\right)_i\right] \right)^2 }{(T-T_{j+1})^2} \mathbb{I}(\tau> T_{j})\\
   & \le 2\E\frac{\sum_{k=T_{j}+1}^{T_{j+1}}\left(\E\left[ \left(b_k\tilde{x}_{k}(\nu^*(d_{T_{j}}) )\right)_i-\left(b_k\tilde{x}_{k}(\nu_{k-1 } )\right)_i \middle| \cH_{T_j} \right]\right)^2 }{T-T_{j+1} }\\ 
   & \le \frac{2 n \Bar{b}^4 L_1^2 D^2}{\sigma_{\min}^2 \underline{\cL}_f^2 } \frac{ \log(T_{j+1}-T_{j})  }{T-T_{j+1} } \\ 
   & \le \frac{C_4 n\Bar{b}^4 L_1^2 D^2 \log T   }{ \sigma_{\min}^2 \underline{\cL}_f^2 \rho^{j}T } 
\end{aligned}
\end{equation*}
And for $B'$, we also have: 
 $$d_{i,T_{j}}' -\E\left[ \left(b_k\tilde{x}_{k}(\nu^*(d_{T_{j}}) )\right)_i \middle| \cH_{T_j} \right]=0, \text{ for any } k\ge T_j+1.$$
 For part $C'$, it follows that
 \begin{equation*}
    \begin{aligned}
   C' & = 2 \E\left[\E\frac{\left(d_{i,T_{j}}' - d_{i}\right) \left(\sum_{k=T_{j}+1}^{T_{j+1}}\left(b_k\tilde{x}_{k}(\nu_{k-1 } )\right)_i-\left(b_k\tilde{x}_{k}(\nu^*(d_{T_{j}}) )\right)_i \right) }{T-T_{j+1}} \mathbb{I}(\tau> T_j)\middle| \cH_{T_j} \right]\\
   & \le 2 \Bar{b}^2 L_1\E \frac{\abs{d_{i,T_{j}}' - d_{i}} \sum_{k=T_{j}+1}^{T_{j+1}}\norm{\nu_{k-1 }-\nu^*(d_{T_{j}}) } }{T-T_{j+1} }\\
   &  \le {\frac{C \sqrt{n} \Bar{b}^3 D L_1 }{\sigma_{\min} \underline{\cL}_f  } \sqrt{\mathbb{E}\left(d_{i ,T_{j}}^{\prime}-d_{i}\right)^{2}} \sqrt{\frac{1}{T-T_j} }  } \\
   & \le {\frac{C \sqrt{n} \Bar{b}^3 D L_1 }{\sigma_{\min} \underline{\cL}_f  } \sqrt{\mathbb{E}\left(d_{i ,T_{j}}^{\prime}-d_{i}\right)^{2}} \sqrt{\frac{1}{\rho^j T } }  }.
    \end{aligned}
\end{equation*}
Thus, we then get the recurrence relation of $d_{i,T_{j} }^{\prime}-d_{i}$:
\begin{equation*}
    \begin{aligned}
    \mathbb{E}\left(d_{i, T_{j+1}}^{\prime}-d_{i}\right)^{2} \leq & \mathbb{E}\left(d_{i ,T_{j}}^{\prime}-d_{i}\right)^{2}
     + \frac{C_4 n\Bar{b}^4 L_1^2 D^2 \log T   }{ \sigma_{\min}^2 \underline{\cL}_f^2\rho^{j}T }\\
     &+{\frac{C \sqrt{n} \Bar{b}^3 D L_1}{\sigma_{\min} \underline{\cL}_f} \sqrt{\mathbb{E}\left(d_{i ,T_{j}}^{\prime}-d_{i}\right)^{2}} \sqrt{\frac{1}{\rho^j T } }  }.
    \end{aligned}
\end{equation*}
Since $d_0=d$, by induction we have $$\mathbb{E}\left(d_{i ,T_{j}}^{\prime}-d_{i}\right)^{2}\le C_5 C_\rho n D^2 \Bar{b}^6 L_1^2 \frac{\log T}{\sigma_{\min}^2 \underline{\cL}_f^2\rho^j T} $$

So, we have
$$
\begin{aligned}
2 \E\left[\sum_{t=1}^{\tau}\norm{\nu^*(d_{t-1})-\nu^*}^2\right] 
\le  & 2 \E\left[\sum_{t=1}^{\tau}\sum_{i\in I_{\textsf{B}}} (d_{i,{t-1}}-d_i)^2\right] \\
\le & 2m \mathbb{E}\sum_{j=1}^{J} (T_j-T_{j-1}) \left[\left(d_{i, T_j}^{\prime}-d_{i}\right)^{2}\right] \le C_5 \frac{m n D^2 \Bar{b}^6 L_1^2 \log^2 T}{\sigma_{\min}^2\underline{\cL}_f^2} +C, \\
 \text{ and } \E\left[\sum_{t=1}^{\tau} \norm{\nu_{t-1}-\nu^*}^2_2\right]\le  &  \frac{C mn \Bar{b}^6 D^2 L_1^2 (\log^2 T)  }{\sigma_{\min}^2 \underline{\cL}_f^2 } +  2C_2,
\end{aligned}
$$
Extending this computation to the control of stopping time, we also have 
\begin{equation*}
     \E (T-\tau) \le \frac{C mn \Bar{b}^6 D^2 L_1^2\log^2 T }{\delta_d^2\sigma_{\min}^2\underline{\cL}_f^2},
\end{equation*}
where the proof essentially follows Lemma \ref{lemma:early_stop}. Similarly, following the proof of Lemma \ref{lemma:c_s}, we also have
 \begin{equation*}
 \begin{aligned}
     \E&\norm{\hat{\mu}_{I_{\textsf{NB}},T}-\mu^*_{I_{\textsf{NB}}} }^2_2 \le O(\frac{\log^2 T}{T} ),\text{ and  } \ \E \norm{\sum_{t=1}^{\tau}(\tilde{a}(\mu^*)-b_t x_t)}^2_2 \le O(T\log^2 T).   \\
 \end{aligned}
 \end{equation*}
 Combining these together, we have the regret upper bound 
\begin{equation*}
    \text{Regret}(A)\leq \widetilde{C}\log^2 T,
\end{equation*}
for some constant $\widetilde{C}=O(mn^2)$ depending on the values in Assumptions \ref{asm:basic}-\ref{asm:nondegeneracy:3}. Here the additional $n$ factor comes from the complementary slackness bound by the scale of our problem as is shown in Lemma \ref{lemma:c_s}.

\section{Numerical Experiments}\label{appdx:exp}

The implementation details on multiple input models are as follows: the dual updates are calculated by closed-form solutions to Equation~\eqref{eq:dual_update} under input I-III and by \textit{cvxpy} (\cite{diamond2016cvxpy}) under input IV. See Table~\ref{table:I} for parameter settings of different inputs. For each $T$, we randomly generate $T$ observations from distribution, run each algorithm in an online fashion, and keep record of their output. The regret is calculated as the difference between the offline optimal (solved by \textit{cvxpy}) and the online output. We report the average regret over 10 repetitions for all following graphs.
\begin{table}[!htbp]
    \centering
    \begin{tabular}{c|c|c|c|c}
        \hline
        Input & $f_t(x)$ & $r(x)$ & $b_{it}$ & $d_i$\\
        \hline
        I & $a_t^\top x$ & $-\kappa\norm{x-d/2}_2^2$ & U$(0,1)$ & 0.1\\
        \hline
        II & $a_t^\top x$ & $-\kappa\norm{x-d/2}_2^2$ & Bernoulli($p_i$) & U$(0.25,0.75)$\\
        \hline
        III & $-\frac{1}{4}x^2 + \xi_t x$ & 0 & 1 & 0.5 \\
        \hline
        IV & $-\frac{a_t}{4}x^2+\frac{a_t x}{2}$ & $ \kappa \min_i x_i/d_i $ & U$(0, 0.5)$ & 0.3 \\
        \hline
    \end{tabular}
    \caption{Parameter Settings of Inputs}
    \label{table:I}
\end{table}

\textit{Input model I: Online welfare maximization with costs, independent reward, and resource consumption.} The reward functions are linear as $f_t(x)=a_t^\top x$. The regularization function is the $\ell_2$ loss $r(x)= -\kappa \norm{x-d/2}_2^2$, which corresponds to the application of online welfare maximization with square costs. The reward coefficients $a_t$'s and the constraint coefficients $b_t$'s are i.i.d.~random vectors with dimension $m=6$. Sepecifically, $a_{it}$ is generated from the uniform distribution $U(0,10)$, and $b_{it}$ is generated from the uniform distribution $U(0,1)$. $\kappa$ is set to $0.001$.

\begin{figure}[!htbp]
    \centering    
    \includegraphics[width = 0.5\linewidth]{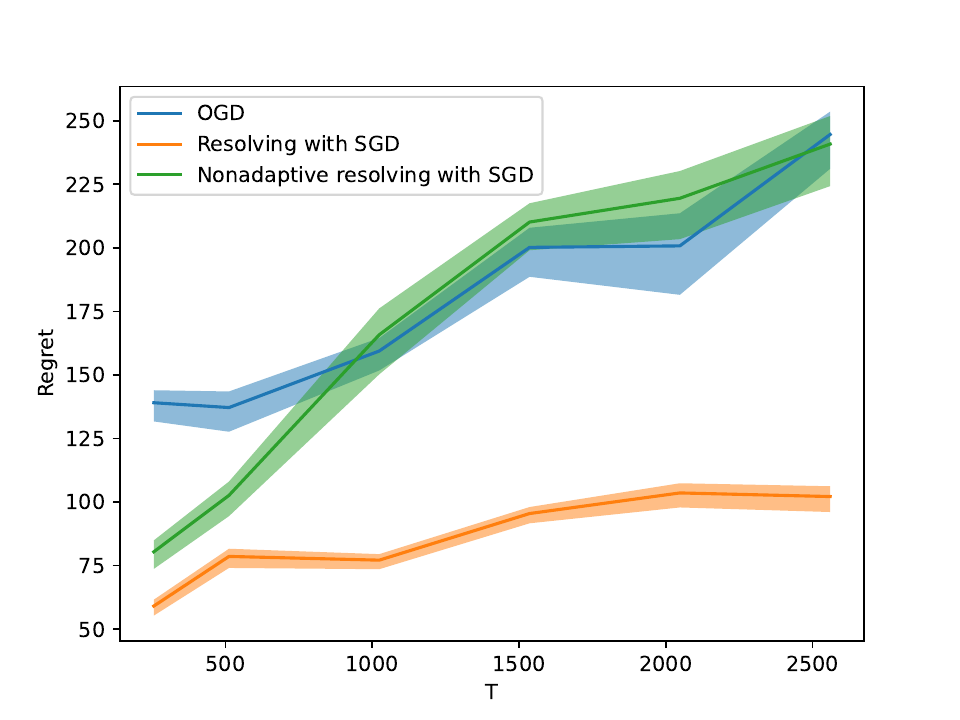}
    \caption{Regret versus horizon (T) under Input I. OGD stands for online gradient descent in \cite{balseiro2020dual}; resolving with SGD is our  Algorithm~\ref{alg:SGD}; nonadaptive resolving with SGD is the nonadaptive version (i.e., without updating the constraints) of Algorithm~\ref{alg:SGD}. }
    \label{fig:regret-input-I}
\end{figure}

\begin{figure}[!htbp]
     \centering
     \begin{subfigure}[t]{0.3\textwidth}
         \centering
         \includegraphics[width=\textwidth]{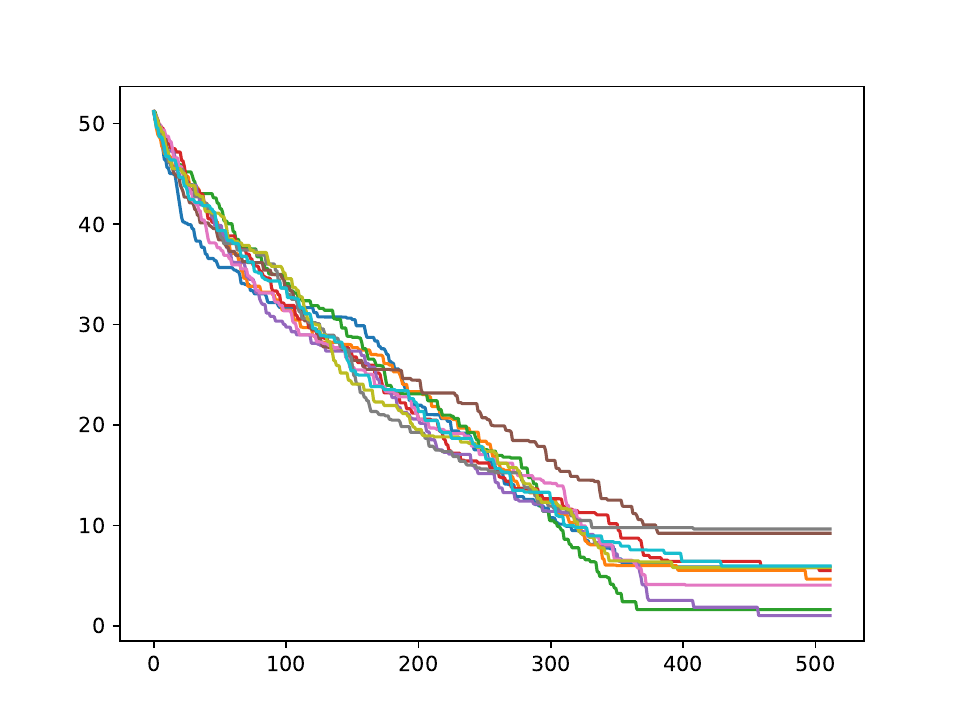}
         \caption{Online gradient descent}
         \label{fig:constraint-ogd-input-I}
     \end{subfigure}
     \hfill
     \begin{subfigure}[t]{0.3\textwidth}
         \centering
         \includegraphics[width=\textwidth]{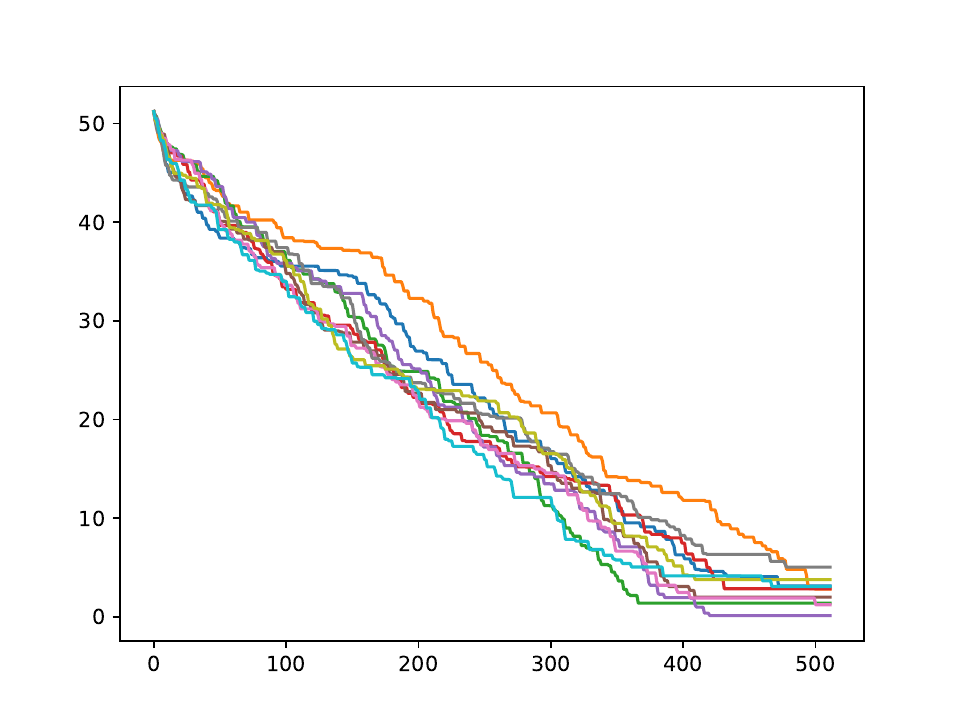}
         \caption{Nonadaptive version of Algorithm~\ref{alg:SGD}}
         \label{fig:constraint-mgd-n-input-I}
     \end{subfigure}
     \hfill
     \begin{subfigure}[t]{0.3\textwidth}
         \centering
         \includegraphics[width=\textwidth]{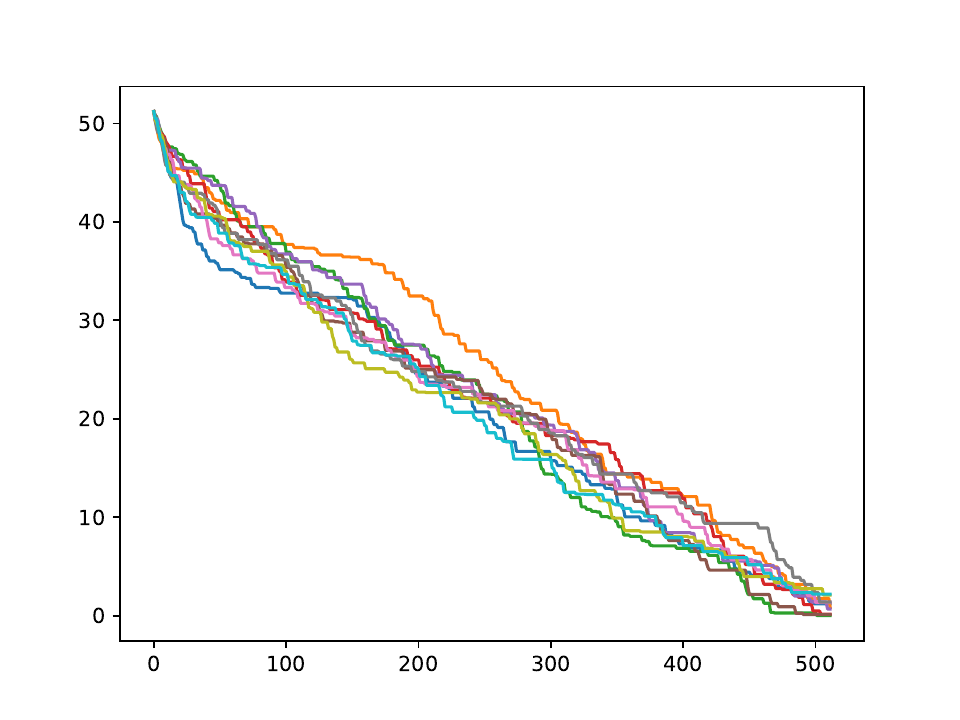}
         \caption{Algorithm~\ref{alg:SGD}}
         \label{fig:constraint-mgd-a-input-I}
     \end{subfigure}
    \caption{Remaining resource of one binding dimension versus time under input model I with $T=512$. Ten curves are displayed, each of which corresponds to one simulation.}
    \label{fig:constraint-input-I}
\end{figure}

To illustrate how the regret scales with the time horizon $T$, we evaluate the algorithms with different $T$ chosen from $\{256, 512, 1024, 1536, 2048, 2560\}$. We find that Resolving with SGD (Algorithm~\ref{alg:SGD}) shows logarithmic regret, while its counterpart without constraint update ($d_t\equiv d$ in Equation~\ref{eq:dual_gradient}) shows a much worse regret. We name the latter algorithm as the ``Nonadaptive resolving with SGD". The online gradient descent (OGD) method  in~\cite{balseiro2020dual} exhibits a $O(\sqrt{T})$ regret as indicated in their theoretical findings. The regret comparison between the algorithms can be found in Figure~\ref{fig:regret-input-I}. In Figure~\ref{fig:constraint-input-I}, we plot the dynamics of resource consumption for one binding dimension of the aforementioned algorithms. Ten curves are displayed, each of which corresponds to one simulation. 
 Being adaptive to the level of remaining resources, Algorithm~\ref{alg:SGD} controls carefully the constraint consumption to ensure that the resources are consumed at a steady rate till they are used up. In comparison, both the OGD and the nonadaptive version of Algorithm~\ref{alg:SGD} stop allocating resources too early, demonstrating the benefits of the constraint updates, which exploit the history of past actions.

\textit{Input model II: Online welfare maximization with costs, dependent reward and resource consumption.} The parameter setting below is based on \cite{balseiro2022best}. The reward functions and the regularization function are the same as in input I, whereas input II considers the case when the reward coefficients $a_t$'s are random variables conditional of the constraint coefficients $b_t$'s. We set $a_t = \text{Proj}_{[0,10]} \{ \theta_t^\top b_t + \delta_t \mathbf{1} \}$, where $\theta_t$ is generated from a multi-variate Gaussian distribution $N(0,\text{diag}(1))$, and $\delta_t$ is generated from the standard Gaussian distribution $N(0,1)$. The constraint coefficients $b_{it}$'s are generated from Bernoulli distribution with probability parameter $p_i$ with $p_i = (1+\alpha)/2$, and $\alpha$ is generated from the beta distribution Beta$(1,3)$. The average resource constraints $d_i$'s are generated from the uniform distribution U$(0.25, 0.75)$. $\kappa$ is set to $0.001$.

\begin{figure}[!htbp]
    \centering
    \includegraphics[width = 0.5\linewidth]{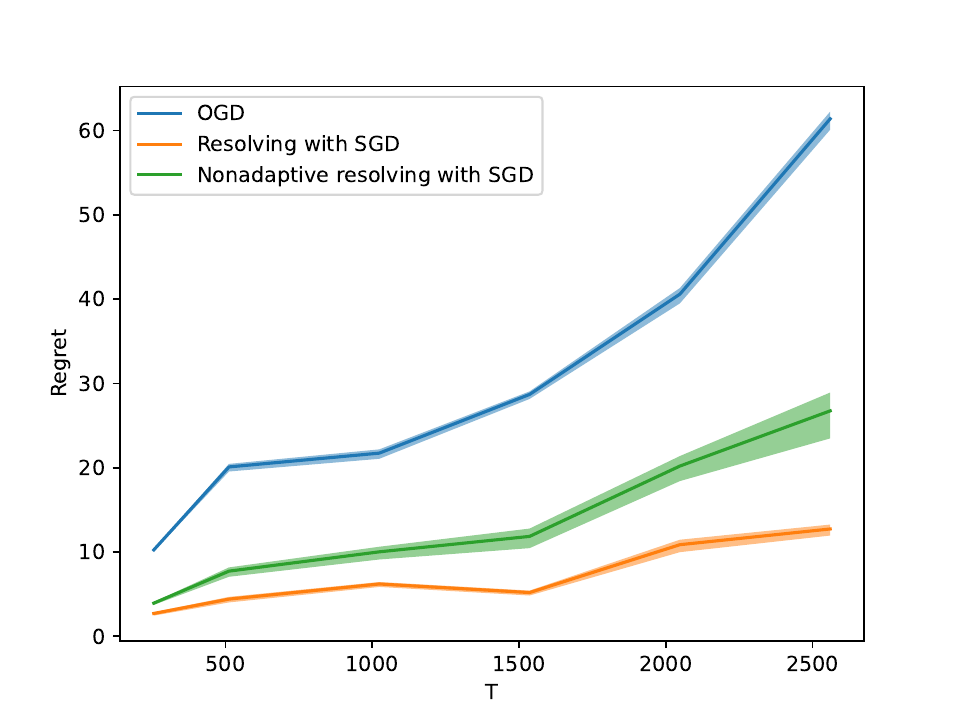}
    \caption{Regret versus horizon (T) under Input II. OGD stands for online gradient descent in \cite{balseiro2020dual}; resolving with SGD is Algorithm~\ref{alg:SGD}; nonadaptive resolving with SGD is the nonadaptive version of Algorithm~\ref{alg:SGD}.}
    \label{fig:regret-input-II}
\end{figure}

Similar to the setting of input I, we evaluate the algorithms under input II with different $T$'s and fix $m=6$. The regret performances and resource consumption are displayed in Figure~\ref{fig:regret-input-II} and Figure~\ref{fig:constraint-input-II}, respectively. Among the three algorithms (Algorithm~\ref{alg:SGD}, the nonadaptive Algorithm~\ref{alg:SGD} and the OGD method in~\cite{balseiro2020dual}),  Algorithm~\ref{alg:SGD} achieves a logarithmic regret, the nonadaptive Algorithm~\ref{alg:SGD} suffers from a higher regret while the regret of OGD grows in a much faster speed. 

\begin{figure}[!htbp]
     \centering
     \begin{subfigure}[t]{0.3\textwidth}
         \centering
         \includegraphics[width=\textwidth]{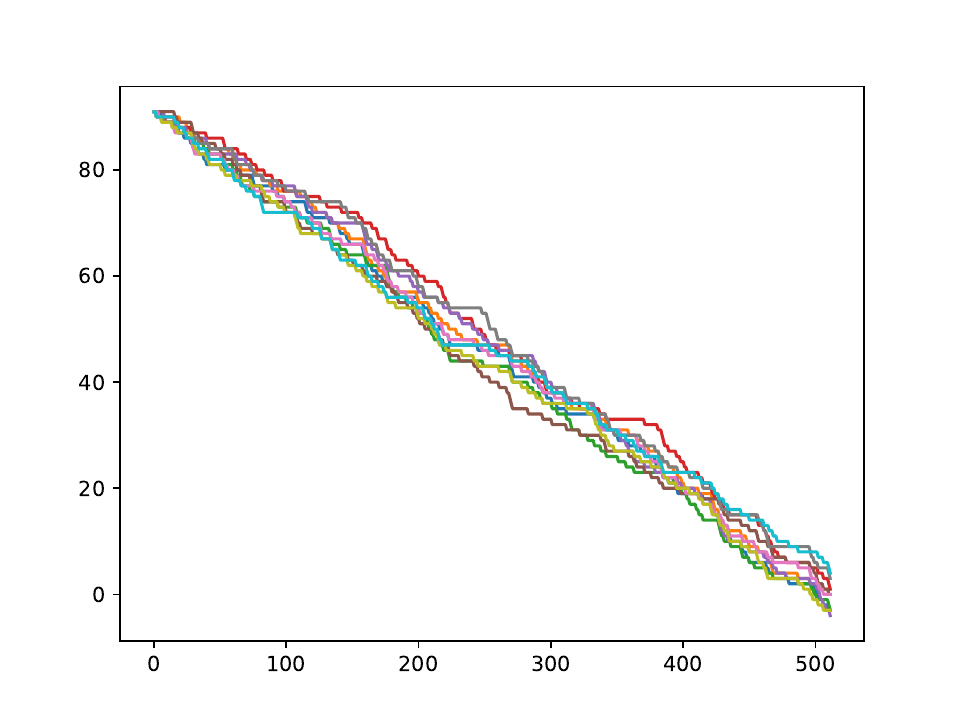}
         \caption{Online gradient descent}
         \label{fig:constraint-ogd-input-II}
     \end{subfigure}
     \hfill
     \begin{subfigure}[t]{0.3\textwidth}
         \centering
         \includegraphics[width=\textwidth]{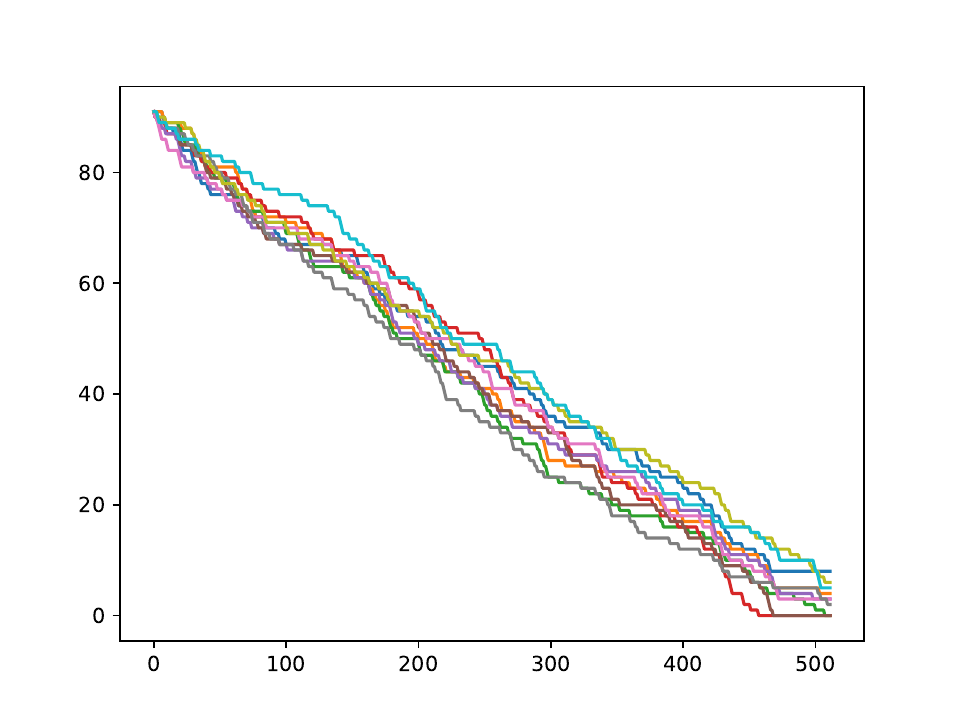}
         \caption{Nonadaptive version of Algorithm~\ref{alg:SGD}}
         \label{fig:constraint-mgd-n-input-II}
     \end{subfigure}
     \hfill
     \begin{subfigure}[t]{0.3\textwidth}
         \centering
         \includegraphics[width=\textwidth]{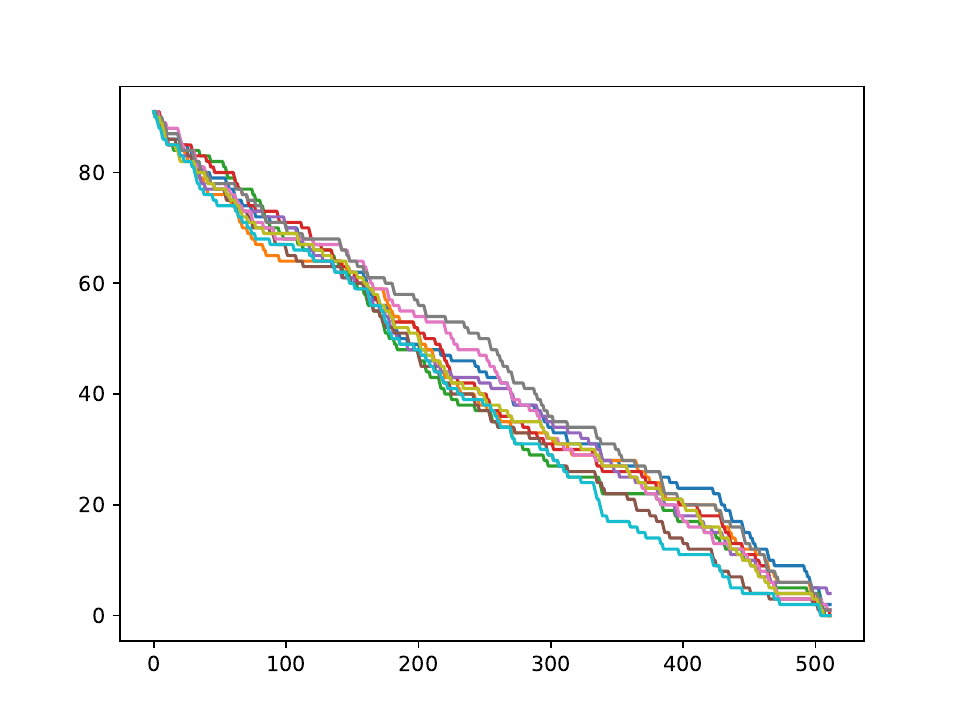}
         \caption{Algorithm~\ref{alg:SGD}}
         \label{fig:constraint-mgd-a-input-II}
     \end{subfigure}
    \caption{Remaining resource of one binding dimension versus time under input model II with $T=512$. Ten curves are displayed, each of which corresponds to one simulation.}
    \label{fig:constraint-input-II}
\end{figure}

\textit{Input model III: Non-regularized online convex resource allocation with one resource.} In this model, we assess the algorithms' performance under a non-regularized special case, where there is only one resource, the reward function $f_t(x) = f_t(x, \xi_t)= -\frac{1}{4}x^2 + \xi_t x$, the constraint $d = \frac{1}{2}$ and cost $b_t = 1$. The random variable $\xi_t$ follows a two-point distribution that takes value in $\{\frac{1}{2}, \frac{3}{4}\}$ with equal probability, i.e., $\PP[\xi_t = \frac{1}{2}] = \PP[\xi_t = \frac{3}{4}] = 0.5$. This special case is used in the proof of Theorem~\ref{thm:ogd_lb}. 

For input model III, the optimal solution to Problem~(\ref{eq:dual_expected}) admits a closed-form due to the simple distribution, which also leads to a relatively small regret compared to other input models. We compare further with the ``No learning" algorithm, which is the convex version of Algorithm 1 in~\cite{li2021online}. It requires the computation of optimal dual solutions, while neither Adaptive (Algorithm~\ref{alg:SGD}) nor OGD needs this information. The regret comparison is shown in Figure~\ref{fig:regret-special-case}. The empirical performance corroborates the theoretical results in that the infrequent resolving saves computation significantly with a relatively small performance loss. We further explain the reason for the performance advantage by plotting the remaining time before stopping in Figure~\ref{fig:remainingTime-special-case}. Benchmark algorithms without constraint update (No learning, OGD) stop allocating resource $O(\sqrt{T})$ steps earlier than Adaptive (Algorithm~\ref{alg:SGD}) and Infreq-adaptive (Algorithm~\ref{alg:infrequent}), which leads to the terrible regret performance. In comparison, Fast (Algorithm~\ref{alg:fast}) uses both the updated constraints and the original constraints, leading to a less accurate dual solution compared to Algorithm~\ref{alg:infrequent} and worse performance.

\begin{figure}[t!]
     \centering
     \begin{subfigure}[t!]{0.44\textwidth}
         \centering
         \includegraphics[width=\textwidth]{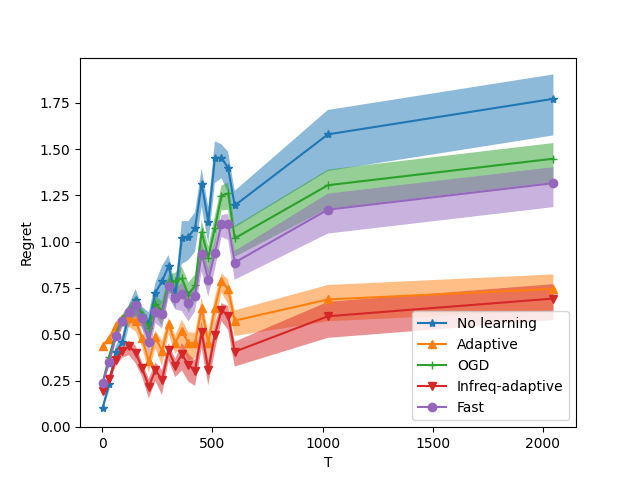}
         \caption{Regret versus horizon ($T$) under input III.}
         \label{fig:regret-special-case}
     \end{subfigure}
     \hfill
     \begin{subfigure}[t!]{0.44\textwidth}
         \centering
         \includegraphics[width=\textwidth]{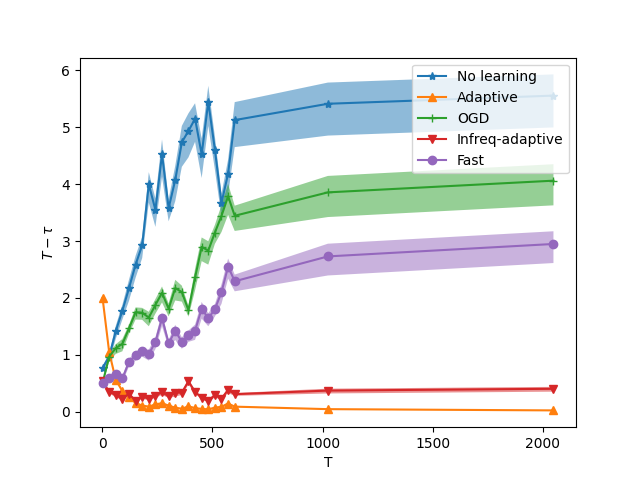}
         \caption{Remaining time ($T-\tau$) versus horizon ($T$) under input III}
         \label{fig:remainingTime-special-case}
     \end{subfigure}
    \caption{Performance evaluation under input III. No learning is the convex version of Algorithm~1 in \cite{li2021online}; OGD is online gradient descent in \cite{balseiro2020dual}; Adaptive is Algorithm~\ref{alg:SGD}; infreq-adaptive is Algorithm~\ref{alg:infrequent}; Fast is Algorithm~\ref{alg:fast}.}
    \label{fig:special-case}
\end{figure}

\textit{Input model IV: Online convex allocation with max-min regularizer.} To demonstrate that our algorithm also works well for nonsmooth regularizers, we set up input model IV. The reward functions are $f_t(x)=-\frac{a_t}{4}x^2+\frac{a_t}{2}x$, where $a_{t}\in \mathbb{R}$ is generated from the uniform distribution $U(0,10)$. Each dimension of the $t$th constraint coefficient, i.e., $b_{it}, i\in [m]$, is generated from the uniform distribution $U(0,0.5)$. We set $m=12$ in this input model. The regularization function is $r(x)=\kappa \min_i x_i/d_i $ so that the minimum allocation is not too small. $\kappa$ is set to $[0, 0.01,0.1]$ to compare the performance of algorithms under different regularization levels.

\begin{figure}
    \centering
    \begin{subfigure}[b]{0.475\textwidth}
            \centering
            \includegraphics[width=\textwidth]{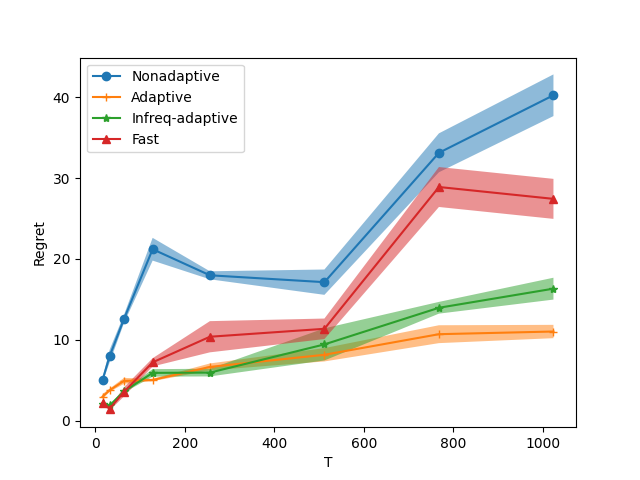}
            \caption[]%
            {$\kappa=0$}    
            \label{fig: maxmin-regret-0}
        \end{subfigure}
        \hfill
        \begin{subfigure}[b]{0.475\textwidth}  
            \centering 
            \includegraphics[width=\textwidth]{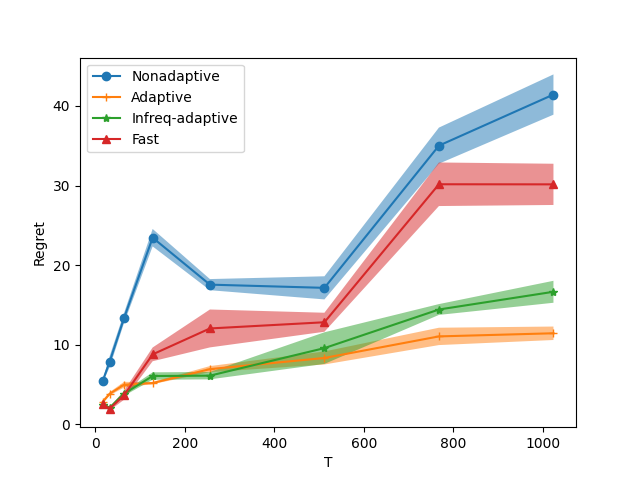}
            \caption[]%
            { $\kappa=0.01$}   
            \label{fig:maxmin-regret-0.01}
        \end{subfigure}
        \vspace{-0.12cm}
        \begin{subfigure}[b]{0.475\textwidth}   
            \centering 
            \includegraphics[width=\textwidth]{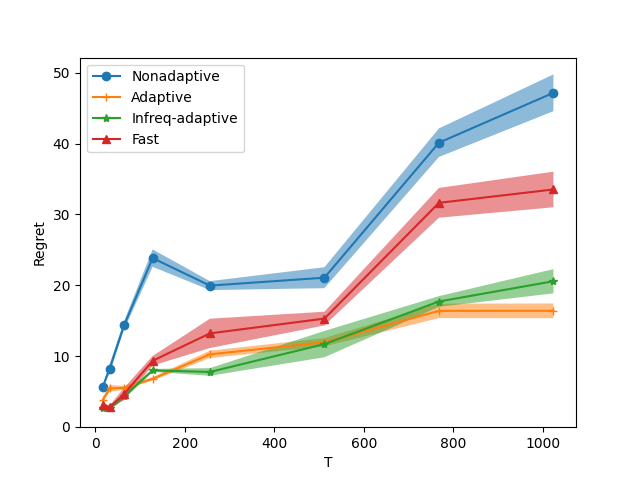}
            \caption[]%
            {$\kappa=0.1$}
            \label{fig:maxmin-regret-0.1}
        \end{subfigure}
        \hfill
        \begin{subfigure}[b]{0.475\textwidth}   
            \centering 
            \includegraphics[width=\textwidth]{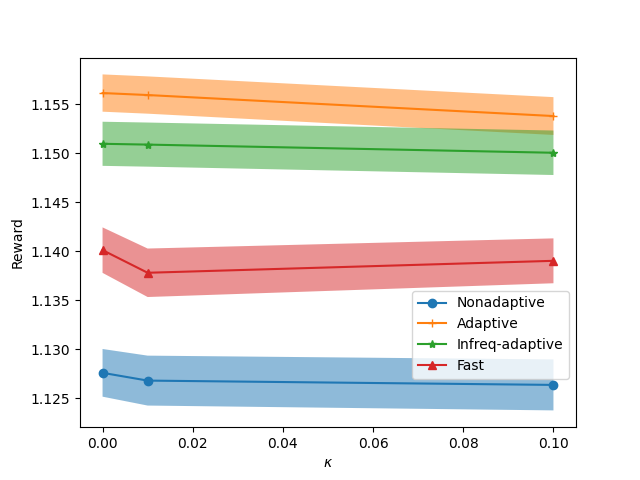}
            \caption[]%
            {Per-step reward ($\sum_t f_t(x_t)/T$) versus $\kappa$ when $T=1024$.}
            \label{fig:maxmin-tradeoff}
        \end{subfigure}
    \caption{Regret versus horizon ($T$) under input IV with different $\kappa$s in Figure~\ref{fig: maxmin-regret-0}. Nonadaptive is Algorithm~\ref{alg:ogd} with OGD; Adaptive is Algorithm~\ref{alg:SGD}; Infreq-adaptive is Algorithm~\ref{alg:infrequent}; Fast is Algorithm~\ref{alg:fast}. Figure~\ref{fig:maxmin-tradeoff} shows the impact of different regularization levels on the per-step reward (with 95\% confidence interval).}
    \label{fig:regret_input_V}
\end{figure}
In Figure~\ref{fig:regret_input_V}, regret curves under different $\kappa$s of Nonadaptive (Algorithm~\ref{alg:ogd} with OGD), Adaptive (Algorithm~\ref{alg:SGD}), Infreq-adaptive (Algorithm~\ref{alg:infrequent}) and Fast (Algorithm~\ref{alg:fast}) are plotted. The regret of Algorithm~\ref{alg:SGD} and Algorithm~\ref{alg:infrequent} grow slower than the other two, which shows the advantage of using updated constraint information. 
It is observed in Figure~\ref{fig:maxmin-tradeoff} that the per-step reward decreases as we regularize the solution by setting $\kappa>0$, which is consistent with the intuition that regularization has a negative impact on the revenue. Also, it is observed that the reward decreases when the regularization level further increases, but we also observe a slight increase for Fast (Algorithm~\ref{alg:fast}) when $\kappa$ increases from $0.01$ to $0.1$, which indicates that a proper regularization level may not decrease the reward too much. It would be interesting to study the problem of choosing the appropriate regularization level in the future.

\section{More Discussions}
\subsection{Computational cost}

Specifically, for strongly convex dual objectives, Our algorithm of frequent resolving requires computing gradients for $O(T^2)$ times in total; for more general dual objectives, it requires $O(T^4)$ times of gradient computation in total.  With a good initialization, we can unevenly divide time horizon $T$ into $\log T$ epochs and only solve empirical dual optimization at the beginning of each epoch. This infrequent algorithm only takes gradient computations for $O(T\log T)$ times for strongly convex problems (which is nearly linear) and $O(T^3\log T)$ for general problems, while delivering an optimal regret bound. The fast algorithm we established has only linear computational cost $O(T)$, which is comparable with OGD but reaches sub-optimal regret $O(\log^2 T)$.

\subsection{Fenchel conjugate of regularizers}
Here we provide Fenchel conjugates for three commonly used regularizers.
\begin{enumerate}
        \item \textbf{$\ell_{1}$-loss}: $r(a):=-\kappa\norm{a}_1$. Define $\cZ := \left\{a \middle| \norm{a}_{\infty}\le L \right\}$, $h(a)=r(a) -\mu^\top a $ then $r^*(\mu)= \max_{a\in\cZ}\{ r(a) -\mu^\top a \}:= \max_{a\in\cZ}\{ h(a) \}$. We have the subgradient: $\nabla h(a) = -\kappa\operatorname{sign}(a)-\mu$. Since $\norm{\mu}_\infty\le G=\kappa$, we know that $h(a)$ takes its maximum only when $a=0$. the conjugate $r^*$ in $\Omega_\mu$ is thus of the form $r^*(\mu)=0$ when $\kappa \ge G =\kappa$.
    \item \textbf{Max-min loss}: $r(a):=\kappa \min_i ({a_i}/{d_i}) $. Define $\cZ := \left\{a \middle| \abs{a_i}\le d_i L, \ i\in [m] \right\}$,  and $z_i=a_i/d_i$. Then we define the function to be maximized as $h(z):=r(a) -\mu^\top a =\kappa \min(z_i)-\sum \mu_i d_i z_i $, for the region $\norm{z}_\infty \le L$. By computing the subgradient of each dimension, we know that the optimal $z$ that maximizes the $h(z)$ must have:
    \begin{equation*}
        z_i=\begin{cases}
            L & \text{ if } \mu_i d_i<0 \\
            -L & \text{ if } \mu_i d_i>\kappa \\
            \min_i(z_i)  & \text{ if } \mu_i d_i\in [0,\kappa]
        \end{cases}. 
    \end{equation*}
    Therefor, we have $z_i=\min_i(z_i)$ if $\mu_i\ge 0$. Moreover, whether $\min_i(z_i)$ takes $L$ or $-L$ depends on the value $\kappa-\sum  d_i \left(\mu_i\right)_+$. If $\kappa-\sum  d_i \left(\mu_i\right)_+>0$, then we will have $\min_i(z_i)=L$, otherwise $\min_i(z_i)=-L$. Thus, the conjugate $r^*$ in $\Omega_\mu$ is of the form $r^*(\mu)=L\cdot\left( \abs{\kappa-\sum  d_i \left(\mu_i\right)_+} -\sum  d_i \left(\mu_i\right)_- \right)  $.
    
    \item \textbf{Negative max loss}: $r(a):=-\kappa\max_i ({a_i}/{d_i})$. Define $\cZ := \left\{a \middle| \abs{a_i}\le d_i L, \ i\in [m] \right\}$, and $z_i=a_i/d_i$. Then we have $h(z):=r(a) -\mu^\top a =-\kappa \max(z_i)-\sum \mu_i d_i z_i$. Following  an analogous argument as in the max-min loss, we have $r^*(\mu)=L\cdot\left( \abs{\kappa+\sum  d_i \left(\mu_i\right)_-} +\sum  d_i \left(\mu_i\right)_+ \right)  $
\end{enumerate}

\end{document}